\documentclass[twoside,11pt]{article}

\usepackage{jmlr2e}

\usepackage{lastpage}
\jmlrheading{25}{2024}{1-\pageref{LastPage}}{10/22; Revised
5/24}{5/24}{22-1159}{Yuanyuan Wang, Wei Huang, Mingming Gong, Xi Geng, Tongliang Liu, Kun Zhang, Dacheng Tao}
\ShortHeadings{Identifiability and Asymptotics in Learning Linear ODEs}{Wang, Huang, Gong, Geng, Liu, Zhang, Tao}

\firstpageno{1}

\usepackage{microtype}
\usepackage{graphicx}
\usepackage{subfigure}
\usepackage{amssymb}
\usepackage{booktabs} 
\usepackage{multirow}
\usepackage{mdframed}
\usepackage{amsmath,nccmath}
\usepackage{bbm}
\usepackage{bigints}

\let\proof\relax

\newenvironment{proof}{\par\noindent{\bf Proof\ }}{\hfill\BlackBox\\[2mm]}

\usepackage{pifont}

\begin{document}

\title{Identifiability and Asymptotics in Learning Homogeneous Linear ODE Systems from Discrete Observations}

\author{\name Yuanyuan Wang\email yuanyuanw2@student.unimelb.edu.au\\
\name Wei Huang \email wei.huang@unimelb.edu.au\\
\name Mingming Gong \thanks{Corresponding author.} \email mingming.gong@unimelb.edu.au\\
\name Xi Geng\email xi.geng@unimelb.edu.au\\
\addr School of Mathematics and Statistics\\
     University of Melbourne, Melbourne, Australia
\AND
\name Tongliang Liu \email tongliang.liu@sydney.edu.au\\
\addr School of Computer Science, Faculty of Engineering\\ 
The University of Sydney, Sydney, Australia
\AND
\name Kun Zhang \email kunz1@cmu.edu\\
\addr Carnegie Mellon University, Pittsburgh, PA, USA\\
Mohamed bin Zayed University of Artificial Intelligence, Abu Dhabi, UAE
\AND
\name Dacheng Tao \email dacheng.tao@sydney.edu.au\\
\addr School of Computer Science, Faculty of Engineering\\ 
The University of Sydney, Sydney, Australia
}

\editor{Lorenzo Rosasco}

\maketitle

\begin{abstract} 
\noindent  Ordinary Differential Equations (ODEs) have recently gained a lot of attention in machine learning. However, the theoretical aspects, for example, identifiability and asymptotic properties of statistical estimation are still obscure. This paper derives a sufficient condition for the identifiability of homogeneous linear ODE systems from a sequence of equally-spaced error-free observations sampled from a single trajectory. When observations are disturbed by measurement noise, we prove that under mild conditions, the parameter estimator based on the Nonlinear Least Squares (NLS) method is consistent and asymptotic normal with $n^{-1/2}$ convergence rate. Based on the asymptotic normality property, we construct confidence sets for the unknown system parameters and propose a new method to infer the causal structure of the ODE system, that is, inferring whether there is a causal link between system variables. Furthermore, we extend the results to degraded observations, including aggregated and time-scaled ones. To the best of our knowledge, our work is the first systematic study of the identifiability and asymptotic properties in learning linear ODE systems. We also construct simulations with various system dimensions to illustrate the established theoretical results.
\end{abstract}


\section{Introduction} \label{Introduction}
Ordinary Differential Equations (ODEs) have been widely used to model dynamic systems in various scientific fields such as physics \citep{ferrell2011modeling,xiu2003oscillation,massatt1983limiting}, biology \citep{hemker1972numerical,huang2006hierarchical,li2011large,lu2011high}, and economics \citep{tu2012dynamical,weber2011optimal,norton1981market}. In recent years, ODEs are also attracting increasing attention in the machine learning community. For instance, ODEs have been used to build new families of deep neural networks \citep{chen2018neural,yan2019robustness,zhang2019approximation} and the connection between ODEs and structural causal models has been established \citep{mooij2013ordinary, rubenstein2018deterministic}.

Existing works mostly focus on the parameter estimation of ODEs  \citep{hemker1972numerical, li2005parameter,xue2010sieve,varah1982spline,liang2008parameter,ramsay1996principal, ramsay2007parameter,poyton2006parameter}. However, before estimating the unknown parameters, it is essential to perform an identifiability analysis of an ODE system; that is, uncovering the mathematical conditions under which the parameters can be uniquely determined from noise-free observations. If a system is not identifiable, the estimation procedure may produce erroneous and misleading parameter estimates \citep{qiu2022identifiability}. This is detrimental in many applications; for example, the estimated parameter can easily lead to wrong causal structures of non-identifiable ODE systems.

The contribution of this paper are summarized as follows:

{\bf Derive identifiability condition for linear ODEs from discrete observations.} We derive the condition for the identifiability of homogeneous linear ODE systems from discrete observations (collected at discrete time points). Specifically, we consider the setting where the observations are sampled from a single trajectory generated from one initial condition. This setting is prevalent in applications that can only access a single trajectory due to the unrepeatable process of the measurements collection. Our identifiability analysis is built upon the work of \citet{stanhope2014identifiability}, which constructs a systematic study of the identifiability of homogeneous linear ODE systems from a continuous trajectory with known initial conditions. Our research extends this framework to more practical scenarios where we only have discrete observations and do not know the initial conditions.

{\bf Derive asymptotic properties for NLS estimator.} Based on our identifiability results, we study the asymptotic properties of parameter estimation from data with measurement noises. We focus on the estimator obtained by the Nonlinear Least Squares (NLS) method, which is simple and widely used in dynamical systems \citep{marquardt1963algorithm, johnson1981analysis,wu1981asymptotic, johnson198516,maeder1990nonlinear}. However, the asymptotic properties of NLS based estimators for ODE systems have not been systematically studied due to the lack of identifiability conditions. We prove that under mild conditions, the NLS estimator is consistent and asymptotic normal with $n^{-1/2}$ convergence rate. In addition, based on the established asymptotic normality theory, we construct the confidence sets of unknown parameters and propose a new method to infer the causal structure of ODE systems, that is, inferring whether there is a causal link between system variables. 

{\bf Extend theoretical results to degraded observations.} We extend the consistency and asymptotic normality results to the observations with degraded quality, including aggregated and time-scaled observations. The aggregated observations are usually caused by time aggregation in the data collection process \citep{silvestrini2008temporal}. The time-scaled observations result from data preprocessing to fit the ODE model. We prove that the ODE model generating the original observations is identifiable from the degraded observations. The asymptotic properties can be naturally extended given the identifiability results. Simulations with various system dimensions are constructed to verify the developed theoretical results.

\section{Identifiability Condition of Homogeneous Linear ODE Systems}\label{sec:identifiability}

Linear ODE systems hold significant importance in modelling and comprehending the dynamics of various real-world phenomena across diverse disciplines. For instance, well-established examples such as the Spring-Mass-Damper system \citep{sharma2014modeling}, the simple population growth model \citep{anisiu2014lotka}, and the heat cooling problem model \citep{langtangen2016scaling} all belong to the category of linear ODE systems. These systems often admit closed-form analytical solutions, thereby facilitating precise predictions of system behaviour and enabling detailed analysis. Moreover, linear ODE systems serve as a foundation for approximating nonlinear systems through linearization techniques, providing valuable insights into the more intricate nonlinear system behaviour. In this paper, we focus on an important special case of linear ODE systems: the homogeneous linear ODE system.

A homogeneous linear ODE system can be defined as:
\begin{equation}\label{eq:ODE model}
\begin{split}
    \Dot{\boldsymbol{x}}(t) &= A \boldsymbol{x}(t) \, , \\
    \boldsymbol{x}(0) &= \boldsymbol{x}_{0}\, ,
\end{split}
\end{equation}
where $t \in [0,\infty)$ denotes the independent variable time, $\boldsymbol{x}(t) \in \mathbb{R}^d $ denotes the state of the ODE system at time $t$, $\Dot{\boldsymbol{x}}(t)$ denotes the first derivative of $\boldsymbol{x}(t)$ with respect to time t, and we refer to both the parameter matrix $A\in \mathbb{R}^{d \times d}$ and the initial condition  $\boldsymbol{x}_0 \in \mathbb{R}^d$ as the system parameters. In this paper we focus on ODE systems with complete observation, that is, all state variables are observable. Therefore, the measurement model can be described as:
\begin{equation}
\boldsymbol{y}(t) = \boldsymbol{x}(t)\, .
\end{equation}
The solution of ODE system (\ref{eq:ODE model}) can be explicitly expressed as:
\begin{equation}\label{eq:ODE solution}
    \boldsymbol{x}(t; \boldsymbol{x}_0, A) = e^{At}\boldsymbol{x}_0 \, ,
\end{equation}
which is also called a trajectory. The symbol $e$ in Equation \eqref{eq:ODE solution} denotes the matrix exponential function. In this paper, we focus on identifiability analysis of the ODEs from observations sampled from a \textbf{single} $d$-dimensional trajectory generated with an initial condition $\boldsymbol{x}_0$. Under the setting of our paper, the term identifiability means that given the error-free observation from a single trajectory of the ODE system, whether the system parameters $A$ and $\boldsymbol{x}_0$ can be uniquely determined.

\subsection{Identifiability Condition from a Whole Trajectory}
Given a fixed initial condition $\boldsymbol{x}_0$, \citet{stanhope2014identifiability} derived a necessary and sufficient condition for identifying the ODE system (\ref{eq:ODE model}) from a whole trajectory, $\{e^{At}\boldsymbol{x}_0\}_{t\in[0,\infty)}$, that is,
$\{\boldsymbol{x}_0, A\boldsymbol{x}_0, \ldots, A^{d-1}\boldsymbol{x}_0\}$ are linearly independent. The notation $A^k$ denotes the $k$th power of the matrix $A$, represented as the product of $k$ copies of matrix $A$, that is, $A^k = A \times A \times \ldots \times A$. However, in practice, we cannot usually observe the initial condition $\boldsymbol{x}_0$. Under this practical circumstance, we need to treat the initial condition also as a system parameter and identify it from the data. In the following, we extend the identifiability definition and condition in \citep{stanhope2014identifiability} to the case where both parameter matrix $A$ and initial condition $\boldsymbol{x}_0$ are system parameters.

\begin{definition}\label{def:identifiability for A and b}
Let $(M^0, \Omega)$ be given parameter spaces,  with $M^0 \subset \mathbb{R}^d$ and $\Omega \subset \mathbb{R}^{d \times d}$. The ODE system $(\ref{eq:ODE model})$ is said to be identifiable in $(M^0, \Omega)$, if for all $\boldsymbol{x}_{0}, \boldsymbol{x}_{0}' \in M^0$ and all $A, A' \in \Omega$,  with ($\boldsymbol{x}_{0}, A) \neq  (\boldsymbol{x}_{0}',  A'$), it holds that $\boldsymbol{x}(\cdot; \boldsymbol{x}_{0}, A) \neq \boldsymbol{x}(\cdot; \boldsymbol{x}_{0}', A')$.
\end{definition}

Here $\boldsymbol{x}(\cdot; \boldsymbol{x}_{0}, A) \neq \boldsymbol{x}(\cdot; \boldsymbol{x}_{0}', A')$ means that there exists at least one $t \geq 0$ such that $\boldsymbol{x}(t; \boldsymbol{x}_{0}, A) \neq \boldsymbol{x}(t; \boldsymbol{x}_{0}', A')$. Then we establish the condition for identifiability of the ODE system based on Definition \ref{def:identifiability for A and b}.
\begin{lemma}\label{theorem:identifiability for A and b} 
Suppose that $M^0 \subset \mathbb{R}^d$ and $\Omega \subset \mathbb{R}^{d \times d}$ are both open subsets. Then the ODE system $(\ref{eq:ODE model})$ is identifiable in $(M^0, \Omega)$ if and only if $\{\boldsymbol{x}_0, A\boldsymbol{x}_0, \ldots, A^{d-1}\boldsymbol{x}_0\}$ are linearly independent for all $\boldsymbol{x}_0 \in M^0$ and all $A\in \Omega$.
\end{lemma} 

The proof of Lemma~\ref{theorem:identifiability for A and b} is a straightforward extension of the proof of Theorem 2.5 in \citep{stanhope2014identifiability} and can be found in Appendix~\ref{proof:theorem2.1}. From the Lemma, we can see that the condition for identifying the system parameters $(A,\boldsymbol{x}_0)$ is the same as that for only identifying $A$, except that the linear independence condition needs to hold for all possible $\boldsymbol{x}_0$ in $M^0$.

\subsection{Identifiability Condition from Discrete Observations}
In practice, typically, we can only access a sequence of discrete observations sampled from a trajectory instead of knowing the whole trajectory. Thus, from now on, we focus on the case where only discrete observations from a trajectory are available. In particular, We extend the identifiability definition of the ODE system \eqref{eq:ODE model} as follows.
\begin{definition}[$(\boldsymbol{x}_0, A)$-identifiability]\label{def:identifiability for discrete observations}
For $\boldsymbol{x}_0 \in \mathbb{R}^d$ and $A \in \mathbb{R}^{d \times d}$, for any $n_0 \geq 1$, let $t_j, j = 1, \ldots, n_0$ be any $n_0$ time points and $\boldsymbol{x}_j := \boldsymbol{x}(t_j;\boldsymbol{x}_0, A)$ be the error-free observation of the trajectory $\boldsymbol{x}(t; \boldsymbol{x}_0, A)$ at time $t_j$. We say the ODE system $(\ref{eq:ODE model})$ is $(\boldsymbol{x}_0, A)$ identifiable from $\boldsymbol{x}_1, \ldots, \boldsymbol{x}_{n_0}$, if for all $\boldsymbol{x}_{0}'\in \mathbb{R}^d$ and all $A'\in \mathbb{R}^{d \times d}$, with ($\boldsymbol{x}_{0}', A') \neq  (\boldsymbol{x}_{0}, A$), it holds that $\exists j$ for $j=1,\ldots,n_0$, such that $\boldsymbol{x}(t_j; \boldsymbol{x}_{0}', A') \neq \boldsymbol{x}(t_j; \boldsymbol{x}_{0}, A)$.
\end{definition} 

This definition is inspired by \citep[Definition 1.6]{qiu2022identifiability}, and it is not a simple extension of Definition~\ref{def:identifiability for A and b} to discrete observations when $M^0 := \mathbb{R}^d$ and $\Omega := \mathbb{R}^{d \times d}$. The reason is that in Definition~\ref{def:identifiability for discrete observations}, the initial condition $\boldsymbol{x}_0$ and parameter matrix $A$ are fixed, and $\boldsymbol{x}_0'$ and $A'$ are an arbitrary vector and an arbitrary matrix in $M^0$ and $\Omega$ respectively. However, in Definition~\ref{def:identifiability for A and b}, both $\boldsymbol{x}_0$ and $\boldsymbol{x}_0'$ are arbitrary vectors in $M^0$ and both $A$ and $A'$ are arbitrary matrices in $\Omega$. In other words, Definition~\ref{def:identifiability for discrete observations} describes an intrinsic property of a single system instead of a collective property of a set of systems. In dealing with the identifiability problem of an ODE system, we aim to check whether the true underlying system parameter $(\boldsymbol{x}_0, A)$ is uniquely determined by error-free observations. Therefore, $(\boldsymbol{x}_0, A)$-identifiability described in Definition~\ref{def:identifiability for discrete observations} is a more natural way to define the identifiability of the ODE system from the practical perspective, and all of the other relevant definitions and theorems in the rest of this paper are derived based on Definition~\ref{def:identifiability for discrete observations}.

To derive the identifiability condition from discrete observations, we focus on the equally-spaced observations. Specifically, data are collected on an equally-spaced time grid, that is, $t_{j+1}-t_j = \Delta t$ for a constant $\Delta t>0$ and $j=1, 2, \ldots$. The motivation is that a time series is most commonly collected equally spaced in practice, which follows the standard rules of collecting data either from a scientific experiment or a natural phenomenon.
Then, based on Definition \ref{def:identifiability for discrete observations}, we derive a sufficient condition for the identifiability of the ODE system from discrete observations. 

\begin{theorem}\label{theorem:identifiability from discrete observations}
For $ \boldsymbol{x}_0 \in$ $\mathbb{R}^d$ and $A \in \mathbb{R}^{d \times d}$, the ODE system \textnormal{(\ref{eq:ODE model})} is $(\boldsymbol{x}_0, A)$ identifiable from \textbf{any} $d+1$ equally-spaced error-free observations $\boldsymbol{x}_1, \boldsymbol{x}_2, \cdots, \boldsymbol{x}_{d+1}$, if the following two conditions are satisfied. 
\begin{enumerate}
    \item[\textnormal{A1}] $\{\boldsymbol{x}_0, A\boldsymbol{x}_0, \ldots, A^{d-1}\boldsymbol{x}_0\}$ are linearly independent.
    \item[\textnormal{A2}] Parameter matrix $A$ has $d$ distinct real eigenvalues.
\end{enumerate}
\end{theorem}

The proof of Theorem~\ref{theorem:identifiability from discrete observations} can be found in Appendix~\ref{proof:theorem2.2}. Here, \textbf{any} $d+1$ equally-spaced error-free observations means that the time interval between two consecutive observations: $\Delta t$ can take any positive value. In other words, the identifiability of the ODE system will not be influenced by the time-space between consecutive observations.

Now, in addition to the identifiability condition from a whole trajectory (condition~A1), when only discrete observations are available, we further require that $A$ has $d$ distinct real eigenvalues (condition A2). This condition seems to be restrictive, however, due to the limited observations (a set of equally-spaced observations sampled from a single trajectory), the condition cannot be relaxed. The reasons are as follows: (1) \textbf{Distinct eigenvalues}: as discussed in \citep{qiu2022identifiability}, almost every $A\in \mathbb{R}^{d\times d}$ (with respect to the Lebesgue measure on $\mathbb{R}^{d\times d}$) has $d$ distinct eigenvalues based on random matrix theory \citep{lehmann1991eigenvalue, tao2012topics}. Therefore, $A$ has $d$ distinct eigenvalues is a natural and reasonable assumption. In addition, to guarantee the ODE system \eqref{eq:ODE model} is $(\boldsymbol{x}_0, A)$ identifiable from any $d+1$ equally-spaced observations sampled from a single trajectory, we need $d$ consecutive observations of them to be linearly independent. To ensure any d equally-spaced observations sampled from a single trajectory are linearly independent, matrix A has $d$ distinct eigenvalues is a necessary condition. Moreover, in Section~\ref{section:asymptotic normality}, to derive the explicit formula of the asymptotic covariance matrix of the NLS estimator's asymptotic normal distribution, we require the parameter matrix has $d$ distinct eigenvalues. (2) \textbf{Real eigenvalues}: according to \citep[Corollary 6.4]{stanhope2014identifiability}, a matrix with complex eigenvalues is not identifiable from any set of equally-spaced observations sampled from a single trajectory. Therefore, we require the eigenvalues to be real. 

Our identifiability condition here is sufficient but not necessary. However, it allows us to derive explicit expressions of $A$ and $\boldsymbol{x}_0$ in terms of the observations. To see this, we set $\Phi(t) := e^{At}$, and let $\boldsymbol{X}_j$ denote the matrix $(\boldsymbol{x}_j, \boldsymbol{x}_{j+1},\ldots,\boldsymbol{x}_{j+d-1})\in\mathbb{R}^{d\times d}$ for $j=1,2$. Then $\boldsymbol{X}_2 = \Phi({\Delta t})\boldsymbol{X}_1$. We show in the proof that $\boldsymbol{X}_1$ is nonsingular if $A$ has $d$ distinct eigenvalues, and thus,  $\Phi(\Delta t) = \boldsymbol{X}_2 \boldsymbol{X}_1^{-1}$. Finally, we can obtain a unique real $A$ by taking logarithm of $e^{A \Delta t}=\boldsymbol{X}_2\boldsymbol{X}_1^{-1}$ if $A$ has $d$ distinct real eigenvalues \citep[Theorem 6.3]{stanhope2014identifiability}. The initial condition $\boldsymbol{x}_0$ can then be calculated by $e^{-At_1}\boldsymbol{x}_1$. Please refer to Appendix~\ref{proof:theorem2.2} for the detail.

Worth mentioning that we can not check whether the proposed identifiability conditions are satisfied in practice since we do not have access to the real underlying system parameters in practical applications. Nevertheless, studying the identifiability conditions of the dynamical system provides us with a better understanding of the system. Moreover, in practice, we can use models satisfying the conditions (through constrained parameter estimation) to learn the real-world data to ensure the learned model's identifiability.

\noindent\textit{\bf Remark} If the ODE system is identifiable from a set of discrete observations sampled from a single trajectory, the system can also be identifiable from the whole corresponding trajectory.

\subsection{Identifiability Condition from Degraded Observations}
In practice, there are cases where we have no records of the original observations of the ODE system and can only access the degraded observations instead. Is the ODE system still identifiable? Furthermore, will the parameter change? We are interested in addressing the problem of identifiability under degraded observations. In particular, we are interested in the aggregated and time-scaled observations. 

According to Theorem~\ref{theorem:identifiability from discrete observations}, under assumptions A1 and A2, the ODE system (\ref{eq:ODE model}) is $(\boldsymbol{x}_0, A)$ identifiable from $n$ equally-spaced error-free observations $\boldsymbol{X}:=\{\boldsymbol{x}_1, \boldsymbol{x}_2, \ldots, \boldsymbol{x}_n\}$, for any $n > d$. Based on these $n$ observations, in the following, we define the aggregated and time-scaled observations.

\begin{definition}[aggregated observations]\label{def:aggregated} 
Let $k\geq 2$ be an integer and $\tilde{n} = \lfloor n/k \rfloor$, where $\lfloor \cdot \rfloor$ denotes the floor function. For each $j=1,2,\ldots,\tilde{n}$, we define $\Tilde{\boldsymbol{x}}_j := \big(\boldsymbol{x}_{(j-1)k+1} + \boldsymbol{x}_{(j-1)k+2} + \cdots + \boldsymbol{x}_{jk}\big)/k$, the average values of $k$ consecutive, non-overlapping observations in $\boldsymbol{X}$, starting from time $\tilde{t}_j:=t_{(j-1)k+1}$. We call $\Tilde{\boldsymbol{X}} := \{\Tilde{\boldsymbol{x}}_1,\Tilde{\boldsymbol{x}}_{2},\ldots,\Tilde{\boldsymbol{x}}_{\tilde{n}}\}$ a set of aggregated observations from $\boldsymbol{X}$.
\end{definition}

For notational simplicity, we use the same notation for aggregated and time-scaled observations from here on. 

\begin{definition}[time-scaled observations]\label{def:time-scaled}
Let $k>0$ be a constant and $\tilde{n}=n$. For each $j=1,2,\ldots,\tilde{n}$, defining $\Tilde{t}_j := kt_j$ be the scaled time, the time-scaled observation at time $\Tilde{t}_j$, denoted by $\Tilde{\boldsymbol{x}}_j$, equals $\boldsymbol{x}_j$. We call $\Tilde{\boldsymbol{X}} := \{\Tilde{\boldsymbol{x}}_1,\Tilde{\boldsymbol{x}}_{2},\ldots,\Tilde{\boldsymbol{x}}_{\tilde{n}}\}$ the set of time-scaled observations from $\boldsymbol{X}$ at the scaled time grid $\{\tilde{t}_1,\ldots,\tilde{t}_{\tilde{n}}\}$.
\end{definition}

\noindent\textit{\bf Remark} Aggregated/time-scaled observations $\Tilde{\boldsymbol{X}}$ follow new ODE systems as (\ref{eq:ODE model}) but with different initial condition and parameter matrix, denoted respectively by $\tilde{\boldsymbol{x}}_0$ and $\tilde{A}$,  and we call the new ODE systems aggregated/time-scaled ODE systems.

Having defined the aggregated/time-scaled observations, we now derive the conditions for identifying the original ODE system with parameters $\boldsymbol{x}_0$ and $A$ from them. In addition to conditions A1 and A2, a common identifiability condition from aggregated/time-scaled observations is
\begin{enumerate}
    \item[A3] The sample size of the aggregated/time-scaled observations $\tilde{n} > d$.
\end{enumerate}


\subsubsection{Identifiability condition from aggregated observations}

Though the identifiability of vector auto-regressive model from aggregated observations has been studied \citep{silvestrini2008temporal,gong2017causal}, the identifiability condition of ODEs from aggregated observations remains unknown. Based on Theorem \ref{theorem:identifiability from discrete observations}, we derive the following corollary.

\begin{corollary}[aggregated observations]\label{corollary:aggregate}
If conditions \textnormal{A1-A3} are satisfied, where $\boldsymbol{x}_0 \in \mathbb{R}^d$ and $A\in \mathbb{R}^{d\times d}$, then the aggregated ODE system parameters $\Tilde{\boldsymbol{x}}_0$ and $\Tilde{A}$ can be uniquely determined by the aggregated observations $\Tilde{\boldsymbol{X}}$, and the original  ODE system \eqref{eq:ODE model} (with parameters $\boldsymbol{x}_0$ and $A$) is $(\boldsymbol{x}_0, A)$ identifiable from the aggregated observations $\Tilde{\boldsymbol{X}}$, with $\boldsymbol{x}_0 = k(I+e^{A\Delta_t} + \cdots + e^{A(k-1)\Delta_t})^{-1}\tilde{\boldsymbol{x}}_0$ and $A = \Tilde{A}$.
\end{corollary}

The proof of Corollary~\ref{corollary:aggregate} can be found in Appendix~\ref{proof:corollary2.2.1}. This corollary implies that the ODE system (\ref{eq:ODE model}) is still identifiable from the aggregated observations under mild conditions. Moreover, the new parameter matrix corresponding to the aggregated ODE system is the same as that of the actual model, that is, $\tilde{A} = A$.

\subsubsection{Identifiability condition from time-scaled observations}
Time plays a critical role in ODE systems. However, in practical applications, using the actual time of the system directly may cause inconvenience. A common practice is to use the method defined in Definition~\ref{def:time-scaled} to scale the actual timeline into a fixed one, such as [0,1], to simplify the calculation \citep{rubanova2019latent, langtangen2016scaling}. How will the time scaling affect the causal relationship between variables, that is, parameter matrix $A$? This motivates us to derive the following corollary. 

\begin{corollary}[time-scaled observations]\label{corollary:timescaled}
If conditions \textnormal{A1-A3} are satisfied, where $\boldsymbol{x}_0 \in \mathbb{R}^d$ and $A\in \mathbb{R}^{d\times d}$, then the time-scaled ODE system parameters $\Tilde{\boldsymbol{x}}_0$ and $\Tilde{A}$ can be uniquely determined by the time-scaled observations $\Tilde{\boldsymbol{X}}$, and the original ODE system \eqref{eq:ODE model} (with parameters $\boldsymbol{x}_0$ and $A$) is $(\boldsymbol{x}_0, A)$ identifiable from the time-scaled observations $\Tilde{\boldsymbol{X}}$, with $\boldsymbol{x}_0 = \Tilde{\boldsymbol{x}}_0$ and $A = k\Tilde{A}$.
\end{corollary}

The proof of Corollary~\ref{corollary:timescaled} can be found in Appendix~\ref{proof:corollary2.2.2}. This corollary implies that the ODE system (\ref{eq:ODE model}) is identifiable from the time-scaled observations, and the new parameter matrix corresponding to the time-scaled ODE system is reduced by a factor of $k$ from the parameter matrix corresponding to the actual model, that is, $\Tilde{A} = A/k$. This corollary provides theoretical support for time scaling and implies that one can safely scale the actual timeline into a fixed one for the ODE system (\ref{eq:ODE model}).

\section{Asymptotic Properties of the NLS Estimator}\label{section:large_sample}

Now that we have established the sufficient condition for the identifiability of the ODE system from discrete error-free observations. However, in practical applications, the observations are typically disturbed by measurement noise. In this case, one can not calculate the true unknown parameters explicitly from  $e^{A \Delta t}=\boldsymbol{X}_2\boldsymbol{X}_1^{-1}$ and $\boldsymbol{x}_0 = e^{-At_1}\boldsymbol{x}_1$. Instead, we resort to parameter estimation procedures to estimate the unknown parameters. In this section, we will investigate the asymptotic properties of the parameter estimator based on the Nonlinear Least Squares (NLS) method. First, we introduce the measurement model and the NLS method.

\noindent{\bf Measurement model.} Suppose the system state $\boldsymbol{x}(t)$ in ODE system (\ref{eq:ODE model}) is measured with noise at time points $t_1,\ldots,t_n$, with $t_i \in [0,T]$ for all $i=1, \ldots, n$, and $0<T<+\infty$. Abusing notation a bit, from now on, we use $\boldsymbol{\theta} := (\boldsymbol{x}_0, A)\in \mathbb{R}^{d+d^2}$ to vectorize the parameters $\boldsymbol{x}_0, A$. We further let $\Theta := (M^0, \Omega)$ to denote the parameter space,  where $M^0 \subset \mathbb{R}^d$ and $\Omega \subset \mathbb{R}^{d\times d}$. The true parameter is denoted as $\boldsymbol{\theta}^* := ( \boldsymbol{x}_0^*, A^*)$. Then the measurement model can be described as: 
\begin{equation}\label{eq:measurement model}
\begin{split}
    \boldsymbol{y}_i &= \boldsymbol{x}(t_i;\boldsymbol{\theta}^*) + \boldsymbol{\epsilon}_i= e^{A^*t_i}\boldsymbol{x}_0^* + \boldsymbol{\epsilon}_i\,,
\end{split}
\end{equation}
for all $i=1,\ldots,n$, where $\boldsymbol{y}_i \in \mathbb{R}^d$ denotes the noisy observation at time $t_i$ and $\boldsymbol{\epsilon}_i \in \mathbb{R}^d$ is the measurement error at time $t_i$.

\noindent{\bf Nonlinear least squares (NLS) method} is well used for parameter estimation in nonlinear regression models including ODEs \citep{jennrich1969asymptotic,xue2010sieve}. In the following, based on the identifiability condition we build in Section \ref{sec:identifiability}, with mild assumptions, we show the consistency and asymptotic normality of the NLS estimator.

Suppose the ODE system \eqref{eq:ODE model} is $(\boldsymbol{x}_0^*, A^*)$ identifiable from a set of equally-spaced error-free observations sampled from a single trajectory: $\boldsymbol{x}_1, \ldots, \boldsymbol{x}_n$, with $\boldsymbol{x}_i = \boldsymbol{x}(t_i; \boldsymbol{\theta}^*) = e^{A^* t_i}\boldsymbol{x}_0^*$ and $t_i \in [0,T]$ for all $i=1,\ldots,n$, then the true parameter $\boldsymbol{\theta}^* = ( \boldsymbol{x}_0^*, A^*)$ uniquely minimizes:
\begin{equation}\label{eq:M(theta)}
    M(\boldsymbol{\theta}) = \cfrac{1}{T}\int_0^T \parallel e^{A^*t}\boldsymbol{x}_0^*-e^{At}\boldsymbol{x}_0\parallel_2^2 dt\, ,
\end{equation}
where $\parallel \cdot \parallel_2$ denotes the Euclidean norm. The proof is straightforward, since the ODE system \eqref{eq:ODE model} is $(\boldsymbol{x}_0^*, A^*)$ identifiable from $\boldsymbol{x}_1, \ldots, \boldsymbol{x}_n$, the ODE system \eqref{eq:ODE model} is also $(\boldsymbol{x}_0^*, A^*)$ identifiable from the corresponding trajectory at time $[0,T]$, which implies that $M(\boldsymbol{\theta})$ attains its unique global minimum at $\boldsymbol{\theta}^*$.
In practical applications, typically, one can only access the noisy observation $\boldsymbol{y}_i$'s as described in \eqref{eq:measurement model}. Therefore, we propose to estimate $\boldsymbol{\theta}^*$ by minimizing the empirical version of $M(\boldsymbol{\theta})$, which is
\begin{equation}\label{eq:Mn(theta)}
    M_n(\boldsymbol{\theta}) 
    = \cfrac{1}{n} \sum_{i=1}^n \parallel \boldsymbol{y}_i - e^{At_i}\boldsymbol{x}_0 \parallel_2^2 \, .
\end{equation}
That is, the NLS estimator of $\boldsymbol{\theta}^*$ is defined as
\begin{equation*} 
	\hat{\boldsymbol{\theta}}_n = \arg\min_{\boldsymbol{\theta}\in \Theta}M_n(\boldsymbol{\theta})\,.
\end{equation*}

\noindent{\bf Assumptions.} Now we investigate the asymptotic properties of the NLS estimator $\hat{\boldsymbol{\theta}}_n$. We first list all the required assumptions:
\begin{enumerate}
    \item[A4] Parameter space $\Theta$ is a compact subset of $\mathbb{R}^{d+d^2}$.
    \item[A5] Error terms $\{\boldsymbol{\epsilon}_i\}$ for $i=1,\ldots, n$ are independent and identically distributed random vectors with mean zero and covariance matrix $\Sigma = \text{diag}(\sigma_1^2, \ldots, \sigma_d^2)$, where $0<\sigma_j^2 < \infty$ for all $j=1,\ldots, d$.
    \item[A6] We have $n$ equally-spaced observations $\boldsymbol{Y} := \{\boldsymbol{y}_1, \ldots, \boldsymbol{y}_n\}$, where $\boldsymbol{y}_i$ is defined by measurement model (\ref{eq:measurement model}). Without loss of generality, we assume observation time starts with $t_1 = 0$, ends with $t_n = T$, and thus the equal time space $\Delta t=
   T/(n-1)$.
    \item[A7] $\boldsymbol{\theta}^*$ is an interior point of the parameter space $\Theta$.
\end{enumerate}

In addition to the aforementioned assumptions A4-A7, assumptions A1 and A2 stated in Theorem \ref{theorem:identifiability from discrete observations} are required with respect to the true parameter $\boldsymbol{\theta}^* = \{ \boldsymbol{x}_0^*, A^*\}$. These two assumptions guarantee the ODE system (\ref{eq:ODE model}) is $(\boldsymbol{x}_0^*, A^*)$ identifiable from any $d+1$ equally-spaced error-free observations sampled from the trajectory $\boldsymbol{x}(\cdot; \boldsymbol{x}_0^*, A^*)$. A1 and A2 are needed because the identifiability of the system is a prerequisite for obtaining a consistent parameter estimator. A4 is commonly used in deriving consistency for parameter estimators, as demonstrated in references such as \citep[Thm 2.1, Thm 2.5, Thm 2.6]{newey1994large}. While alternative conditions for consistency exist without the compactness assumption, they typically require the objective function to be convex \citep{newey1994large}. Given that our objective function $M_n(\boldsymbol{\theta})$ is not convex, the compactness assumption remains indispensable. The compactness assumption implicitly requires having known bounds on the true parameter values, which can be challenging to check in real-world situations since the true parameters of the ODE systems are unknown. However, since our derived confidence sets do not depend on the bounds of the parameter values, we may safely assume enormously large parameter boundaries. Consequently, in practical applications, there is no need to verify this assumption, and the derived statistical inferences will remain unaffected. A5 is a common way to define measurement noise, note that we do not require the error terms follow a normal distribution. It is worth mentioning that the error terms are not necessarily identically distributed. In other words, one can easily generate our theoretical results for cases that only require the error terms being independently distributed. A6 ensures the observations are collected at equally-spaced time points, following the rules of observations collection required by Theorem \ref{theorem:identifiability from discrete observations}. Assumption A7 is a standard condition for proving asymptotic normality \citep{kundu1993asymptotic,mira1995nonlinear,xue2010sieve}.

\subsection{Consistency}
In this subsection, we study the consistency of the NLS estimator. An estimator is said to be consistent, if it converges in probability to the true value of the parameter.
\begin{theorem}\label{theorem:consistency}
Suppose assumptions \textnormal{A1, A2} are satisfied with respect to $\boldsymbol{\theta}^*$ and assumptions \textnormal{A4-A6} hold, the NLS estimator $\hat{\boldsymbol{\theta}}_n \xrightarrow{p} \boldsymbol{\theta}^*$, as $n\rightarrow \infty$.
\end{theorem}

The notation $\xrightarrow{p}$ stands for convergence in probability. The proof of Theorem~\ref{theorem:consistency} can be found in Appendix~\ref{proof:theorem3.1}. This theorem shows the consistency of our NLS estimator $\hat{\boldsymbol{\theta}}_n$ to the true $\boldsymbol{\theta}^*$. That is, as $n$ goes to infinity, the NLS estimator $\hat{\boldsymbol{\theta}}_n$ converges to the true system parameters $\boldsymbol{\theta}^*$ with probability approaching 1. Note that the consistency of the estimator is a necessary condition for the estimator's asymptotic normality.

In the case where we only observe degraded data, we let $\Tilde{\boldsymbol{Y}} := (\Tilde{\boldsymbol{y}}_1, \Tilde{\boldsymbol{y}}_2, \ldots, \Tilde{\boldsymbol{y}}_{\tilde{n}})$ denote the noisy aggregated/time-scaled data, which are collected from the original observations $\boldsymbol{Y}$ in the same way as the corresponding error-free observations $\Tilde{\boldsymbol{X}}$ from $\boldsymbol{X}$ defined in Definition~\ref{def:aggregated}/\ref{def:time-scaled}. Then we can estimate the corresponding parameters $\tilde{\boldsymbol{\theta}}^*:=(\tilde{\boldsymbol{x}}^*_0, \tilde{A}^*)$ by minimizing the NLS objective function in \eqref{eq:Mn(theta)}, with the data replaced by $\tilde{\boldsymbol{Y}}$. Let such an estimator be $\hat{\tilde{\boldsymbol{\theta}}}:= (\hat{\tilde{\boldsymbol{x}}}_0,\hat{\tilde{A}})$. Using the relationship between $\tilde{\boldsymbol{\theta}}^*$ and $\boldsymbol{\theta}^*$ found in Corollary~\ref{corollary:aggregate}/\ref{corollary:timescaled}, we can define a mapping $g:\mathbb{R}^{d+d^2}\rightarrow\mathbb{R}^{d+d^2}$ from $\tilde{\boldsymbol{\theta}}^*$ to $\boldsymbol{\theta}^*$, and obtain an estimator of $\boldsymbol{\theta}^*$ by $\hat{\boldsymbol{\theta}}_{\tilde{n}}:=g(\hat{\tilde{\boldsymbol{\theta}}})$.
In the following, we present the expressions of $g$ and the consistency of the NLS estimators by using the aggregated and time-scaled observations, respectively.

\begin{corollary}[aggregated observations]\label{corollary:consistency aggregate}
Suppose assumptions \textnormal{A1, A2} are satisfied with respect to $\boldsymbol{\theta}^*$ and assumptions \textnormal{A3-A6} hold, the NLS estimator $\hat{\boldsymbol{\theta}}_{\tilde{n}}\xrightarrow{p} \boldsymbol{\theta}^*$, as $\tilde{n}\rightarrow \infty$, where $\hat{\boldsymbol{\theta}}_{\tilde{n}}:=g(\hat{\tilde{\boldsymbol{\theta}}}) := \big(k(I+e^{\hat{\tilde{A}}\Delta_t} + \cdots + e^{\hat{\tilde{A}}(k-1)\Delta_t})^{-1}\hat{\tilde{\boldsymbol{x}}}_0, \hat{\tilde{A}}\big)$.
\end{corollary}

\begin{corollary}[time-scaled observations]\label{corollary:consistency timescaled}
Suppose assumptions \textnormal{A1, A2} are satisfied with respect to $\boldsymbol{\theta}^*$ and assumptions \textnormal{A3-A6} hold, the NLS estimator $\hat{\boldsymbol{\theta}}_{\tilde{n}}\xrightarrow{p} \boldsymbol{\theta}^*$, as $\tilde{n}\rightarrow \infty$, where $\hat{\boldsymbol{\theta}}_{\tilde{n}}:=g(\hat{\tilde{\boldsymbol{\theta}}}) := (\hat{\tilde{\boldsymbol{x}}}_0, k\hat{\tilde{A}})$.
\end{corollary}

The proofs of Corollary~\ref{corollary:consistency aggregate} and Corollary~\ref{corollary:consistency timescaled} can be found in Appendix~\ref{proof:corollary3.1.1} and Appendix~\ref{proof:corollary3.1.2}, respectively. Now in addition to the assumptions mentioned in Theorem~\ref{theorem:consistency}, assumption A3 is also required for the consistency from degraded observations. These two corollaries show that the NLS estimators obtained from degraded observations are consistent to the true parameters of the original ODE \eqref{eq:ODE model}. Since we have established the identifiability conditions for the original system parameters of the ODE (\ref{eq:ODE model}) from the degraded observations in Corollary~\ref{corollary:aggregate}/\ref{corollary:timescaled} in Section \ref{sec:identifiability}, the consistency of their NLS estimators is a natural result from Theorem~\ref{theorem:consistency}. To see this, we first show that the NLS estimator from aggregated/time-scaled observations converges to the true system parameter corresponding to the new ODE system, that is, $\hat{\tilde{\boldsymbol{\theta}}}\xrightarrow{p} \tilde{\boldsymbol{\theta}}^*$, as $\tilde{n}\rightarrow \infty$. Since we have derived the mapping $g$ such that $g(\tilde{\boldsymbol{\theta}}^*) = \boldsymbol{\theta}^*$ in Corollary~\ref{corollary:aggregate}/\ref{corollary:timescaled}, then by multivariate continuous mapping theorem, one takes the function $g(\cdot)$ with respect to $\hat{\tilde{\boldsymbol{\theta}}}$ and $\tilde{\boldsymbol{\theta}}^*$, respectively, one can reach the conclusion $g(\hat{\tilde{\boldsymbol{\theta}}})\xrightarrow{p} g(\tilde{\boldsymbol{\theta}}^*)$, as $\tilde{n}\rightarrow \infty$. That is, $\hat{\boldsymbol{\theta}}_{\tilde{n}}\xrightarrow{p} \boldsymbol{\theta}^*$, as $\tilde{n}\rightarrow \infty$.

\subsection{Asymptotic Normality}\label{section:asymptotic normality}
After establishing the consistency of our NLS estimators, we can study their asymptotic distributions.

\begin{theorem}\label{theorem:Asymptotic normality}
Suppose assumptions \textnormal{A1, A2} are satisfied with respect to $\boldsymbol{\theta}^*$ and assumptions \textnormal{A4-A7} hold, we have $\sqrt{n}(\hat{\boldsymbol{\theta}}_n -\boldsymbol{\theta}^* )\xrightarrow{d} N(\boldsymbol{0}, H^{-1}VH^{-1}) $, as $n\rightarrow \infty$, where $H= \nabla_{\boldsymbol{\theta}}^2M(\boldsymbol{\theta}^*)$, is the Hessian matrix of $M(\boldsymbol{\theta})$ at $\boldsymbol{\theta}^*$ and $ V = \lim_{n\rightarrow \infty}var(\sqrt{n}\nabla_{\boldsymbol{\theta}} M_n(\boldsymbol{\theta}^*))$,  with $\nabla_{\boldsymbol{\theta}} M_n(\boldsymbol{\theta}^*)$ the gradient of $M_n(\boldsymbol{\theta})$ at $\boldsymbol{\theta}^*$.  $M(\boldsymbol{\theta})$ and $M_n(\boldsymbol{\theta})$ are defined in Equations \textnormal{(\ref{eq:M(theta)})} and \textnormal{(\ref{eq:Mn(theta)})}, respectively.
\end{theorem}

The notation $\xrightarrow{d}$ stands for convergence in distribution. The proof of Theorem~\ref{theorem:Asymptotic normality} can be found in Appendix~\ref{proof:theorem3.2}. From the theorem, we see that the NLS estimator is asymptotically normal, and the convergence rate is $n^{-1/2}$. Here, the rate meets the one of the standard parametric NLS estimator \citep{jennrich1969asymptotic, wu1981asymptotic}. This result is reasonable because our model in \eqref{eq:measurement model} is a parametric one.

Now, with the asymptotic normality result, if we can estimate the asymptotic covariance matrix $\Sigma^*:= H^{-1}VH^{-1}$, we can perform statistical inference for the unknown system parameters. In particular, we derive the explicit expressions of matrices $H$ and $V$ in \eqref{eq:H} and \eqref{eq:V_RR^T}. According to their formulae, they are functions of the true system parameter $\boldsymbol{\theta}^*$, which are unknown in practice. Therefore, we approximate $H$ and $V$ by substituting $\boldsymbol{\theta}^*$ with the NLS parameter estimate $\hat{\boldsymbol{\theta}}_n$ in their formulae. Then the inference can be performed based on the approximated covariance matrix, denoted by $\hat{\Sigma}_n$. In the following two subsections, we introduce the details of the inference. For those who are not familiar with statistical inference, please refer to \citep{borokov2018mathematical, shao2003mathematical} for relevant concepts and methods.

\subsubsection{Confidence sets for unknown parameters}\label{confidence sets}
Based on Theorem \ref{theorem:Asymptotic normality}, we can derive the approximate confidence sets for the unknown true parameters $\boldsymbol{\theta}^*$. In our paper, we employ the term ``confidence set" as a general descriptor for a range of values within which we have a certain level of confidence that the true parameters reside. This umbrella term encompasses two specific types of confidence sets: confidence intervals (CIs) and confidence regions (CRs). When the parameter under consideration is one-dimensional, we refer to the confidence set as a confidence interval (CI). When the parameter is multidimensional, we use the term confidence region (CR) to describe the confidence set. For the detailed theoretical definition, please refer to \citep{borokov2018mathematical}.

We first construct the simultaneous confidence region (CR) for all the $d+d^2$ parameters $\boldsymbol{\theta}^*_i$, where $\boldsymbol{\theta}^*_i$ denotes the $i$th entry of $\boldsymbol{\theta}^*$, and $i=1,\ldots,d+d^2$. According to Theorem \ref{theorem:Asymptotic normality}, we have
\begin{equation*}  
	n(\hat{\boldsymbol{\theta}}_n -  \boldsymbol{\theta}^*)^{\top} \{\Sigma^*\}^{-1} (\hat{\boldsymbol{\theta}}_n -  \boldsymbol{\theta}^*) \xrightarrow{d} \chi_{d+d^2}^2\,, \mbox{ as } n\rightarrow \infty\,.
\end{equation*}

Therefore, we approximate the $1-\alpha$ CR by the set 
\begin{equation}\label{eq:CR}
  \bigg\{\boldsymbol{\theta}: \quad n(\hat{\boldsymbol{\theta}}_{n} -  \boldsymbol{\theta})^{\top} \hat{\Sigma}_n^{-1} (\hat{\boldsymbol{\theta}}_{n} -  \boldsymbol{\theta}) \leq \chi_{d+d^2}^2(1-\alpha)\bigg\}\,,
\end{equation}
where $\chi_{m}^2(1-\alpha)$ denotes the upper-tail critical value of $\chi^2$ distribution with $m$ degrees of freedom at significance level $\alpha$. Please note that in statistical hypothesis testing, the significance level $\alpha$ represents the probability of rejecting the null hypothesis when it is actually true. Typically set to $5\%$ or lower depending on the field of study, $\alpha$ is chosen by the experimenter. In the meantime, $1-\alpha$ denotes the corresponding confidence level, where CI/CR contains the true parameter value $(1-\alpha)\%$ of the time.

We then construct a pointwise confidence interval (CI) for each $\boldsymbol{\theta}^*_i$, $i=1,\ldots,d+d^2$. Based on Theorem $\ref{theorem:Asymptotic normality}$, we can derive that 
\begin{equation*}
	\sqrt{n}(\hat{\boldsymbol{\theta}}_{ni} - \boldsymbol{\theta}^*_i) \xrightarrow{d} N(0, D_i(\Sigma^*))\,, \mbox{ as } n\rightarrow \infty\,,
\end{equation*}
where $\hat{\boldsymbol{\theta}}_{ni}$ denotes the $i$th entry of $\hat{\boldsymbol{\theta}}_n$ and $D_i(M)$ denotes the $i$th diagonal entry of matrix $M$. Then, the $1-\alpha$ CI for each $\boldsymbol{\theta}_i^*$ can be estimated by
\begin{equation}\label{eq:CI}
    \bigg[\hat{\boldsymbol{\theta}}_{ni} - z_{\alpha/2}\sqrt{D_i(\hat{\Sigma}_n)/n}\,, \hspace{0.3
    cm} \hat{\boldsymbol{\theta}}_{ni} + z_{\alpha/2}\sqrt{D_i(\hat{\Sigma}_n)/n}\bigg]\,,
\end{equation}
where $z_{\alpha/2}$ is the two-tailed critical value of the standard normal distribution at significance level $\alpha$. That is, if $Z\sim N(0, 1)$, then $\mathbb{P}(Z < -z_{\alpha/2}) + \mathbb{P}(Z >z_{\alpha/2}) = \alpha$.

\subsubsection{Infer the causal structure of the ODE system}\label{hypothesis test}
Another application of Theorem \ref{theorem:Asymptotic normality} is to infer the causal structure among variables within the ODE system, specifically by testing the hypothesis $a_{jk}^* = 0$, where $a_{jk}^*$ denotes the $jk$-th entry of the true parameter matrix $A^*$, with $j,k=1,\ldots,d$. When $a_{jk}^* \neq 0$, as delineated by the expression of the ODE system \eqref{eq:ODE model}, the derivative of $x_j(t)$ is influenced by $x_k(t)$, implying a causal link from variable $x_k$ to variable $x_j$, as referenced in \citet{scholkopf2021toward}. Here, $x_j$ denotes the $j$-th variable of the ODE system, and $x_j(t)$ denotes the state of the $j$-th variable at time $t$. Note that the ODE system is fully observable, that is, there are no latent variables interacting with the system.

Then based on the derived $1-\alpha$ CI for each parameter in \eqref{eq:CI}, we propose to conduct a hypothesis test:
\begin{equation}
    H_0: a_{jk}^* = 0 \hspace{0.2cm} \mbox{vs} \hspace{0.2cm} H_1: a_{jk}^* \neq 0 
\end{equation}
by assessing the following inequality:
\begin{equation}\label{eq:test ajk}
    |\hat{a}_{jk}| > z_{\alpha/2}\sqrt{D_{d+(j-1) d + k}(\hat{\Sigma}_n)/n}\,,
\end{equation}
where $\hat{a}_{jk}$ denotes the estimator of $a_{jk}^*$, which is the test statistic and equals the $d+(j-1)d+k$-th entry of $ \hat{\boldsymbol{\theta}}_n$.
If \eqref{eq:test ajk} holds, we have significant evidence to reject the null hypothesis at the significance level $\alpha$, and conclude that $a_{jk}^* \neq 0$, thereby affirming a causal link from variable $x_k$ to variable $x_j$.

\subsubsection{Asymptotic normality of NLS estimators from degraded observations}

Here we present the asymptotic normality of NLS estimators from degraded observations. Recalling the transformation rules from $\tilde{\boldsymbol{\theta}}^*$ to $\boldsymbol{\theta}^*$, $g$, defined in Corollary~\ref{corollary:consistency aggregate}/\ref{corollary:consistency timescaled}, we denote its gradient at $\tilde{\boldsymbol{\theta}}^*$ by $G:=\nabla g(\tilde{\boldsymbol{\theta}}^*)$. We have derived the explicit formulae of $H(T,\boldsymbol{\theta}^*)$ and $V(T, \boldsymbol{\theta}^*)$ as matrix functions of $T$ and $\boldsymbol{\theta}^*$ in ~\eqref{eq:H}/\eqref{eq:V_RR^T}. Then we establish the following corollaries.

\begin{corollary}[aggregated observations]\label{corollary:normality aggregate}
Suppose assumptions \textnormal{A1, A2} are satisfied with respect to $\boldsymbol{\theta}^*$, assumptions \textnormal{A3-A6} hold and assumption \textnormal{A7} is satisfied with respect to $\tilde{\boldsymbol{\theta}}^*$, we have $\sqrt{\tilde{n}}(\hat{\boldsymbol{\theta}}_{\tilde{n}} -\boldsymbol{\theta}^* )\xrightarrow{d} N(\boldsymbol{0}, G\tilde{H}^{-1}\tilde{V} \tilde{H}^{-1}G^{\top}) $, as $\tilde{n}\rightarrow \infty$, where $\hat{\boldsymbol{\theta}}_{\tilde{n}}$ is defined in Corollary~\ref{corollary:consistency aggregate}, $\tilde{H} = H(\tilde{T},\tilde{\boldsymbol{\theta}}^*)$ and $\tilde{V} = V(\tilde{T},\tilde{\boldsymbol{\theta}}^*)/k$, with $\tilde{T} = (\lfloor 
n/k\rfloor-1)kT/(n-1)  $ and $\tilde{\boldsymbol{\theta}}^* = (\tilde{\boldsymbol{x}}_0^*, \tilde{A}^*) = \big((I+e^{A^*\Delta_t} + \cdots + e^{A^*(k-1)\Delta_t})\boldsymbol{x}_0^*/k, A^*\big)$.
\end{corollary}
 
\begin{corollary}[time-scaled observations]\label{corollary:normality timescaled}
Suppose assumptions \textnormal{A1, A2} are satisfied with respect to $\boldsymbol{\theta}^*$, assumptions \textnormal{A3-A6} hold and assumption \textnormal{A7} is satisfied with respect to $\tilde{\boldsymbol{\theta}}^*$, we have $\sqrt{\tilde{n}}(\hat{\boldsymbol{\theta}}_{\tilde{n}} -\boldsymbol{\theta}^* )\xrightarrow{d} N(\boldsymbol{0}, G\tilde{H}^{-1}\tilde{V} \tilde{H}^{-1}G^{\top}) $, as $\tilde{n}\rightarrow \infty$, where $\hat{\boldsymbol{\theta}}_{\tilde{n}}$ is defined in Corollary~\ref{corollary:consistency timescaled}, $\tilde{H} = H(kT,\tilde{\boldsymbol{\theta}}^*)$ and $\tilde{V} = V(kT,\tilde{\boldsymbol{\theta}}^*)$, with $\tilde{\boldsymbol{\theta}}^* = (\tilde{\boldsymbol{x}}_0^*, \tilde{A}^*) =(\boldsymbol{x}_0^*, A^*/k)$. 
\end{corollary}

The proofs of Corollary~\ref{corollary:normality aggregate} and Corollary~\ref{corollary:normality timescaled} can be found in Appendix~\ref{proof:corollary3.2.1} and Appendix~\ref{proof:corollary3.2.2}, respectively.
With the consistency property of the NLS estimators, these two corollaries can be directly derived from Theorem \ref{theorem:Asymptotic normality} by using multivariate Delta method. The explicit formulae for matrices $G$, $\tilde{H}$ and $\tilde{V}$ for aggregated/time-scaled observations are derived in the proofs.

Worth to be noted that matrix $\tilde{V}$ in Corollary~\ref{corollary:normality aggregate} is not $V(T, \boldsymbol{\theta}^*)$ by substituting $T$ and $\boldsymbol{\theta}^*$ with $\tilde{T}$ and $\tilde{\boldsymbol{\theta}}^*$, but also reduced by a factor of $k$. The reason is that, the equation of $V$ includes the variance matrix $\Sigma$ of the error term $\boldsymbol{\epsilon}_i$ in~\eqref{eq:V_RR^T}. By the generation rules of aggregated observations defined in Definition~\ref{def:aggregated}, the variance of the aggregated noise term $\tilde{\boldsymbol{\epsilon}}_i$ becomes $k$ times smaller than that of the original one, that is $\tilde{\Sigma} = \Sigma/k$. And thanks to the reduced variance, we will show that the parameter estimates of aggregated observations can reach the same level of accuracy as that of the original observations with a much smaller sample size in the simulation results in subsection~\ref{simulation:aggregated}.

Now that we have derived the asymptotic normality results from aggregated/time-scaled observations. We can perform statistical inference for the unknown original system parameters $\boldsymbol{\theta}^*$ using the same way introduced in subsection~\ref{confidence sets} and~\ref{hypothesis test}.

\section{Simulations}\label{sec:simulation}
In this section, we illustrate the theoretical results established in Section \ref{section:large_sample} by simulation .

\subsection{Data Simulation}\label{subsec:data simulation}
For each $d=2,3,4$, we first randomly generate a $d\times d$ parameter matrix $A_d^*$ and a $d\times 1$ initial condition $\boldsymbol{x}_{0d}^*$ as the true system parameters for each $d$-dimensional ODE system (\ref{eq:ODE model}). Moreover, to test whether $a_{jk}^* = 0$, we randomly set several entries to zero in each $A_d^*$. Without loss of generality, we set $T=1$. Then $n$ equally-spaced noisy observations are generated based on Equation~\eqref{eq:measurement model} in $[0, 1]$ time interval with error term $\boldsymbol{\epsilon}_i\sim N(\boldsymbol{0}, \text{diag}(0.05^2, \ldots, 0.05^2))$. We tested various sample sizes for each $d$-dimensional ODE system. For each configuration, we run 200 random replications. 
The $A_d^*$ and $\boldsymbol{x}_{0d}^*$ are shown below.\\
$A_2^* = \begin{bmatrix} 
    1.76 & -0.1\\
    0.98 & 0
    \end{bmatrix}$,
$A_3^* = \begin{bmatrix} 
    1.76 & 0 & 0.98\\
    2.24 & 0 & -0.98\\
    0.95 & 0 & -0.1
    \end{bmatrix}$,    
 $A_4^* = \begin{bmatrix} 
    1.76 & 0.9 & 0 & 2.24\\
    1.87 & -0.98 & 0 & -1.15\\
    -1.1 & 0 & 0.64 & 0\\
    1.26 & 0.12 & 0.94 & 0
    \end{bmatrix}$,    
    
\noindent and $\boldsymbol{x}_{02}^*=[1.87, -0.98]^{\top}$, 
$\boldsymbol{x}_{03}^* = [0.41,0.14,1.45]^{\top}$,
$\boldsymbol{x}_{04}^* = [-0.42,1.01,1.97,-0.38]^{\top}$.

Note that since the NLS loss function~\eqref{eq:Mn(theta)} is a non-convex function, in practical application, one may require a global optimization technique to obtain the NLS estimates. However, in this paper, we focus on the theoretical statistical properties analysis of the NLS estimator. Theoretically, the non-convexity of the NLS loss function does not influence our derived theoretical results. Therefore, for the purpose of illustrating our theoretical results, we do not apply the global optimization technique in our simulation due to its high computational cost. Instead, we use a bound-constrained minimization technique~\citep{branch1999subspace} to obtain the NLS estimates. In order to get the global minimum NLS estimate (or a local minimum that is close enough to the global minimum), we apply two tricks when implementing the optimization method. Firstly, we initialize the parameter with a value close to the true parameter (for example, $\boldsymbol{\theta}^* - 0.001$). Secondly, we constrain the bounds of the parameter within a reasonable neighbourhood of the true parameter (for example, $[\boldsymbol{\theta}^* - 0.5,\boldsymbol{\theta}^*+0.5]$). According to the simulation results below, the attained NLS estimates are precise enough to illustrate the correctness of our theoretical results. In other words, suppose we apply a global optimization technique to obtain the NLS estimates. In that case, the simulation results will be more supportive of the correctness of our theoretical results with the help of potentially more accurate NLS estimates.

\begin{figure*}[ht]
  \centering
  \includegraphics[width=15cm]{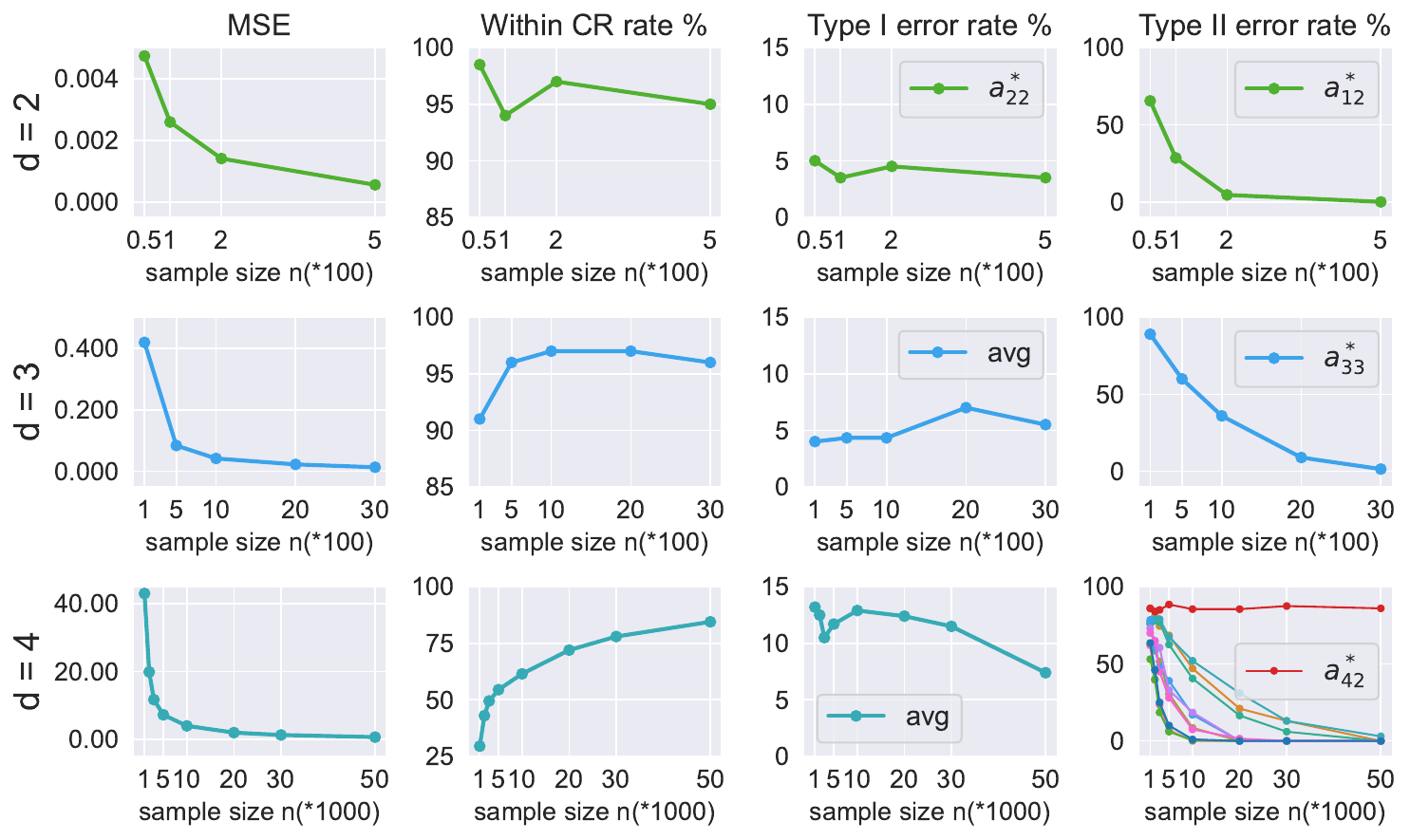}
  \caption{Simulation results for $d=2,3,4$ dimensional ODE system, respectively}
  \label{fig:simulation}
\end{figure*}

\subsection{Metrics} 
\noindent{\bf Mean Squared Error (MSE)} is introduced to check the consistency of the parameter estimator. It is defined as: 
\begin{equation*}
	\text{MSE} = \frac{1}{N}\sum_{j=1}^{N} \parallel \hat{\boldsymbol{\theta}}_{n}^{(j)}  - \boldsymbol{\theta}^*\parallel_2^2\,, 
\end{equation*}
where $\hat{\boldsymbol{\theta}}_{n}^{(j)}$ denotes the estimated parameter of the $j$th replication, and $N$ is the number of replications.

\noindent{\bf Within CR rate} is defined as the rate of replications with the true parameters $\boldsymbol{\theta}^*$ included in the CR at $95\%$ confidence level. Whether the $95\%$ CR includes $\boldsymbol{\theta}^*$ at each replication can be calculated by Equation (\ref{eq:CR}). Worth to be mentioned that here we use this metric aiming to test the correctness of our asymptotic normal theory, for example, the variance matrix $\Sigma^*$, not to infer the confidence region. Therefore, we use the true value of $\Sigma^*$ here rather than the estimated one in Equation (\ref{eq:CR}). Since we set the confidence level at $95\%$, if our theoretical results are correct, the within CR rate in our
simulation results should approach $95\%$.

\noindent{\bf Type I/II error rate} is calculated based on the hypothesis test introduced in subsection~\ref{hypothesis test}. We set the significance level $\alpha = 0.05$. Type I error rate is the rate of replications rejecting the null hypothesis, for $a_{jk}^* = 0$. And type II error rate is the rate of replications not rejecting the null hypothesis, for $a_{jk}^* \neq 0$. Whether we reject the null hypothesis or not is calculated by Equation (\ref{eq:test ajk}). Given the significance level set at $0.05$, the anticipated outcome is that if our theoretical results are correct, the Type I error rate in our simulation results should approximate $5\%$, while the Type II error rate should tend towards zero. It is noteworthy that the lower the type II error rate is, the more powerful our causal structure inference test is.

\subsection{Results Analysis}\label{results analysis}
The simulation results are presented in Figure~\ref{fig:simulation}. We first explain the legend in the figure. Since $A_2^*$ only has one zero entry $a_{22}^*$, the type I error rate is based on $a_{22}^*$. However, $A_3^*$ and $A_4^*$ have multiple zero entries, and the type I error rate for each zero entry is similar. Therefore, we show the average value of all zero entries in $A_3^*$ and $A_4^*$, labelled as avg. For type II error rate in cases with $d=2$ and $d=3$, we only show the value of $a_{12}^*$ and $a_{33}^*$, respectively. Because all other non-zero entries in $A_2^*$ and $A_3^*$ have zero or close to zero type II error rate since the sample size $n$ is small. For $d=4$, we present the results of all the $11$ non-zero entries in $A_4^*$, but due to space limitations, we only label entry $a_{42}^*$, which has a different trend from others. 

It can be seen from the first column in Figure~\ref{fig:simulation} that for all three cases where $d=2,3$ and $4$, MSE decreases and approaches zero with the increase of sample size $n$, which indicates the consistency of the estimators. As can be seen from the figure, for $d=2$ and $d=3$ cases, the within CR rate is around $95\%$, and the type I error rate is around $5\%$ for all different sample sizes $n$. Moreover, the type II error rate reduces as the sample size increases and attains or approaches zero when the sample size is large enough. This result implies the correctness of our asymptotic normality theory and indicates that the test of causal structure inference for the ODE system is powerful.

For the 4-dimensional case, we can see that with the increase of sample size, the within CR rate increases, and the type I error rate decreases, which implies that as the parameter estimates approach their true parameter values, the within CR rate and type I error rate is closer to $95\%$ and $5\%$ respectively. They do not attain their ideal values in our simulation because the parameter estimates are not precise enough under the current sample sizes due to the high dimension of the system parameters. For the same reason,  the type II error rate of entry $a_{42}^*$ keeps high. This result is reasonable because $a_{42}^*=0.12$ is close to zero. When the parameter estimate is not accurate enough, it is easy to get a wrong result that does not reject the null hypothesis. The type II error rates for $d=2$ and $d=3$ cases also support this conclusion. We can see that the absolute values of $a_{12}^*$ in $A_2^*$ and $a_{33}^*$ in $A_3^*$ are also small. Nevertheless,  as sample size increases, with the help of sufficiently accurate parameter estimates, their type II error rates approach zero.

It is worth noting that in the 4-dimensional case, the type II error rates of other entries approach zero when the sample size is much smaller compared to the case of $a_{42}^*$. This result implies that when the causal effect between variables, that is, $|a_{jk}^*|$, is significant, the causal structure can be easily and correctly discovered using our method. However, for cases where the causal effect is small or negligible, we need a sufficiently large sample to discover the causal relationship.

\subsection{Simulation Results from Degraded Observations}
In this subsection, we illustrate the corollaries built on aggregated/time-scaled observations in Section~\ref{section:large_sample} by simulation.

We chose the $d=3$ case with the true system parameters $(\boldsymbol{x}_{03}^*, A_3^*)$ the same as the one we used in subsection~\ref{subsec:data simulation}. The original noisy observations are generated using the same way we introduced in subsection~\ref{subsec:data simulation}. And then the aggregated/time-scaled observations are generated from the original ones based on Definition~\ref{def:aggregated}/Definition~\ref{def:time-scaled} with various $k$.

\subsubsection{Simulation results from aggregated observations}\label{simulation:aggregated}
In the following, we show the simulation results for aggregated observations with $k=5,10, \text{and } 20$, respectively. Moreover, to compare the results from the aggregated and the original observations, we also present the simulation results from the original observations in Table~\ref{aggregated_k1}. We use $n$ and $\tilde{n}$ to denote the sample size of the original observations and the aggregated observations, respectively.

\begin{table}[hbt!]
  \caption{Simulation results from original observations}
  \label{aggregated_k1}
  \centering
  \begin{tabular}{lllllllllllll}
    \toprule
    \multicolumn{2}{c}{Sample Size} & \multirow{2}{*}{MSE} & \multirow{2}{*}{CR }
    & \multicolumn{3}{c}{Type I Error Rate$\%$} &  \multicolumn{6}{c}{Type II Error Rate$\%$} \\
    \cmidrule(r){1-2} \cmidrule(r){5-7} \cmidrule{8-13}
     $n$ & $\tilde{n}$ &  & Rate$\%$ & $a_{12}$  & $a_{22}$  & $a_{32}$  & $a_{11}$  & $a_{13}$ & $a_{21}$  & $a_{23}$ & $a_{31}$  & $a_{33}$\\
    \midrule
    100  & - & 0.480 & 94   & 3   & 2   & 5.5 & 0 & 0 & 0 & 0 & 0 & 82.5 \\
    200  & - & 0.243 & 97.5 & 5.5 & 4.5 & 4   & 0 & 0 & 0 & 0 & 0 & 75.5 \\
    500  & - & 0.093 & 97   & 3.5 & 4   & 2.5 & 0 & 0 & 0 & 0 & 0 & 38.5 \\
    1000 & - & 0.045 & 95   & 5   & 4.5 & 2   & 0 & 0 & 0 & 0 & 0 & 8.5 \\
    2000 & - & 0.023 & 98   & 7   & 4   & 3.5 & 0 & 0 & 0 & 0 & 0 & 0   \\
    \bottomrule
  \end{tabular}
\end{table}

\begin{table}[hbt!]
  \caption{Simulation results from aggregated observations with $k=5$}
  \label{aggregated_k5}
  \centering
  \begin{tabular}{lllllllllllll}
    \toprule
    \multicolumn{2}{c}{Sample Size} & \multirow{2}{*}{MSE} & \multirow{2}{*}{CR }
    & \multicolumn{3}{c}{Type I Error Rate$\%$} &  \multicolumn{6}{c}{Type II Error Rate$\%$} \\
    \cmidrule(r){1-2} \cmidrule(r){5-7} \cmidrule{8-13}
     $n$ & $\tilde{n}$ &  & Rate$\%$ & $a_{12}$  & $a_{22}$  & $a_{32}$  & $a_{11}$  & $a_{13}$ & $a_{21}$  & $a_{23}$ & $a_{31}$  & $a_{33}$\\
    \midrule
    100  & 20  & 0.498 & 95.5 & 1.5 & 2   & 2.5 & 0 & 0 & 0 & 0 & 0 & 88.5 \\
    200  & 40  & 0.247 & 99   & 4.5 & 3.5 & 3   & 0 & 0 & 0 & 0 & 0 & 81 \\
    500  & 100 & 0.092 & 97.5 & 3.5 & 4   & 2.5 & 0 & 0 & 0 & 0 & 0 & 42 \\
    1000 & 200 & 0.045 & 95.5 & 5   & 4.5 & 2   & 0 & 0 & 0 & 0 & 0 & 9.5 \\
    2000 & 400 & 0.023 & 98.5 & 6.5 & 4   & 3.5 & 0 & 0 & 0 & 0 & 0 & 0   \\
    \bottomrule
  \end{tabular}
\end{table}

\begin{table}[hbt!]
  \caption{Simulation results from aggregated observations with $k=10$}
  \label{aggregated_k10}
  \begin{tabular}{lllllllllllll}
    \toprule
    \multicolumn{2}{c}{Sample Size} & \multirow{2}{*}{MSE} & \multirow{2}{*}{CR }
    & \multicolumn{3}{c}{Type I Error Rate$\%$} &  \multicolumn{6}{c}{Type II Error Rate$\%$} \\
    \cmidrule(r){1-2} \cmidrule(r){5-7} \cmidrule{8-13}
     $n$ & $\tilde{n}$ &  & Rate$\%$ & $a_{12}$  & $a_{22}$  & $a_{32}$  & $a_{11}$  & $a_{13}$ & $a_{21}$  & $a_{23}$ & $a_{31}$  & $a_{33}$\\
    \midrule
    100  & 10  & 0.536 & 84   & 1   & 0.5 & 2.5 & 0 & 0 & 0 & 0 & 0 & 94.5 \\
    200  & 20  & 0.262 & 98   & 3   & 2.5 & 2.5 & 0 & 0 & 0 & 0 & 0 & 85.5 \\
    500  & 50  & 0.093 & 98   & 3   & 3   & 1.5 & 0 & 0 & 0 & 0 & 0 & 47 \\
    1000 & 100 & 0.045 & 96   & 5   & 4.5 & 2   & 0 & 0 & 0 & 0 & 0 & 8.5 \\
    2000 & 200 & 0.023 & 98.5 & 6.5 & 3.5 & 3.5 & 0 & 0 & 0 & 0 & 0 & 0   \\
    \bottomrule
  \end{tabular}
\end{table}

\begin{table}[hbt!]
  \caption{Simulation results from aggregated observations with $k=20$}
  \label{aggregated_k20}
  \centering
  \begin{tabular}{lllllllllllll}
    \toprule
    \multicolumn{2}{c}{Sample Size} & \multirow{2}{*}{MSE} & \multirow{2}{*}{CR }
    & \multicolumn{3}{c}{Type I Error Rate$\%$} &  \multicolumn{6}{c}{Type II Error Rate$\%$} \\
    \cmidrule(r){1-2} \cmidrule(r){5-7} \cmidrule{8-13}
     $n$ & $\tilde{n}$ &  & Rate$\%$ & $a_{12}$  & $a_{22}$  & $a_{32}$  & $a_{11}$  & $a_{13}$ & $a_{21}$  & $a_{23}$ & $a_{31}$  & $a_{33}$\\
    \midrule
    100  & 5   & 0.858 & 30.5 & 0.5 & 1   & 0.5 & 0 & 0 & 0 & 0 & 3.5 & 99 \\
    200  & 10  & 0.301 & 86   & 0.5 & 3   & 1.5 & 0 & 0 & 0 & 0 & 0   & 91 \\
    500  & 25  & 0.093 & 98   & 3.5 & 2.5 & 1   & 0 & 0 & 0 & 0 & 0   & 53.5 \\
    1000 & 50  & 0.045 & 96.5 & 4   & 3.5 & 2   & 0 & 0 & 0 & 0 & 0   & 11.5 \\
    2000 & 100 & 0.023 & 99   & 6   & 4   & 3.5 & 0 & 0 & 0 & 0 & 0   & 0   \\
    \bottomrule
  \end{tabular}
\end{table}
The results show that for all three cases where $k=5,10 \text{ and } 20$, MSE decreases and approaches zero with the increase of sample size $\tilde{n}$, which indicates the consistency of the estimators. As can be seen from the tables, the within CR rate is around $95\%$ and the type I error rate for each of the zero entries in $A$ is around $5\%$ when the sample size $\tilde{n}$ is large enough in all $k=5,10 \text{ and } 20$ cases. In addition, the type II error rate reduces to zero as the sample size $\tilde{n}$ increases. This result implies the correctness of our asymptotic normality theory and indicates the test of causal structure inference for the ODE system is powerful.

As can be seen from the tables, the type II error rate from the aggregated observations tend to be slightly greater than that of the original observations for each $n$. Specifically, with the greater the $k$ being, the greater the type II error rate is. This is reasonable, because the sample size of the aggregated observations ($\tilde{n}$) is much smaller than that of the original ones ($n$), which causes the lack of accuracy of the parameter estimates when the sample size $\tilde{n}$ is not large enough. Thus leading to a higher MSE and a higher type II error rate. 

However, it can be seen from the last two rows in each of the four tables, the MSEs and type II error rates are almost same for each case, which implies that when the sample size of the aggregated observations $\tilde{n}$ is large enough, the parameter estimates of the aggregated observations can reach the same level of accuracy as that of the original observations. In addition, the power of inferring the causal structure of the ODE system from aggregated observations can be as good as that from the original observations. The reason why the aggregated observations with a much smaller sample size $\tilde{n} = n/k$ can still perform as good as the original observations with the corresponding size $n$ is that, the variance of the aggregated noise term $\tilde{\boldsymbol{\epsilon}}_i$ becomes $k$ times smaller than that of the original one $\boldsymbol{\epsilon}_i$, that is
\begin{equation*}
    \tilde{\Sigma} = \Sigma/k
\end{equation*}
based on the generation rules of the aggregated observations. Therefore, with a much smaller noise variance, the parameter estimates from aggregated observations can reach the same level of accuracy as that of the original observations with a much smaller sample size.

\subsubsection{Simulation results from time-scaled observations}
In the following, we show the simulation results for time-scaled observations with $k=0.01, 0.1, 1, 10, \text{ and }100$, respectively. Since for all the metrics except MSE, under all cases the simulation results are identical, we present their results in one table. The MSE for each $k$ are the same up to $10^{-5}$, therefore, the differences of MSE are negligible and one can safely conclude that the simulation results for time-scaled observations with various $k$ is the same as that of the original observations (that is, $k=1$). This implies the correctness of our theoretical results of the time-scaled observations established in Section~\ref{section:large_sample}.
\begin{table}[hbt!]
  \caption{Simulation results from time-scaled observations with $k=0.01, 0.1, 1, 10, 100$}
  \label{time-scaled}
  \centering
  \begin{tabular}{lllllllllllll}
    \toprule
    \multicolumn{2}{c}{Sample Size} & \multirow{2}{*}{MSE} & \multirow{2}{*}{CR }
    & \multicolumn{3}{c}{Type I Error Rate$\%$} &  \multicolumn{6}{c}{Type II Error Rate$\%$} \\
    \cmidrule(r){1-2} \cmidrule(r){5-7} \cmidrule{8-13}
     $n$ & $\tilde{n}$ &  & Rate$\%$ & $a_{12}$  & $a_{22}$  & $a_{32}$  & $a_{11}$  & $a_{13}$ & $a_{21}$  & $a_{23}$ & $a_{31}$  & $a_{33}$\\
    \midrule
    100  & 100   & 0.419 & 96   & 2   & 2.5 & 2 & 0 & 0 & 0 & 0 & 0 & 85.5 \\
    200  & 200  & 0.269 & 94   & 7   & 6   & 4.5 & 0 & 0 & 0 & 0 & 0   & 72.5 \\
    500  & 500  & 0.104 & 91.5 & 5.5 & 5.5 & 6   & 0 & 0 & 0 & 0 & 0   & 38.5 \\
    1000 & 1000  & 0.050 & 92.5 & 5   & 6   & 3   & 0 & 0 & 0 & 0 & 0   & 9 \\
    2000 & 2000 & 0.022 & 95   & 3.5 & 4.5 & 0.5 & 0 & 0 & 0 & 0 & 0   & 1   \\
    \bottomrule
  \end{tabular}
\end{table}

\section{Related Work}
In this section, we introduce the related work from three closely related aspects.
\subsection{Identifiability Analysis of Linear ODE Systems} 
Most current studies for identifiability analysis of parameters in linear dynamical systems are in the control theory \citep{bellman1970structural, gargash1980necessary,glover1974parametrizations,grewal1976identifiability,rosenbrock1974structural}. In the applied mathematics area, \citet{stanhope2014identifiability,qiu2022identifiability} provided systematic studies of the identifiability analysis of linear ODE systems from a \textbf{single} trajectory. Our identifiability analysis is built upon \citep{stanhope2014identifiability}. However, instead of building identifiability based on the entire continuous trajectory, we extend the identifiability work to more practical scenarios where we only have discrete observations sampled from the trajectory and do not know the initial conditions. The authors in \citep{stanhope2014identifiability} also discussed using equally-spaced error-free observations to calculate the system parameters explicitly. Our work's main distinction is that we build the identifiability condition entirely on the parameter of interest, that is, ($A$ and $\boldsymbol{x}_0$), without further linearly independent assumption on the observations. With the help of our identifiability condition, the parameter estimator's asymptotic properties can be established with mild assumptions. Specifically, the covariance matrix of the asymptotic normality distribution can be explicitly expressed. 
Thus, we can perform statistical inference for the unknown parameters. Moreover, we treat the initial condition $\boldsymbol{x}_0$ as a parameter while $\boldsymbol{x}_0$ is a given fixed value in their setting. 
In the most recent work \citep{qiu2022identifiability}, the authors proposed several quantitative scores for identifiability analysis of linear ODEs in practice. 

\subsection{Parameter Estimation Methods for ODE Systems} 
The NLS method is applied to estimate parameters in ODE systems in \citep{bock1983recent,biegler1986nonlinear,xue2010sieve}. However, to our knowledge, no existing work provides a systematic asymptotic analysis of the NLS estimator for the ODE system \eqref{eq:ODE model}. The closest related work is that of \citet{xue2010sieve}, who studied asymptotic properties of the NLS estimator based on approximating ODEs' solutions by using the Runge-Kutta algorithm \citep{dormand1980family}. Nonetheless, their work requires strong assumptions, which are complex and cumbersome to verify. In contrast, our approach necessitates milder assumptions and yields an analytic covariance matrix for the asymptotic normal distribution. In addition to the NLS method, the two-stage smoothing-based estimation method is also well used for parameter estimation in ODE systems and its asymptotic properties have been extensively explored \citep{varah1982spline,chen2008efficient,chen2008estimation,wu2014sparse,brunton2016discovering}. This method usually applies smoothing approaches such as penalized splines to estimate the state variables and their derivatives at the first stage. Thus a large number of observations are needed to ensure the estimates' accuracy. 
Principal differential analysis \citep{ramsay1996principal,heckman2000penalized,poyton2006parameter,qi2010asymptotic} and Bayesian approaches \citep{ghosh2021variational} were also proposed to estimate unknown parameters in ODE systems. In recent years, several neural-network-based parameter estimation methods for ODE systems have been proposed~\citep{rubanova2019latent,lu2021learning,qin2019data}. In these works, the authors use multiple (usually a large number) trajectories instead of a single trajectory to train the neural network model and aim to perform trajectory prediction. No identifiability is guaranteed.  

\subsection{Connection between Causality and Differential Equations}
\citet{aalen2012causality} suggested, differential equations allow for a natural interpretation of causality in dynamic systems. The authors in~\citep{mooij2013ordinary, rubenstein2018deterministic, bongers2018random}
built an explicit bridge between the differential equations and the causal models by establishing the relationship between ODEs/Random Differential Equations (RDEs) and structural causal models. \citet{hansen2014causal} and \citet{wang2024generator} proposed causal interpretations and identifiability analysis of Stochastic Differential Equations (SDEs). \citet{bellot2021neural} proposed a method to consistently discover the causal structure of SDE systems based on penalized neural ODEs~\citep{chen2018neural}. These works aim to build a theoretical connection between causality and differential equations in various ways. Our work contributes to this body of literature by proposing a method for inferring the causal structure of ODEs from a statistical perspective. To our knowledge, our approach represents the first application of hypothesis testing in ODE systems for the inference of causal structure.

\section{Conclusion}\label{conclusion} 
In this paper, we derived a sufficient condition for identifiability of homogeneous linear ODE systems from a sequence of equally-spaced error-free observations. Specifically, the observations lie on a single trajectory. Furthermore, we studied the consistency and asymptotic normality of the NLS estimator. The inference of unknown parameters based on the established theoretical results was also investigated. In particular, we proposed a new method to infer the causal structure of the ODE system. Finally, we extended the identifiability and asymptotic properties results to cases with aggregated and time-scaled observations.

A time series is most commonly collected at equally-spaced time points in practice, which motivates us to focus on the study over equally-spaced observations from a single trajectory in this work. However, as the authors pointed out in~\citep{voelkle2013continuous}, using irregularly-spaced observations can be advantageous in obtaining more information from the dynamical system. Therefore, extending the study to cases with irregularly-spaced observations from a single trajectory is a possible direction for future work.

\acks{YW was supported by the Australian Government Research Training Program (RTP) Scholarship from the University of Melbourne. XG was supported by ARC DE210101352. MG was supported by ARC DE210101624.  TL was partially supported by the following Australian Research Council projects: FT220100318, DP220102121, LP220100527,
LP220200949, and IC190100031. KZ was supported in part by the NSF-Convergence Accelerator Track-D award \#2134901, by the National Institutes of Health (NIH) under Contract R01HL159805, by grants from Apple Inc., KDDI Research, Quris AI, and IBT, and by generous gifts from Amazon, Microsoft Research, and Salesforce}

\newpage

\appendix
\section{Detailed proofs}
In this appendix, we present the detailed proofs of all our lemmas, theorems and corollaries.
\subsection{Proof of Lemma~\ref{theorem:identifiability for A and b}}\label{proof:theorem2.1}
\begin{proof}
For $A, A', \boldsymbol{x}_{0}, \boldsymbol{x}_{0}'$ defined in Definition \ref{def:identifiability for A and b}, if $\boldsymbol{x}_{0} \neq \boldsymbol{x}_{0}'$, one sees that $\boldsymbol{x}(\cdot;\boldsymbol{x}_{0}, A) \neq \boldsymbol{x}(\cdot; \boldsymbol{x}_{0}', A')$, because $\boldsymbol{x}(0;\boldsymbol{x}_{0}, A) \neq \boldsymbol{x}(0; \boldsymbol{x}_{0}', A')$. Therefore, one only needs to consider the case where $\boldsymbol{x}_{0} = \boldsymbol{x}_{0}'$ and $A \neq A'$. Under this circumstance, the Definition \ref{def:identifiability for A and b} and Lemma \ref{theorem:identifiability for A and b} are similar to \citep[Definition 2.3, Theorem 2.5]{stanhope2014identifiability}, with the only difference is that the identifiability of ODE system (\ref{eq:ODE model}) not only applies to a fixed initial condition $\boldsymbol{x}_0$ but an open set $M^0 \subset \mathbb{R}^d$. According to \citep[Proof of Theorem 2.5]{stanhope2014identifiability}, we can directly get the proof for Lemma \ref{theorem:identifiability for A and b}.
\end{proof}

\subsection{Proof of Theorem~\ref{theorem:identifiability from discrete observations}}\label{proof:theorem2.2}

\begin{proof}
Let $\Phi(t)$ denote the principal matrix solution of model (\ref{eq:ODE model}), that is, $\Phi(t) := e^{At}$. Let $\boldsymbol{X}_q$ denote the matrix $(\boldsymbol{x}_q, \boldsymbol{x}_{q+1},\ldots,\boldsymbol{x}_{q+d-1}) \in \mathbb{R}^{d \times d}$, for $q=1,2$. As defined in the statement of Theorem~\ref{theorem:identifiability from discrete observations}, we have $d+1$ equally-spaced observations $\boldsymbol{x}_j = \boldsymbol{x}(t_j;\boldsymbol{x}_0, A)$, with the $t_j$'s denoting equally-spaced time points, for $j=1,\ldots, d+1$, we denote the equal time space as $\Delta t := t_{j+1} - t_j$, for all $j=1,\ldots, d$.  
According to the solution function (\ref{eq:ODE solution}), one obtains
\begin{equation*}
\begin{split}
    \boldsymbol{x}_{j+1} &= \boldsymbol{x}(t_{j+1}; \boldsymbol{x}_0, A) = \boldsymbol{x}(t_{j} + \Delta t; \boldsymbol{x}_0, A) 
    = e^{A \Delta t}e^{At_j} \boldsymbol{x}_0 = e^{A \Delta t} \boldsymbol{x}(t_j;\boldsymbol{x}_0, A) 
    = \Phi({\Delta t}) \boldsymbol{x}_j\,,
\end{split}
\end{equation*}
for all $j=1,\ldots, d$. Therefore,  
\begin{equation*}
    \boldsymbol{X}_2 = \Phi({\Delta t})\boldsymbol{X}_1\,.
\end{equation*}
Then the proof can be broken down into three steps.
\vspace{2mm}

\noindent{\bf Step \romannumeral 1}: We show that $\boldsymbol{X}_1$ is invertible, thus $\Phi(\Delta t) = \boldsymbol{X}_2 \boldsymbol{X}_1^{-1}$.

\noindent Suppose that $\boldsymbol{X}_1 = (\boldsymbol{x}_1, \boldsymbol{x}_{2},\ldots,\boldsymbol{x}_{d}) \in \mathbb{R}^{d \times d}$ is singular, that is, ${\boldsymbol{x}_1, \boldsymbol{x}_{2},\ldots,\boldsymbol{x}_{d}}$ are linearly dependent, then there exists a non-zero vector $\boldsymbol{c} = [c_1, c_2,\ldots, c_d]^{\top} \neq \boldsymbol{0}_d$ satisfying $\boldsymbol{X}_1 \boldsymbol{c} = \boldsymbol{0}_d$, that is,
\begin{equation*}
    c_1 \boldsymbol{x}_1 + c_2 \boldsymbol{x}_2 + \cdots + c_d \boldsymbol{x}_d = \boldsymbol{0}_d\,,
\end{equation*}
where $\boldsymbol{0}_d$ denotes the $d$-dimensional zero vector.
 By plugging the solution of $\boldsymbol{x}_j$ in the above equation, one obtains that
\begin{equation*}
    c_1 e^{At_1}\boldsymbol{x}_0 + c_2 e^{A\Delta t}  e^{At_1}\boldsymbol{x}_0 + \cdots + c_d e^{A(d-1)\Delta t}e^{At_1}\boldsymbol{x}_0  = \boldsymbol{0}_d\,,
\end{equation*}
which is
\begin{equation}\label{eq:eq1}
    (c_1 I + c_2 e^{A\Delta t} + \cdots + c_d e^{A(d-1)\Delta t}) e^{At_1}\boldsymbol{x}_0 = \boldsymbol{0}_d\,.
\end{equation}
Under condition A2 stated in Theorem\ref{theorem:identifiability from discrete observations}, one gets the Jordan decomposition of parameter matrix $A$, denoting as $A=Q\Lambda Q^{-1}$, with 
$\Lambda = \text{diag}(\lambda_1, \lambda_2, \ldots, \lambda_d)$, a $d$-dimensional diagonal matrix, where $\lambda_1, \lambda_2,\ldots, \lambda_d$ are $d$ distinct real eigenvalues of the parameter matrix $A$. Then Equation~\eqref{eq:eq1} is equivalent to
\begin{equation*}
    Q(c_1 I + c_2 e^{\Lambda \Delta t} +\cdots+ c_d e^{\Lambda (d-1)\Delta t})e^{\Lambda t_1}Q^{-1}\boldsymbol{x}_0 = \boldsymbol{0}_d\,.
\end{equation*}
Since matrix $Q$ is invertible, by multiplying $Q^{-1}$ in both hand sides of the equation, one sees that
\begin{equation*}
    (c_1 I + c_2 e^{\Lambda \Delta t} +\cdots+ c_d e^{\Lambda (d-1)\Delta t})e^{\Lambda t_1}Q^{-1}\boldsymbol{x}_0 = \boldsymbol{0}_d\,,
\end{equation*}
which is
\begin{equation*}
    \begin{bmatrix}
    c_1 + c_2e^{\lambda_{1}\Delta t}+ \cdots + c_de^{\lambda_{1}(d-1)\Delta t} & & \\
    & \ddots & \\
    & & c_1 + c_2e^{\lambda_{d}\Delta t}+ \cdots + c_de^{\lambda_{d}(d-1)\Delta t}
    \end{bmatrix}e^{\Lambda t_1}Q^{-1}\boldsymbol{x}_0 = \boldsymbol{0}_d\,,
\end{equation*}
Let 
\begin{equation*}
    u_i :=  e^{\lambda_i t_1}(c_1 + c_2e^{\lambda_{i}\Delta t}+ \cdots + c_de^{\lambda_{i}(d-1)\Delta t})\,,
\end{equation*}
for all $i=1,\ldots, d$ and matrix $U := \text{diag}(u_1,\ldots, u_d)$, Equation \eqref{eq:eq1} is equivalent to
\begin{equation}
    UQ^{-1}\boldsymbol{x}_0 = \boldsymbol{0}_d\,.\label{eq:UQx0}
\end{equation}

We next show that if Equation \eqref{eq:UQx0} holds, then $U$ is a zero matrix. Let  
\begin{equation*}
   \tilde{\Lambda} = \begin{bmatrix}
    1 & \lambda_1 & \cdots & \lambda_1^{d-1}\\
    1 & \lambda_2 & \cdots & \lambda_2^{d-1}\\
    \vdots &\vdots & \ddots &\vdots\\
    1 & \lambda_d & \cdots & \lambda_d^{d-1}
    \end{bmatrix}\,. 
\end{equation*}
Since, by condition A2, $\tilde{\Lambda}$ is a square Vandermonde matrix with all the $\lambda_i$'s distinct, $\tilde{\Lambda}$ is invertible. Thus, for any $\boldsymbol{c}$, there always exists a unique vector
$\boldsymbol{l} = [l_1, l_2,\ldots, l_d]^{\top}\in \mathbb{R}^d$ such that
\begin{equation*}
    \boldsymbol{l} =\tilde{\Lambda}^{-1}\begin{bmatrix}
    u_1\\u_2\\ \vdots\\ u_d
    \end{bmatrix} \,.
\end{equation*}
That is
$l_1+ l_2 \lambda_i +\ldots + l_d \lambda^{d-1}_i=u_i$ for all $i=1,\ldots,d$.
Using the Jordan decomposition form of $A$, we then have \begin{align*}
l_1 \boldsymbol{x}_0 + l_2 A\boldsymbol{x}_0 + \cdots +  l_d A^{d-1}\boldsymbol{x}_0 =  Q(l_1 I + l_2 \Lambda + \cdots + l_d \Lambda^{d-1})Q^{-1}\boldsymbol{x}_0\
=Q U Q^{-1}\boldsymbol{x}_0\,.  
\end{align*}
Thus, if Equation \eqref{eq:UQx0} holds, then we have 
\begin{equation}\label{eq:eq2}
l_1 \boldsymbol{x}_0 + l_2 A\boldsymbol{x}_0 + \cdots +  l_d A^{d-1}\boldsymbol{x}_0 = \boldsymbol{0}_d \,.
\end{equation}

Under condition A1 stated in Theorem~\ref{theorem:identifiability from discrete observations}, if Equation \eqref{eq:eq2} holds,
one obtains $\boldsymbol{l} = \boldsymbol{0}_d$ since $\{\boldsymbol{x}_0, A\boldsymbol{x}_0,\ldots, A^{d-1}\boldsymbol{x}_0\}$ are linearly independent. Then $u_i=0$ for all $i=1,\ldots,d$ and $U$ is a zero matrix, which means
\begin{equation}\label{eq:system of equation2}
\begin{cases}
c_1 + c_2e^{\lambda_{1}\Delta t}+ \cdots + c_de^{\lambda_{1}(d-1)\Delta t} = 0 \,,\\ 
\cdots\\ 
c_1 + c_2e^{\lambda_{d}\Delta t}+ \cdots + c_de^{\lambda_{d}(d-1)\Delta t} = 0 \,.
\end{cases}
\end{equation}
Equation \eqref{eq:system of equation2} can be written in matrix form as:
\begin{equation*}
   \begin{bmatrix}
    1 & e^{\lambda_1 \Delta t} & e^{2\lambda_1 \Delta t} & \cdots & e^{(d-1)\lambda_1 \Delta t}\\
    1 & e^{\lambda_2 \Delta t} & e^{2\lambda_2 \Delta t} & \cdots & e^{(d-1)\lambda_2 \Delta t}\\
    \vdots & \vdots & \vdots & \ddots & \vdots \\
    1 & e^{\lambda_d \Delta t} & e^{2\lambda_d \Delta t} & \cdots & e^{(d-1)\lambda_d \Delta t}
    \end{bmatrix} \begin{bmatrix}
    c_1\\c_2 \\ \vdots\\ c_d
    \end{bmatrix} = \boldsymbol{0}_d\,,
\end{equation*}
where, by condition~A2, the first matrix on the left-hand side is again an invertible square Vandermonde matrix. 
This implies that $\boldsymbol{c}=[c_1, c_2,\ldots, c_d]^{\top} = \boldsymbol{0}_d$, 
which contradicts the assumption ($\boldsymbol{c}\neq \boldsymbol{0}_d$) made at the beginning of the proof. Therefore, one concludes that $\boldsymbol{X}_1$ is invertible, and thus
\begin{equation*}
    \Phi(\Delta t) = \boldsymbol{X}_2\boldsymbol{X}_1^{-1}\,.
\end{equation*}

\noindent{\bf Step \romannumeral 2}: We prove that under conditions A1 and A2, there always exists a unique logarithm of matrix $\Phi(\Delta t) (= e^{A \Delta t})$, thus, $A$ is identifiable from $\Phi(\Delta t)$. 

\noindent We first present a lemma we will use for our proof.
\begin{lemma}\label{theorem:matrix logrithm identifiability}
\textnormal{\citep[Theorem 6.3]{stanhope2014identifiability}} Let C be a real square matrix. Then there exists a unique real solution $Y$ to the equation $C = e^Y$ if and only if all the eigenvalues of $C$ are positive real and no Jordan block of $C$ belonging to any eigenvalue appears more than once.
\end{lemma}
 
The Jordan decomposition of $e^{A \Delta t}$ is
\begin{equation*}
    e^{A \Delta t} = Q e^{\Lambda \Delta t} Q^{-1} = Q 
    \begin{bmatrix}
        e^{\lambda_{1} \Delta t} & & \\
        & \ddots & \\
        & & e^{\lambda_{d} \Delta t}
    \end{bmatrix}Q^{-1}\,.
\end{equation*}
Under condition A2, $\lambda_1, \lambda_2,\ldots, \lambda_d$ are $d$ distinct real values, therefore, $e^{A \Delta t}$ has $d$ distinct positive real eigenvalues $e^{\lambda_j \Delta t}$, for all $j=1, 2,\ldots,d$.  Then by Lemma \ref{theorem:matrix logrithm identifiability}, one obtains $A \Delta t$ by taking logarithm of $e^{A \Delta_t}$, thus one obtains $A$.

\noindent{\bf Step \romannumeral 3}: We show the initial condition $\boldsymbol{x}_0$ is identifiable. \\
One sees that
\begin{equation*}
\begin{split}
    \text{det}(e^{A t}) &= \text{det}(Qe^{\Lambda t}Q^{-1}) = \text{det}(Q^{-1})\text{det}(e^{\Lambda t})\text{det}(Q) \\
    &=\text{det}(Q^{-1}Q)\text{det}(e^{\Lambda t}) = \text{det}(e^{\Lambda t}) = e^{\sum_{i=1}^d \lambda_i t} \neq 0\,,
\end{split}
\end{equation*}
for any $t>0$. Therefore, $e^{At}$ is nonsingular for any $t >0$. Since $\boldsymbol{x}_1 = e^{A t_1}\boldsymbol{x}_0$, one obtains that $\boldsymbol{x}_0 = e^{-At_1} \boldsymbol{x}_1$.

Therefore, we have proved that both $A$ and $\boldsymbol{x}_0$ can be explicitly calculated by using any $d+1$ equally-spaced error-free observations, which means the ODE system~\eqref{eq:ODE model} is identifiable from these observations.
\end{proof}

\subsection{Proof of Corollary~\ref{corollary:aggregate}}\label{proof:corollary2.2.1}
\begin{proof}
By definition of the aggregated observations in Definition~\ref{def:aggregated}, one sees that
\begin{equation}\label{eq:aggregated x_j}
\begin{split}
    \tilde{\boldsymbol{x}}_j 
    &= (\boldsymbol{x}_{(j-1)k+1} + \boldsymbol{x}_{(j-1)k+2} + \cdots + \boldsymbol{x}_{jk})/k\\ 
    &= ( e^{At_{(j-1)k+1}}\boldsymbol{x}_0 + e^{At_{(j-1)k+2}}\boldsymbol{x}_0 + \cdots + e^{At_{jk}}\boldsymbol{x}_0)/k \\ 
    &= e^{At_{(j-1)k+1}}(I+e^{A\Delta_t} + \cdots + e^{A(k-1)\Delta_t})\boldsymbol{x}_0/k \,,
\end{split}
\end{equation}
where $\Delta t = t_{j+1} - t_j$, for all $j=1,\ldots, n-1$, denotes the equal time space of the original observations.

If we set the time of the aggregated observation $\tilde{\boldsymbol{x}}_j$ as the time of the first original observation (that is, $\boldsymbol{x}_{(j-1)k+1}$) of these $k$ consecutive, non-overlapping observations.
Then according to Equation~\eqref{eq:aggregated x_j}, one sees that the solution function of the aggregated observations follows the structure of
\begin{equation}\label{eq:aggregated ODE}
    \tilde{\boldsymbol{x}}(t; \tilde{\boldsymbol{x}}_0, \tilde{A}) = e^{\tilde{A}t}\tilde{\boldsymbol{x}}_0\,,
\end{equation}
with 
\begin{equation*}
    \tilde{\boldsymbol{x}}_0 = (I+e^{A\Delta_t} + \cdots + e^{A(k-1)\Delta_t})\boldsymbol{x}_0/k\, \mbox{ and } \tilde{A} = A\,.
\end{equation*}
This implies that the aggregated observations follow a new ODE system as \eqref{eq:ODE model} but with a different system parameter $( \tilde{\boldsymbol{x}}_0, \tilde{A})$. To see whether $(\tilde{\boldsymbol{x}}_0, \tilde{A})$ are the true parameters corresponding to the new ODE system, one has to prove the identifiability of the new ODE system from the aggregated observations.

According to Theorem~\ref{theorem:identifiability from discrete observations}, one needs to prove $\{\tilde{\boldsymbol{x}}_0, \tilde{A}\tilde{\boldsymbol{x}}_0,\ldots, \tilde{A}^{d-1}\tilde{\boldsymbol{x}}_0\}$ are linearly independent first, that is,  proving $\{\tilde{\boldsymbol{x}}_0, A\tilde{\boldsymbol{x}}_0,\ldots, A^{d-1}\tilde{\boldsymbol{x}}_0\}$ are linearly independent. 
Let $\boldsymbol{l} =[l_1, l_2,\ldots, l_d]^{\top}$ be such that
\begin{equation}\label{eq:aggregated linearly independent}
    l_1 \tilde{\boldsymbol{x}}_0 + l_2 A\tilde{\boldsymbol{x}}_0 + \cdots + l_d A^{d-1}\tilde{\boldsymbol{x}}_0 = \boldsymbol{0}_d\,,
\end{equation}
we want to show that $\boldsymbol{l} =\boldsymbol{0}_d$. 
If one sets 
\begin{equation*}
    B := (I+e^{A\Delta_t} + \cdots + e^{A(k-1)\Delta_t})/k\,,
\end{equation*}
then one sees
$\tilde{\boldsymbol{x}}_0 = B\boldsymbol{x}_0$.
Taking Jordan decomposition of $A$ in $B$, one obtains that
\begin{equation*}
    B = Q(I+e^{\Lambda\Delta_t} + \cdots + e^{\Lambda(k-1)\Delta_t})Q^{-1}/k\,.
\end{equation*}
If one sets
$ s_j = 1+e^{\lambda_j \Delta t} + \cdots +e^{\lambda_j (k-1)\Delta t}$,
for all $j=1,\ldots, d$, and denotes $S$ as the diagonal matrix with the $j$th element being $s_j$, one obtains that
\begin{equation}\label{eq:B}
    B = QSQ^{-1}/k\,.
\end{equation}
Substituting $\tilde{\boldsymbol{x}}_0$ in Equation~\eqref{eq:aggregated linearly independent} with $B\boldsymbol{x}_0$, one obtains that
\begin{equation*}
    (l_1I + l_2 A + \cdots + l_dA^{d-1})B\boldsymbol{x}_0 = \boldsymbol{0}_d\,.
\end{equation*}
By further taking Jordan decomposition of $A$, one obtains that
\begin{equation*}
    Q(l_1I + l_2 \Lambda + \cdots + l_d\Lambda^{d-1})Q^{-1}B\boldsymbol{x}_0 = \boldsymbol{0}_d\,.
\end{equation*}
By plugging $B$ expressed in Equation~\eqref{eq:B} into the previous equation, and multiply $kQ^{-1}$ in both-hand sides of the equation, one obtains that
\begin{equation*}
    (l_1I + l_2 \Lambda + \cdots + l_d\Lambda^{d-1})SQ^{-1}\boldsymbol{x}_0 = \boldsymbol{0}_d\,,
\end{equation*}
that is,
\begin{equation}
    \begin{bmatrix}
    s_1(l_1 + l_2\lambda_{1}+ \cdots + l_d\lambda_{1}^{(d-1)}) & & \\
    & \ddots & \\
    & & s_d(l_1 + l_2\lambda_{d}+ \cdots + l_d\lambda_{d}^{(d-1)})
    \end{bmatrix}Q^{-1}\boldsymbol{x}_0 = \boldsymbol{0}_d.
\end{equation}
According to the proof process of Theorem~\ref{theorem:identifiability from discrete observations} in Appendix~\ref{proof:theorem2.2}, one obtains that 
\begin{equation}\label{eq:ssystem}
\begin{cases}
    s_1(l_1 + l_2 \lambda_1 + \cdots + l_d \lambda_1 ^{d-1}) = 0 \,,\\
    \cdots\\
    s_d(l_1 + l_2 \lambda_d + \cdots + l_d \lambda_d ^{d-1}) = 0\,.
\end{cases}  
\end{equation}
Since $s_j > 0$ for all $j=1,\ldots, d$, the system of equations~\eqref{eq:ssystem} is equivalent to 
\begin{equation}\label{eq:system}
\begin{cases}
    l_1 + l_2 \lambda_1 + \cdots + l_d \lambda_1 ^{d-1} = 0\,, \\ 
    \cdots\\
    l_1 + l_2 \lambda_d + \cdots + l_d \lambda_d ^{d-1} = 0\,.
\end{cases}  
\end{equation}
Now Equation~\eqref{eq:system} can be written in matrix form as:
\begin{equation*}
    \begin{bmatrix}
    1 & \lambda_1 & \lambda_1^2  & \cdots & \lambda_1^{d-1}\\
    1 & \lambda_2  & \lambda_2^2 & \cdots & \lambda_2^{d-1}\\
    \vdots & \vdots & \vdots & \ddots & \vdots \\
    1 & \lambda_d & \lambda_d^2 & \cdots & \lambda_d^{d-1}
    \end{bmatrix}
    \begin{bmatrix}
    l_1\\
    l_2\\
    \vdots \\
    l_d
    \end{bmatrix}=\boldsymbol{0}_d    \,,
\end{equation*}
where, by condition A2, the first matrix on the left-hand side is an invertible square Vandermonde matrix. This implies that $\boldsymbol{l}= \boldsymbol{0}_d$. Therefore, $\{\tilde{\boldsymbol{x}}_0, A\tilde{\boldsymbol{x}}_0,\ldots, A^{d-1}\tilde{\boldsymbol{x}}_0\}$ are linearly independent. 

Obviously, the aggregated observations are equally-spaced, since
\begin{equation}
    \tilde{t}_{j+1}-\tilde{t}_{j} = t_{1+jk} - t_{1+(j-1)k} = k\Delta t \,,
\end{equation}
for all $j=1,\ldots, \tilde{n}-1$. 
Under condition A2, $\tilde{A} (\text{that is, } A)$ has $d$ distinct real eigenvalues, and under condition A3, $\tilde{n} > d$. Therefore, by Theorem~\ref{theorem:identifiability from discrete observations}, one concludes that the new ODE system is identifiable from aggregated observations with $\tilde{\boldsymbol{x}}_0 = (I+e^{A\Delta_t} + \cdots + e^{A(k-1)\Delta_t})\boldsymbol{x}_0/k$ and $\tilde{A} = A$. 
Since 
\begin{equation*}
    I+e^{A\Delta_t} + \cdots + e^{A(k-1)\Delta_t}
    = Q(I+e^{\Lambda\Delta_t} + \cdots + e^{\Lambda(k-1)\Delta_t})Q^{-1}
\end{equation*}
is invertible, one obtains that
\begin{equation*}
    \boldsymbol{x}_0 = k(I+e^{A\Delta_t} + \cdots + e^{A(k-1)\Delta_t})^{-1}\tilde{\boldsymbol{x}}_0\,,
\end{equation*}
and obviously,
\begin{equation*}
    A = \tilde{A}\,.
\end{equation*}
Therefore, the original ODE system with initial condition $\boldsymbol{x}_0$ and parameter matrix $A$ is fully identifiable from the aggregated observations.
\end{proof} 

\subsection{Proof of Corollary~\ref{corollary:timescaled}}
\begin{proof}\label{proof:corollary2.2.2}
By Definition of the time-scaled observations in Definition~\ref{def:time-scaled}, one sees that
\begin{equation}\label{eq:time-scaled}
\begin{split}
    \tilde{\boldsymbol{x}}_j 
    = \boldsymbol{x}_j
    = \boldsymbol{x}(t_j; \boldsymbol{x}_0, A)
    = e^{At_j}\boldsymbol{x}_0 =  e^{(A/k)\cdot kt_j}\boldsymbol{x}_0
\,.
\end{split}
\end{equation}
By definition the time of the time-scaled observation $\tilde{\boldsymbol{x}}_j$ is $kt_j$. Therefore, according to Equation~\ref{eq:time-scaled}, one sees that the solution function of the time-scaled observations follows the structure of
\begin{equation*}
    \tilde{\boldsymbol{x}}(t; \tilde{\boldsymbol{x}}_0, \tilde{A}) = e^{\tilde{A}t}\tilde{\boldsymbol{x}}_0\,,
\end{equation*}
with $\tilde{\boldsymbol{x}}_0 = \boldsymbol{x}_0\,, \mbox{ and } \tilde{A} = A/k$.
This implies that the time-scaled observations follow a new ODE system as \eqref{eq:ODE model} but with a different system parameter $( \tilde{\boldsymbol{x}}_0, \tilde{A})$. To see whether $( \tilde{\boldsymbol{x}}_0, \tilde{A})$ are the true parameters corresponding to the new ODE system, one has to prove the identifiability of the new ODE system from the time-scaled observations.

Obviously, $\{\tilde{\boldsymbol{x}}_0, \tilde{A} \tilde{\boldsymbol{x}}_0,\ldots, \tilde{A}^{d-1}\tilde{\boldsymbol{x}}_0\} 
= \{ \boldsymbol{x}_0, A\boldsymbol{x}_0/k,\ldots, A^{d-1}\boldsymbol{x}_0/k^{d-1} \}$ are linearly independent, since $\{\boldsymbol{x}_0, A \boldsymbol{x}_0,\ldots, A^{d-1}\boldsymbol{x}_0\}$ are linearly independent under condition A1. 
Moreover, $\tilde{A}$ has $d$ distinct real eigenvalues since $A$ has $d$ distinct real eigenvalues under condition A2.  

Based on the generation rules of the time-scaled observations, one sees that the new time-scaled observations are also equally-spaced. And under condition A3, $\tilde{n} = n > d$. Then by Theorem \ref{theorem:identifiability from discrete observations}, one concludes that the new ODE system is identifiable from the time-scaled observations, with $\tilde{\boldsymbol{x}}_0 = \boldsymbol{x}_0$ and $\tilde{A} = A/k$.

Since $k\neq 0$, simple calculation shows that
\begin{equation*}
    \boldsymbol{x}_0 = \tilde{\boldsymbol{x}}_0, \, A = k\tilde{A}\,.
\end{equation*}
Therefore, the original ODE system with initial condition $\boldsymbol{x}_0$ and parameter matrix $A$ is fully identifiable from the time-scaled observations.
\end{proof}

\subsection{Proof of Theorem~\ref{theorem:consistency}}\label{proof:theorem3.1}

We first present two Lemmas we will use for our proof.

\begin{lemma}\label{lemma:uniform convergence}
\textnormal{\citep[Lemma 2.9]{newey1994large}} Suppose $M_1$, $M_n$: $\mathbb{R}^{d_p}\rightarrow \mathbb{R}^{d_p'}$ for some finite integer $d_p, d_p' \geq 1$. If $\Theta$ is compact, $M_1(\boldsymbol{\theta})$ is continuous, 
\begin{equation}\label{eq:uniform_convergence_condition1}
    M_n(\boldsymbol{\theta}) \xrightarrow{p}  M_1(\boldsymbol{\theta})\,,\mbox{ as } n\rightarrow \infty
\end{equation}
for all $\boldsymbol{\theta} \in \Theta$, and there exist $\alpha >0$ and $B_n = O_p(1)$ such that for all $\boldsymbol{\theta}_1, \boldsymbol{\theta}_2 \in \Theta$,
\begin{equation}\label{eq:uniform_convergence_condition2}
    \parallel M_n(\boldsymbol{\theta}_1) - M_n(\boldsymbol{\theta}_2)\parallel_2 \,\leq B_n\parallel\boldsymbol{\theta}_1 -\boldsymbol{\theta}_2 \parallel_2^{\alpha}
\end{equation}
almost surely, then 
\begin{equation*}
    \sup_{\boldsymbol{\theta}\in \Theta}\parallel M_n(\boldsymbol{\theta}) -  M_1(\boldsymbol{\theta}) \parallel_2\, \xrightarrow{p} 0\,, \mbox{ as }n\rightarrow \infty \,.
\end{equation*}
\end{lemma}

\begin{lemma}\label{lemma:M_estimator consistency}
\textnormal{\citep[Theorem 5.7]{van2000asymptotic}} Let $M_n$ be random functions and let $M_1$ be a fixed function of $\boldsymbol{\theta}$ such that for every $\epsilon > 0$
\begin{equation}\label{eq:consistency_condition1}
    \sup_{\boldsymbol{\theta} \in \Theta}\parallel M_n(\boldsymbol{\theta}) - M_1(\boldsymbol{\theta})\parallel_2 \xrightarrow{p} 0 \,,
\end{equation}
\begin{equation}\label{eq:consistency_condition2}
    \sup_{\boldsymbol{\theta}: d(\boldsymbol{\theta}, \boldsymbol{\theta}^*) \geq \varepsilon} M_1(\boldsymbol{\theta}) > M_1(\boldsymbol{\theta}^*)\,.
\end{equation}
Then any sequence of estimators $\hat{\boldsymbol{\theta}}_n$ with 
\begin{equation}\label{eq:consistency_condition3}
    M_n(\hat{\boldsymbol{\theta}}_n) \leq M_n(\boldsymbol{\theta}^*) + o_p(1)
\end{equation}
converges in probability to $\boldsymbol{\theta}^*$. $d(\boldsymbol{\theta}, \boldsymbol{\theta}^*)$ denotes the distance between $\boldsymbol{\theta}$ and $\boldsymbol{\theta}^*$.
\end{lemma}

\begin{proof}
Recall that we have defined $M(\boldsymbol{\theta})$ and $M_n(\boldsymbol{\theta})$ in Equations~\eqref{eq:M(theta)} and~\eqref{eq:Mn(theta)}, respectively. 
Here, we set
\begin{equation}\label{eq:M1(theta)}
\begin{split}
    M_1(\boldsymbol{\theta}) := &M(\boldsymbol{\theta}) + E[\parallel \boldsymbol{\epsilon}\parallel_2^2]
    = \cfrac{1}{T}\int_0^T \parallel e^{A^*t}\boldsymbol{x}_0^*-e^{At}\boldsymbol{x}_0\parallel_2^2 dt
    +E[\parallel \boldsymbol{\epsilon} \parallel_2^2]\,.
\end{split}
\end{equation}
In order to prove the NLS estimator $\hat{\boldsymbol{\theta}}_n \xrightarrow{p} \boldsymbol{\theta}^*$, as $n\rightarrow \infty$, by Lemma~\ref{lemma:M_estimator consistency}, we need to prove that
all three conditions~\eqref{eq:consistency_condition1},~\eqref{eq:consistency_condition2} and~\eqref{eq:consistency_condition3} are satisfied w.r.t $M_n(\boldsymbol{\theta})$ and $M_1(\boldsymbol{\theta})$. Therefore, the proof can be broken down into three steps based on the proofs of each of these three conditions. In the following, we will show that all these three conditions are satisfied.

\noindent{\bf Step \romannumeral 1}: We prove that condition~\eqref{eq:consistency_condition1} in Lemma~\ref{lemma:M_estimator consistency} is satisfied based on Lemma~\ref{lemma:uniform convergence}.

\noindent According to Lemma~\ref{lemma:uniform convergence}, to prove the uniform convergence of $M_n(\boldsymbol{\theta})$ to $M_1(\boldsymbol{\theta})$ in parameter space $\Theta$, that is, condition~\eqref{eq:consistency_condition1}, we first need to prove the point-wise convergence of $M_n(\boldsymbol{\theta})$ to $M_1(\boldsymbol{\theta})$, that is, condition~\eqref{eq:uniform_convergence_condition1}.

According to Equation~\eqref{eq:Mn(theta)}, one sees that
\begin{equation}\label{eq:Mn(theta)_decomposition}
\begin{split}
    M_n(\boldsymbol{\theta}) &= \cfrac{1}{n} \sum_{i=1}^n \parallel \boldsymbol{y}_i - e^{At_i}\boldsymbol{x}_0 \parallel_2^2\\
    &= \cfrac{1}{n} \sum_{i=1}^n \parallel e^{A^*t_i}\boldsymbol{x}_0^*-e^{At_i}\boldsymbol{x}_0\parallel_2^2 + \cfrac{1}{n}\sum_{i=1}^n \parallel \boldsymbol{\epsilon}_i \parallel_2^2 + \cfrac{2}{n}\sum_{i=1}^n \boldsymbol{\epsilon}_i^{\top}(e^{A^*t_i}\boldsymbol{x}_0^*-e^{At_i}\boldsymbol{x}_0) \, .
\end{split}
\end{equation}
Let $n$ tend to infinity, one obtains that 
\begin{equation}\label{eq:Mn(theta)_item1}
\begin{split}
    &\lim_{n\rightarrow \infty} \cfrac{1}{n} \sum_{i=1}^n 
    \parallel e^{A^*t_i}\boldsymbol{x}_0^*-e^{At_i}\boldsymbol{x}_0\parallel_2^2\\
    =& \lim_{n\rightarrow \infty} \cfrac{n-1}{nT} \sum_{i=1}^n \cfrac{T}{n-1}
    \parallel e^{A^*t_i}\boldsymbol{x}_0^*-e^{At_i}\boldsymbol{x}_0\parallel_2^2
    \,=\,\cfrac{1}{T}\int_0^T 
    \parallel e^{A^*t}\boldsymbol{x}_0^*-e^{At}\boldsymbol{x}_0\parallel_2^2\,dt \, .
\end{split}
\end{equation}
By weak law of large numbers, one sees that 
\begin{equation}\label{eq:Mn(theta)_item2}
   \cfrac{1}{n}\sum_{i=1}^n \parallel \boldsymbol{\epsilon}_i \parallel_2^2 \hspace{0.1cm} \xrightarrow {p} E[\parallel\boldsymbol{\epsilon}\parallel_2^2]\,,\mbox{ as } n\rightarrow \infty \,,
\end{equation}
where we set 
$E[\parallel \boldsymbol{\epsilon}\parallel_2^2] :=E[\parallel \boldsymbol{\epsilon}_i\parallel_2^2] = \sum_{j=1}^d\sigma_j^2$, for all $i=1,\ldots, n$.
Since
\begin{equation*}
    E\bigg[\cfrac{2}{n}\sum_{i=1}^n \boldsymbol{\epsilon}_i^{\top}(e^{A^*t_i}\boldsymbol{x}_0^*-e^{At_i}\boldsymbol{x}_0)\bigg] = 0\,,
\end{equation*}
then by Chebyshev's inequality, for any $\varepsilon>0$, one has
\begin{equation}\label{eq:the third item}
    P\bigg(\bigg| \cfrac{2}{n}\sum_{i=1}^n \boldsymbol{\epsilon}_i^{\top}(e^{A^*t_i}\boldsymbol{x}_0^*-e^{At_i}\boldsymbol{x}_0) -0\bigg|\geq \varepsilon\bigg)
    \leq var\bigg(\frac{2}{n}\sum_{i=1}^n \boldsymbol{\epsilon}_i^{\top}(e^{A^*t_i}\boldsymbol{x}_0^*-e^{At_i}\boldsymbol{x}_0)\bigg)/\varepsilon^2 \, .
\end{equation}
\begin{equation}\label{eq:variance}
\begin{split}
    var\bigg(\cfrac{2}{n}\sum_{i=1}^n \boldsymbol{\epsilon}_i^{\top}(e^{A^*t_i}\boldsymbol{x}_0^*-e^{At_i}\boldsymbol{x}_0)\bigg)
    =&\cfrac{4}{n^2}\sum_{i=1}^n var\big(\boldsymbol{\epsilon}_i^{\top}(e^{A^*t_i}\boldsymbol{x}_0^*-e^{At_i}\boldsymbol{x}_0)\big)\\
    = \cfrac{4}{n^2}\sum_{i=1}^n E[\{\boldsymbol{\epsilon}_i^{\top}(e^{A^*t_i}\boldsymbol{x}_0^*-e^{At_i}\boldsymbol{x}_0)\}^2]
    \leq &\cfrac{4}{n^2}\sum_{i=1}^n E[\parallel\boldsymbol{\epsilon}_i\parallel_2^2 \parallel e^{A^*t_i}\boldsymbol{x}_0^*-e^{At_i}\boldsymbol{x}_0\parallel_2^2] \, ,\\
\end{split}
\end{equation}
by Cauchy-Schwarz inequality. 
Specifically,
\begin{equation}\label{eq:exp(A*t)x*-exp(At)x}
\begin{split}
    \parallel e^{A^* t_i }\boldsymbol{x}_0^* -e^{At_i}\boldsymbol{x}_0\parallel_2
    = &\parallel e^{A^* t_i }\boldsymbol{x}_0^* 
    - e^{A^* t_i }\boldsymbol{x}_0 
    + e^{A^* t_i }\boldsymbol{x}_0 -e^{At_i}\boldsymbol{x}_0\parallel_2\\
    = & \parallel e^{A^* t_i }(\boldsymbol{x}_0^*-\boldsymbol{x}_0) + 
    (e^{A^* t_i }-e^{At_i})\boldsymbol{x}_0\parallel_2\\
    \leq & \parallel e^{A^* t_i }(\boldsymbol{x}_0^*-\boldsymbol{x}_0)\parallel_2 + \parallel (e^{A^* t_i }-e^{At_i})\boldsymbol{x}_0\parallel_2\\
    \leq &\parallel e^{A^* t_i }\parallel_2 \parallel \boldsymbol{x}_0^*-\boldsymbol{x}_0\parallel_2 + \parallel e^{A^* t_i }-e^{At_i}\parallel_2 \parallel \boldsymbol{x}_0 \parallel_2\\
    \leq &
    \underbrace{\parallel e^{A^* t_i }\parallel_F}_\text{item 1}
    \underbrace{\parallel \boldsymbol{x}_0^*-\boldsymbol{x}_0\parallel_2}_\text{item 2} + 
    \underbrace{\parallel e^{A^* t_i }-e^{At_i}\parallel_F}_\text{item 3}
    \underbrace{\parallel \boldsymbol{x}_0 \parallel_2}_\text{item 4}\, ,
\end{split}
\end{equation}
where $\parallel M \parallel_2$ denotes the subordinate matrix norm induced by the norm $\parallel \cdot \parallel_2$ of a matrix $M \in \mathbb{R}^{m\times n}$, with 
\begin{equation*}
    \parallel M \parallel_2 \,= \sup_{ \boldsymbol{v}\in \mathbb{R}^n ,\boldsymbol{v} \neq \boldsymbol{0}}\cfrac{\parallel M\boldsymbol{v} \parallel_2}{\parallel \boldsymbol{v}\parallel_2}\,,
\end{equation*}
and $\parallel M \parallel_F$ is the Frobenius norm of matrix $M$, with 
\begin{equation*}
    \parallel M \parallel_F \,= \sqrt{ \sum_{i=1}^m \sum_{j=1}^n|M_{ij}|^2}\,,
\end{equation*}
where $M_{ij}$ is the $ij$-th entry of matrix $M$.

Under assumption A4, parameter space $\Theta$ is a compact subset of $\mathbb{R}^{d+d^2}$, therefore, $\Theta$ can be enclosed with a $(d+d^2)$-dimensional box. We denote it as $\Theta \subset [-l, l]^{d+d^2}$ with $0<l< \infty$. Then we will analyse each of the four items in Equation~\eqref{eq:exp(A*t)x*-exp(At)x}.
\begin{equation*}
\begin{split}\label{eq:exp(A*t)}
    \text{item 1} = &\parallel e^{A^* t_i} \parallel_F \,
    = \, \parallel e^{A^* t_i}  - e^{I_d t_i} + e^{I_d t_i}\parallel_F 
    \leq \,\parallel e^{A^* t_i}  - e^{I_d t_i} \parallel_F + \parallel e^{I_d t_i}  \parallel_F\\
    \leq &\parallel A^* t_i-I_d t_i\parallel_F e^{\parallel I_d t_i \parallel_F} e^{\parallel A^* t_i - I_d t_i \parallel_F}+ \parallel e^{I_d t_i}\parallel_F \,,
\end{split}
\end{equation*}
where $I_d$ denotes the $d$-dimensional identity matrix. Simple calculation shows that
\begin{equation*}
    \parallel A^* - I_d \parallel_F \\
    \leq \sqrt{ (l+1)^2 \times d + l^2\times (d^2 -d) }\\
    = \sqrt{l^2d^2 + 2ld +d} \, ,
\end{equation*}
\begin{equation*}
  \parallel I_d \parallel_F \,= \sqrt{d} \, ,
\end{equation*}
and
\begin{equation*}
    \parallel e^{I_d t_i} \parallel_F \,=\sqrt{e^{2t_i} \times d} = e^{t_i} \sqrt{d}\, .
\end{equation*}
Since $t_i \leq T$ for all $i = 1,\ldots, n$ and $0<T<\infty$ by condition A6, therefore, one obtains that
\begin{equation}\label{eq:item1}
    \text{item 1} = \,\parallel e^{A^* t_i} \parallel_F
    \,\leq C_d' \,,
\end{equation}
where $0 < C_d' < \infty$ is a constant only depending on $d$.
Similarly,
\begin{equation}\label{eq:item2}
    \text{item 2} = \,\parallel \boldsymbol{x}_0^* - \boldsymbol{x}_0 \parallel_2 \,\leq \sqrt{(2l)^2 \times d} = 2l\sqrt{d} \, ,
\end{equation}
and 
\begin{equation*}\label{eq:exp(A*t)-exp(At)}
\begin{split}
    \text{item 3} = &\parallel e^{A^*t_i}-e^{At_i} \parallel_F
   \, \leq\, \parallel A^* t_i - A t_i\parallel_F e^{\parallel A t_i \parallel_F} 
    e^{\parallel A^* t_i - A t_i \parallel_F} \\
    = &\parallel A^* - A\parallel_F t_i e^{\parallel A  \parallel_F t_i} 
    e^{\parallel A^*  - A  \parallel_F t_i} \, ,
\end{split}
\end{equation*}
where 
\begin{equation*}
    \parallel A \parallel_F \,\leq \sqrt{l^2 \times d^2} = ld\, ,
\end{equation*}
\begin{equation*}
    \parallel A^* -A \parallel_F \,\leq \sqrt{(2l)^2 \times d^2} = 2ld \, ,
\end{equation*}
therefore, by simple calculation one obtains that
\begin{equation}\label{eq:item3}
    \text{item 3} = \,\parallel e^{A^*t_i}-e^{At_i} \parallel_F \,\leq C_d'' \, ,
\end{equation}
with $0 < C_d'' < \infty$, is a constant only depends on $d$.
One sees that
\begin{equation}\label{eq:item4}
    \text{item 4} = \,\parallel \boldsymbol{x}_{0}\parallel_2 \,\leq \sqrt{l^2 \times d} = l\sqrt{d}\,.
\end{equation}
Combining \eqref{eq:item1}, \eqref{eq:item2}, \eqref{eq:item3}, \eqref{eq:item4} and \eqref{eq:exp(A*t)x*-exp(At)x}, one sees that
\begin{equation*}
    \parallel e^{A^* t_i }\boldsymbol{x}_0^* -e^{At_i}\boldsymbol{x}_0\parallel_2 \,\leq C_d \, ,
\end{equation*}
where $0 < C_d < \infty$ is a constant only depends on $d$. Then (\ref{eq:variance}) can be expressed as
\begin{equation*}\label{eq:variance2}
\begin{split}
    &var\bigg(\cfrac{2}{n}\sum_{i=1}^n \boldsymbol{\epsilon}_i^{\top}(e^{A^*t_i}\boldsymbol{x}_0^*-e^{At_i}\boldsymbol{x}_0)\bigg)
    \leq \, \cfrac{4}{n^2}\sum_{i=1}^n E[\parallel\boldsymbol{\epsilon}_i\parallel_2^2 \parallel e^{A^*t_i}\boldsymbol{x}_0^*-e^{At_i}\boldsymbol{x}_0\parallel_2^2]\\
    \leq \, &4\,C_d^2E[\parallel \boldsymbol{\epsilon}\parallel_2^2]/n
    =\, 4\,C_d^2\sum_{j=1}^d\sigma_j^2 /n\, .
\end{split}
\end{equation*}
Therefore, (\ref{eq:the third item}) can be expressed as
\begin{equation*}
    P\bigg(\bigg| \cfrac{2}{n}\sum_{i=1}^n \boldsymbol{\epsilon}_i^{\top}(e^{A^*t_i}\boldsymbol{x}_0^*-e^{At_i}\boldsymbol{x}_0) -0\bigg|\geq \varepsilon\bigg)
    \leq \cfrac{4\,C_d^2}{n\varepsilon^2}\sum_{j=1}^d\sigma_j^2
    \rightarrow 0 \, , \mbox{ as } n \rightarrow \infty.
\end{equation*}
That is 
\begin{equation}\label{eq:Mn(theta)_item3}
    \cfrac{2}{n}\sum_{i=1}^n \boldsymbol{\epsilon}_i^{\top}(e^{A^*t_i}\boldsymbol{x}_0^*-e^{At_i}\boldsymbol{x}_0)  \xrightarrow {p} 0\, , \mbox{ as } n\rightarrow \infty\,.
\end{equation}
Combining \eqref{eq:Mn(theta)_item1}, \eqref{eq:Mn(theta)_item2}, \eqref{eq:Mn(theta)_item3} and \eqref{eq:Mn(theta)_decomposition}, one obtains that
\begin{equation}\label{eq:convergence}
    M_n(\boldsymbol{\theta})  \xrightarrow {p} \cfrac{1}{T}\int_0^T \parallel e^{A^*t}\boldsymbol{x}_0^*-e^{At}\boldsymbol{x}_0\parallel_2^2 dt
    +E[\parallel \boldsymbol{\epsilon} \parallel_2^2]\, ,\mbox{ as } n\rightarrow \infty\, .
\end{equation}
By the definition of $M_1(\boldsymbol{\theta})$ in Equation~\eqref{eq:M1(theta)}, the right-hand side of~\eqref{eq:convergence} is $M_1(\theta)$, that is,
\begin{equation*}
     M_n(\boldsymbol{\theta})  \xrightarrow {p} M_1(\boldsymbol{\theta})\, , \mbox{ as } n\rightarrow \infty\, .
\end{equation*}
Now that we have proved that condition~\eqref{eq:uniform_convergence_condition1} in Lemma~\ref{lemma:uniform convergence} is satisfied. Next, we will find a $\alpha >0$ and a $B_n = O_p(1)$ such that for all $\boldsymbol{\theta}_1, \boldsymbol{\theta}_2 \in \Theta$, condition~\eqref{eq:uniform_convergence_condition2} is met.

For any $\boldsymbol{\theta}_1 = (\boldsymbol{x}_{01}, A_1), \boldsymbol{\theta}_2 =(\boldsymbol{x}_{02}, A_2)  \in \Theta$, 
\begin{equation}\label{eq:Mn(theta1)-Mn(theta2)}
\begin{split}
     &\parallel M_n(\boldsymbol{\theta}_1)- M_n(\boldsymbol{\theta}_2)\parallel_2  \\
    =\,&\bigg |\cfrac{1}{n} \sum_{i=1}^n \parallel \boldsymbol{y}_i - e^{A_1t_i}\boldsymbol{x}_{01} \parallel_2^2 - \cfrac{1}{n} \sum_{i=1}^n \parallel \boldsymbol{y}_i - e^{A_2t_i}\boldsymbol{x}_{02} \parallel_2^2 \bigg | \\
    = \,&\bigg | \cfrac{1}{n} \sum_{i=1}^n \{\parallel e^{A_1t_i}\boldsymbol{x}_{01} \parallel _2^2 -\parallel e^{A_2t_i}\boldsymbol{x}_{02} \parallel _2^2 -2 (e^{A^*t_i}\boldsymbol{x}_0^* + \boldsymbol{\epsilon}_i)^\top (e^{A_1t_i}\boldsymbol{x}_{01}-e^{A_2t_i}\boldsymbol{x}_{02})\} \bigg |  \\
    \leq\, &\cfrac{1}{n}\sum_{i=1}^n \big | \parallel e^{A_1t_i}\boldsymbol{x}_{01} \parallel _2^2 -\parallel e^{A_2t_i}\boldsymbol{x}_{02} \parallel _2^2 \\
    &-2 (e^{A^*t_i}\boldsymbol{x}_0^*)^\top (e^{A_1t_i}\boldsymbol{x}_{01}-e^{A_2t_i}\boldsymbol{x}_{02})
    -2\boldsymbol{\epsilon}_i^\top (e^{A_1t_i}\boldsymbol{x}_{01}-e^{A_2t_i}\boldsymbol{x}_{02}) \big| \\
    \leq\, &\cfrac{1}{n}\sum_{i=1}^n \big \{\big | \parallel e^{A_1t_i}\boldsymbol{x}_{01} \parallel _2^2 -\parallel e^{A_2t_i}\boldsymbol{x}_{02} \parallel _2^2 \big |  \\
    &+ \big | 2 (e^{A^*t_i}\boldsymbol{x}_0^*)^\top (e^{A_1t_i}\boldsymbol{x}_{01}-e^{A_2t_i}\boldsymbol{x}_{02}) \big | 
    + \big | 2\boldsymbol{\epsilon}_i^\top (e^{A_1t_i}\boldsymbol{x}_{01}-e^{A_2t_i}\boldsymbol{x}_{02}) \big | \big \}\\
    \leq\, & \cfrac{1}{n}\sum_{i=1}^n \{\parallel e^{A_1t_i}\boldsymbol{x}_{01}-e^{A_2t_i}\boldsymbol{x}_{02} \parallel_2 ( \parallel e^{A_1t_i}\boldsymbol{x}_{01} \parallel _2 + \parallel e^{A_2t_i}\boldsymbol{x}_{02}  \parallel _2  )   \\
    &+ 2 \parallel e^{A^*t_i}\boldsymbol{x}_0^* \parallel_2  \cdot \parallel e^{A_1t_i}\boldsymbol{x}_{01}-e^{A_2t_i}\boldsymbol{x}_{02} \parallel_2  
    + 2 \parallel \boldsymbol{\epsilon}_i \parallel_2 \cdot \parallel e^{A_1t_i}\boldsymbol{x}_{01}-e^{A_2t_i}\boldsymbol{x}_{02} \parallel_2 \} \\
    = \,&\cfrac{1}{n}\sum_{i=1}^n\parallel e^{A_1t_i}\boldsymbol{x}_{01}-e^{A_2t_i}\boldsymbol{x}_{02} \parallel_2  ( \parallel e^{A_1t_i}\boldsymbol{x}_{01} \parallel _2 + \parallel e^{A_2t_i}\boldsymbol{x}_{02}  \parallel _2   
     + 2 \parallel e^{A^*t_i}\boldsymbol{x}_0^* \parallel_2 
    + 2\parallel \boldsymbol{\epsilon}_i \parallel_2 
    ).
\end{split}
\end{equation}
Similar to the process of analysing $\parallel e^{A^*t_i}\boldsymbol{x}_{0}^*-e^{At_i}\boldsymbol{x}_{0}\parallel_2$ in Equation~\eqref{eq:exp(A*t)x*-exp(At)x}, by some calculation one obtains that
\begin{equation*}
\begin{split}
    &\parallel e^{A_1 t_i }\boldsymbol{x}_{01} -e^{A_2t_i}\boldsymbol{x}_{02}\parallel_2
    \,\leq \,\parallel e^{A_1 t_i }\parallel_F \parallel \boldsymbol{x}_{01}-\boldsymbol{x}_{02}\parallel_2 + \parallel e^{A_1 t_i }-e^{A_2t_i}\parallel_F \parallel \boldsymbol{x}_{02} \parallel_2\\
    \leq &\,C_d'\parallel \boldsymbol{x}_{01}-\boldsymbol{x}_{02}\parallel_2 +\, C_d''' \parallel A_1 -A_2\parallel_F
   \, \leq \,\text{max}(C_d', C_d''') (\parallel \boldsymbol{x}_{01}-\boldsymbol{x}_{02}\parallel_2 + \parallel A_1 -A_2\parallel_F)\, ,
\end{split}
\end{equation*}
where $0< C_d', C_d'''<\infty$ are constants only depending on $d$.
Since
\begin{equation*}
\begin{split}
    &(\parallel \boldsymbol{x}_{01}-\boldsymbol{x}_{02}\parallel_2 + \parallel A_1 -A_2\parallel_F)^2
    \,\leq \,2(\parallel \boldsymbol{x}_{01}-\boldsymbol{x}_{02}\parallel_2^2 + \parallel A_1 -A_2\parallel_F^2)
    = \,2\parallel \boldsymbol{\theta}_1 -\boldsymbol{\theta}_2 \parallel_2^2\, ,
\end{split}
\end{equation*}
one obtains that
\begin{equation*}
    \parallel \boldsymbol{x}_{01}-\boldsymbol{x}_{02}\parallel_2 + \parallel A_1 -A_2\parallel_F \,
    \leq \sqrt{2}\parallel  \boldsymbol{\theta}_1 -\boldsymbol{\theta}_2 \parallel_2\, .
\end{equation*}
Therefore,
\begin{equation}\label{eq:eA1tx1-eA2tx2}
    \parallel e^{A_1 t_i }\boldsymbol{x}_{01} -e^{A_2t_i}\boldsymbol{x}_{02}\parallel_2
    \,\leq \,\sqrt{2}\,\text{max}(C_d', C_d''')\parallel  \boldsymbol{\theta}_1 -\boldsymbol{\theta}_2 \parallel_2 \, .
\end{equation}
Then we analyse the second item in~\eqref{eq:Mn(theta1)-Mn(theta2)}, some simple calculation shows that
\begin{equation}\label{eq:eA1tx1}
    \parallel e^{A_1 t_i} \boldsymbol{x}_{01}\parallel_2
   \, \leq \,\parallel e^{A_1 t_i} \parallel_2 \parallel \boldsymbol{x}_{01}\parallel_2
   \, \leq \,\parallel e^{A_1 t_i} \parallel_F \parallel \boldsymbol{x}_{01}\parallel_2
   \, \leq C_d'l\sqrt{d} \, .
\end{equation}
Similarly, one sees that
\begin{equation}\label{eq:eA2tx2}
    \parallel e^{A_2 t_i} \boldsymbol{x}_{02}\parallel_2\,, \parallel e^{A^* t_i} \boldsymbol{x}_{0}^*\parallel_2 \,\leq C_d'l\sqrt{d}\, .
\end{equation}
Combining \eqref{eq:eA1tx1-eA2tx2}, \eqref{eq:eA1tx1}, \eqref{eq:eA2tx2} and \eqref{eq:Mn(theta1)-Mn(theta2)}, one obtains that
\begin{equation*}
\begin{split}
    &\parallel M_n(\boldsymbol{\theta}_1) - M_n(\boldsymbol{\theta}_2)\parallel_2
   \, \leq \,\cfrac{1}{n}\sum_{i=1}^n\sqrt{2} \,\text{max}(C_d',C_d''')   
     (4C_d'ld+2\parallel \boldsymbol{\epsilon}_i \parallel_2) 
     \parallel \boldsymbol{\theta}_1 -\boldsymbol{\theta}_2 \parallel_2\\
    \leq  &\bigg(\tilde{C}_d^2 + \cfrac{2\tilde{C}_d}{n}\sum_{i=1}^n\parallel \boldsymbol{\epsilon}_i\parallel_2\bigg)
    \parallel\boldsymbol{\theta}_1 -\boldsymbol{\theta}_2 \parallel_2 \, ,
\end{split}
\end{equation*}
where $\tilde{C}_d = \text{max}(\sqrt{2} \,\text{max}(C_d',C_d'''),4C_d'ld)$ and $ 0<\tilde{C}_d<\infty $ is a constant that only depends on $d$. If one sets 
\begin{equation*}
    B_n := \tilde{C}_d^2 + \cfrac{2\tilde{C}_d}{n}\sum_{i=1}^n\parallel \boldsymbol{\epsilon}_i\parallel_2\,,
\end{equation*}
one sees
\begin{equation*}
    \parallel M_n(\boldsymbol{\theta}_1) - M_n(\boldsymbol{\theta}_2)\parallel_2\,
    \leq  B_n\parallel\boldsymbol{\theta}_1 -\boldsymbol{\theta}_2 \parallel_2\, ,
\end{equation*}
for all $\boldsymbol{\theta}_1, \boldsymbol{\theta}_2 \in \Theta$.
Let 
\begin{equation*}
    F_n := \frac{1}{n}\sum_{i=1}^n \parallel \boldsymbol{\epsilon}_i \parallel_2\,,
\end{equation*}
by Chebyshev's inequality, one sees that
\begin{equation*}
    F_n = O_p(E[F_n] + \sqrt{var(F_n)})\,,
\end{equation*}
where
\begin{equation*}
    E[F_n] = E\bigg[\cfrac{1}{n}\sum_{i=1}^n \parallel \boldsymbol{\epsilon}_i \parallel_2\bigg]
    = \cfrac{1}{n}\sum_{i=1}^n E[\parallel \boldsymbol{\epsilon}_i \parallel_2] 
    = E[\parallel \boldsymbol{\epsilon}\parallel_2]
    = O(1) \, ,
\end{equation*}
because $E[\parallel \boldsymbol{\epsilon} \parallel_2^2] = \sum_{i=1}^d \sigma_i^2 < \infty$ implies that $E[\parallel \boldsymbol{\epsilon}\parallel_2] < \infty$. And 
\begin{equation*}
\begin{split}
    var(F_n) = &var\bigg(\cfrac{1}{n}\sum_{i=1}^n \parallel \boldsymbol{\epsilon}_i \parallel_2\bigg)
    = \cfrac{1}{n^2}\sum_{i=1}^n var(\parallel \boldsymbol{\epsilon}_i \parallel_2)\\
    = &\cfrac{1}{n^2}\sum_{i=1}^n\{E[\parallel \boldsymbol{\epsilon}_i \parallel_2^2] -(E[\parallel \boldsymbol{\epsilon}_i \parallel_2])^2\}
    = O(1)\, ,
\end{split}
\end{equation*}
therefore, one obtains $F_n = O_p(1)$, which implies $B_n = O_p(1)$. Thus, the condition~\eqref{eq:uniform_convergence_condition2} is satisfied with $\alpha = 1$. Since $M_1(\boldsymbol{\theta})$ is continuous w.r.t $\boldsymbol{\theta}$, and under assumption A4, parameter space $\Theta$ is compact, therefore, by Lemma~\ref{lemma:uniform convergence}, one sees that
\begin{equation*}
 \sup_{\boldsymbol{\theta}\in \Theta}\parallel M_n(\boldsymbol{\theta} ) - M_1(\boldsymbol{\theta})\parallel_2 \, \xrightarrow {p} 0\, , \mbox{ as } n\rightarrow \infty \,,
\end{equation*}
that is, the condition~\eqref{eq:consistency_condition1} in Lemma~\ref{lemma:M_estimator consistency} is satisfied.

\noindent{\bf Step \romannumeral 2}: We show that condition~\eqref{eq:consistency_condition2} in Lemma~\ref{lemma:M_estimator consistency} is satisfied.

\noindent Recall that
\begin{equation*}
    M_1(\boldsymbol{\theta}) = \cfrac{1}{T}\int_0^T \parallel e^{A^*t}\boldsymbol{x}_0^*-e^{At}\boldsymbol{x}_0\parallel_2^2 dt
    +E[\parallel \boldsymbol{\epsilon} \parallel_2^2]\, ,
\end{equation*}
when assumptions A1 and A2 are satisfied with respect to $\boldsymbol{\theta}^*$, according to Theorem~\ref{theorem:identifiability from discrete observations}, the ODE system~\eqref{eq:ODE model} is identifiable at $\boldsymbol{\theta}^* = (\boldsymbol{x}_0^*, A^*)$ from any $d+1$ equally-spaced error-free observations, which implies the ODE system is also identifiable at  $\boldsymbol{\theta}^*$ from the corresponding trajectory at $[0,T]$ , which further implies that $M_1(\boldsymbol{\theta})$ attains its unique global minimum at $\boldsymbol{\theta}^*$. Therefore, one sees that the condition~\eqref{eq:consistency_condition2} in Lemma~\ref{lemma:M_estimator consistency} is satisfied.

\noindent{\bf Step \romannumeral 3}: We show that condition~\eqref{eq:consistency_condition3} in Lemma~\ref{lemma:M_estimator consistency} is satisfied.

\noindent The definition of $\hat{\boldsymbol{\theta}}_n$, that is
\begin{equation*}
    \hat{\boldsymbol{\theta}}_n = \arg \min_{\boldsymbol{\theta}\in \Theta}M_n(\boldsymbol{\theta}) \,,
\end{equation*}
implies that the condition~\eqref{eq:consistency_condition3} is satisfied.\\

Now that we have proved that all the three conditions in Lemma~\ref{lemma:M_estimator consistency} are satisfied, therefore, one concludes that
\begin{equation*}
   \hat{\boldsymbol{\theta}}_n \xrightarrow {p} \boldsymbol{\theta}^*\, ,\mbox{ as } n\rightarrow \infty \, .
\end{equation*}
\end{proof}

\subsection{Proof of Corollary~\ref{corollary:consistency aggregate}}\label{proof:corollary3.1.1}
\begin{proof}
The main task in this proof is to prove that the NLS parameter estimator $\hat{\tilde{\boldsymbol{\theta}}}:= (\hat{\tilde{\boldsymbol{x}}}_0,\hat{\tilde{A}})$ from aggregated observations is consistent to the true system parameters corresponding to the new ODE system, that is, $\tilde{\boldsymbol{\theta}}^*:=(\tilde{\boldsymbol{x}}^*_0, \tilde{A}^*)$. Once one has proved this result, one can reach the conclusion in Corollary~\ref{corollary:consistency aggregate} by taking the function $g(\cdot)$ with respect to $\hat{\tilde{\boldsymbol{\theta}}}$ and $\tilde{\boldsymbol{\theta}}^*$, respectively.

In order to prove
\begin{equation*}
    \hat{\tilde{\boldsymbol{\theta}}} \xrightarrow{p} \tilde{\boldsymbol{\theta}}^*\,, \mbox{ as } \tilde{n}\rightarrow \infty\,,
\end{equation*}
base on Theorem~\ref{theorem:consistency}, one needs to prove that the assumptions A1, A2 and A4-A6 are satisfied with respect to the new ODE system, the new parameter $\tilde{\boldsymbol{\theta}}^*$ and the new error terms $\tilde{\boldsymbol{\epsilon}}_i$ corresponding to the aggregated observations $\tilde{\boldsymbol{Y}}$.

Under assumptions A1-A3, according to the proof of Corollary~\ref{corollary:aggregate} in Appendix~\ref{proof:corollary2.2.1}, one sees that assumptions A1-A2 are satisfied with respect to $\tilde{\boldsymbol{\theta}}^*$, where
\begin{equation}\label{eq:aggregated tilde_theta and theta}
    \tilde{\boldsymbol{\theta}}^* = (\tilde{\boldsymbol{x}}_0^*, \tilde{A}^*) = \big((I+e^{A^*\Delta_t} + \cdots + e^{A^*(k-1)\Delta_t})\boldsymbol{x}_0^*/k, A^*\big)\,.
\end{equation}
Since under assumption A4, $\Theta$ is compact, and according to the relationship between $\tilde{\boldsymbol{\theta}}^*$ and $\boldsymbol{\theta}^* = (\boldsymbol{x}_0^*, A^*)$ in Equation~\eqref{eq:aggregated tilde_theta and theta}, one obtains that the new parameter space $\tilde{\Theta}$ is compact.

Based on the generation rules of the aggregated observations, assumption A5 is satisfied with the new error terms $\{\tilde{\boldsymbol{\epsilon}}_i\}$ being independent and identically distributed random vectors with mean zero and covariance matrix 
\begin{equation*}
    \tilde{\Sigma} =  \Sigma/k = diag(\frac{\sigma_1^2}{k},\ldots, \frac{\sigma_d^2}{k})\,.
\end{equation*}

By aggregated observations generation rules, assumption A6 is satisfied with $\tilde{t}_i = t_{(i-1)k+1}$, $\Delta \tilde{t} = k \Delta t$ and 
\begin{equation*}
    \tilde{T} = (\lfloor n/k\rfloor-1)kT/(n-1)\,,
\end{equation*}
where $\lfloor \cdot \rfloor $ stands for the floor function. 

Then by Theorem \ref{theorem:consistency}, one concludes that $\hat{\tilde{\boldsymbol{\theta}}}\xrightarrow{p} \tilde{\boldsymbol{\theta}}^*$, as $\tilde{n}\rightarrow \infty$.

By definition of $\hat{\boldsymbol{\theta}}_{\tilde{n}}$ in Corollary~\ref{corollary:consistency aggregate}, that is
\begin{equation*}
    \hat{\boldsymbol{\theta}}_{\tilde{n}}:=g(\hat{\tilde{\boldsymbol{\theta}}}) := \big(k(I+e^{\hat{\tilde{A}}\Delta_t} + \cdots + e^{\hat{\tilde{A}}(k-1)\Delta_t})^{-1}\hat{\tilde{\boldsymbol{x}}}_0, \hat{\tilde{A}}\big)\,,
\end{equation*}
one obtains that 
\begin{equation*}
\begin{split}
    g(\tilde{\boldsymbol{\theta}}^*) &= \big(k(I+e^{\tilde{A}^*\Delta_t} + \cdots + e^{\tilde{A}^*(k-1)\Delta_t})^{-1}\tilde{\boldsymbol{x}}_0^*, \tilde{A}^*\big)
    = (\boldsymbol{x}_0^*, A^*)
    = \boldsymbol{\theta}^*\,.
\end{split}
\end{equation*}
By multivariate continuous mapping theorem, one concludes that
\begin{equation*}
    g(\hat{\tilde{\boldsymbol{\theta}}}) \xrightarrow{p} g(\tilde{\boldsymbol{\theta}}^*)\,, \mbox{ as } \tilde{n}\rightarrow \infty\,,
\end{equation*}
that is
\begin{equation*}
    \hat{\boldsymbol{\theta}}_{\tilde{n}} \xrightarrow{p}\boldsymbol{\theta}^*\,, \mbox{ as } \tilde{n}\rightarrow \infty\,.
\end{equation*}
\end{proof}

\subsection{Proof of Corollary~\ref{corollary:consistency timescaled}}\label{proof:corollary3.1.2}
\begin{proof}
Similar to the proof of Corollary~\ref{corollary:consistency aggregate} in Appendix~\ref{proof:corollary3.1.1}, one first needs to prove that the NLS parameter estimator $\hat{\tilde{\boldsymbol{\theta}}}:= (\hat{\tilde{\boldsymbol{x}}}_0,\hat{\tilde{A}})$ from time-scaled observations is consistent to the true system parameters corresponding to the new ODE system, that is, $\tilde{\boldsymbol{\theta}}^*:=(\tilde{\boldsymbol{x}}^*_0, \tilde{A}^*)$. Once one has proved this result, then the conclusion in Corollary~\ref{corollary:consistency timescaled} can be reached by taking the function $g(\cdot)$ with respect to $\hat{\tilde{\boldsymbol{\theta}}}$ and $\tilde{\boldsymbol{\theta}}^*$, respectively.

In order to prove
\begin{equation*}
    \hat{\tilde{\boldsymbol{\theta}}} \xrightarrow{p} \tilde{\boldsymbol{\theta}}^*\,, \mbox{ as } \tilde{n}\rightarrow \infty\,,
\end{equation*}
base on Theorem~\ref{theorem:consistency}, one needs to prove that the assumptions A1, A2 and A4-A6 are satisfied with respect to the new ODE system, the new parameter $\tilde{\boldsymbol{\theta}}^*$ and the new error terms $\tilde{\boldsymbol{\epsilon}}_i$ corresponding to the time-scaled observations $\tilde{\boldsymbol{Y}}$.

Under assumptions A1-A3, according to the proof of Corollary~\ref{corollary:timescaled} in Appendix~\ref{proof:corollary2.2.2}, one sees that assumptions A1-A2 are satisfied with respect to $\tilde{\boldsymbol{\theta}}^*$, where
\begin{equation}\label{eq:time-scaled tilde_theta and theta}
    \tilde{\boldsymbol{\theta}}^* = (\tilde{\boldsymbol{x}}_0^*,\tilde{A}^*)
    =(\boldsymbol{x}_0^*, A^*/k)\,.
\end{equation}
Since under assumption A4, $\Theta$ is compact, and according to the relationship between $\tilde{\boldsymbol{\theta}}^*$ and $\boldsymbol{\theta}^* = (\boldsymbol{x}_0^*, A^*)$ in Equation~\eqref{eq:time-scaled tilde_theta and theta}, one obtains that the new parameter space $\tilde{\Theta}$ is compact.

Based on the generation rules of the time-scaled observations, assumption A5 is satisfied with the new error terms $\{\tilde{\boldsymbol{\epsilon}}_i\}$ being the original error terms $\{\boldsymbol{\epsilon}_i\}$.

Assumption A6 is satisfied with
\begin{equation*}
    \tilde{t}_i = kt_i\,, 
    \Delta \tilde{t} = k \Delta t\,,
    \tilde{T} = kT\,.
\end{equation*}

Then by Theorem \ref{theorem:consistency}, one concludes that $\hat{\tilde{\boldsymbol{\theta}}}\xrightarrow{p} \tilde{\boldsymbol{\theta}}^*$, as $\tilde{n}\rightarrow \infty$.

By definition of $\hat{\boldsymbol{\theta}}_{\tilde{n}}$ in Corollary~\ref{corollary:consistency timescaled}, that is
\begin{equation*}
    \hat{\boldsymbol{\theta}}_{\tilde{n}}:=g(\hat{\tilde{\boldsymbol{\theta}}}) := (\hat{\tilde{\boldsymbol{x}}}_0, k\hat{\tilde{A}})\,,
\end{equation*}
one obtains that 
\begin{equation*}
    g(\tilde{\boldsymbol{\theta}}^*) 
    = (\tilde{\boldsymbol{x}}_0^*, k\tilde{A}^*\big)
    = (\boldsymbol{x}_0^*, A^*)
    = \boldsymbol{\theta}^*\,.
\end{equation*}
By multivariate continuous mapping theorem, one concludes that
\begin{equation*}
    g(\hat{\tilde{\boldsymbol{\theta}}}) \xrightarrow{p} g(\tilde{\boldsymbol{\theta}}^*)\,, \mbox{ as } \tilde{n}\rightarrow \infty\,,
\end{equation*}
that is
\begin{equation*}
    \hat{\boldsymbol{\theta}}_{\tilde{n}} \xrightarrow{p}\boldsymbol{\theta}^*\,, \mbox{ as } \tilde{n}\rightarrow \infty\,.
\end{equation*}

\end{proof}

\subsection{Proof of Theorem~\ref{theorem:Asymptotic normality}}\label{proof:theorem3.2}
\begin{proof}
Recall that we have defined $M_n(\boldsymbol{\theta})$ in Equation~\eqref{eq:Mn(theta)}, one sees that $M_n(\boldsymbol{\theta})$ is twice differentiable at $\boldsymbol{\theta} \in \Theta$. Then by the mean value theorem, one obtains
\begin{equation}\label{eq:mean value theorem}
    \nabla_{\boldsymbol{\theta}} M_n(\hat{\boldsymbol{\theta}}_n) = \nabla_{\boldsymbol{\theta}} M_n(\boldsymbol{\theta}^*)+ \nabla_{\boldsymbol{\theta}}^2M_n(\tilde{\boldsymbol{\theta}}) (\hat{\boldsymbol{\theta}}_n-\boldsymbol{\theta}^*)\,,
\end{equation}
where $\nabla_{\boldsymbol{\theta}} M_n(\boldsymbol{\theta}^*)$ denotes the gradient of $M_n(\boldsymbol{\theta})$ with respect to $\boldsymbol{\theta}$ at $\boldsymbol{\theta}^*$, and $\nabla_{\boldsymbol{\theta}}^2M_n(\boldsymbol{\theta}^*)$ denotes the Hessian matrix of $M_n(\boldsymbol{\theta})$ with respect to $\boldsymbol{\theta}$ at $\boldsymbol{\theta}^*$, and $\tilde{\boldsymbol{\theta}}$ is in the line joining $\hat{\boldsymbol{\theta}}_n$ and $\boldsymbol{\theta}^*$. 

Since assumptions A1 and A2 are satisfied with respect to $\boldsymbol{\theta}^*$ and assumptions A4-A6 hold, then by Theorem~\ref{theorem:consistency}, one obtains that 
\begin{equation*}
    \hat{\boldsymbol{\theta}}_n \xrightarrow{p} \boldsymbol{\theta}^*\,,\mbox{ as } n\rightarrow \infty\,.
\end{equation*}
Assumption A7 implies that $\hat{\boldsymbol{\theta}}_n$ is an interior point of $\Theta$ as $n\rightarrow \infty$, and by definition, 
\begin{equation*}
    \hat{\boldsymbol{\theta}}_n = \arg\min_{\boldsymbol{\theta}}M_n(\boldsymbol{\theta})\,,
\end{equation*}
therefore, one obtains that 
\begin{equation*}
    \nabla_{\boldsymbol{\theta}} M_n(\hat{\boldsymbol{\theta}}_n) = \boldsymbol{0}\,.
\end{equation*}  
Suppose that $\nabla_{\boldsymbol{\theta}}^2M_n(\tilde{\boldsymbol{\theta}})$ is nonsingular, by rearranging Equation~\eqref{eq:mean value theorem}, one obtains that
\begin{equation*}
    \sqrt{n} (\hat{\boldsymbol{\theta}}_n-\boldsymbol{\theta}^*) = -\{\nabla_{\boldsymbol{\theta}}^2M_n(\tilde{\boldsymbol{\theta}})\}^{-1}\sqrt{n}\nabla_{\boldsymbol{\theta}} M_n(\boldsymbol{\theta}^*)\,.
\end{equation*} 
By definition of $\tilde{\boldsymbol{\theta}}$ and the convergence of $\hat{\boldsymbol{\theta}}_n$ to $\boldsymbol{\theta}^*$, one obtains that
\begin{equation*}
    \tilde{\boldsymbol{\theta}}\xrightarrow{p}\boldsymbol{\theta}^*\,, \mbox{ as } n\rightarrow \infty\,,
\end{equation*}
and according to the proof of Theorem~\ref{theorem:consistency} in Appendix~\ref{proof:theorem3.1},
 \begin{equation*}
     \sup_{\boldsymbol{\theta}\in \Theta}\parallel M_n(\boldsymbol{\theta} ) - M_1(\boldsymbol{\theta})\parallel_2 \, \xrightarrow {p} 0\, , \mbox{ as } n\rightarrow \infty \,,
 \end{equation*}
one obtains that
\begin{equation*}
    \nabla_{\boldsymbol{\theta}}^2M_n(\tilde{\boldsymbol{\theta}}) \xrightarrow{p}  \nabla_{\boldsymbol{\theta}}^2M_1(\boldsymbol{\theta}^*)\,,\mbox{ as }n\rightarrow \infty\,,
\end{equation*}
By the relationship between $M_1(\boldsymbol{\theta})$ and $M(\boldsymbol{\theta})$ defined in Equation~\eqref{eq:M1(theta)}, one sees that
\begin{equation*}
    \nabla_{\boldsymbol{\theta}}^2M_1(\boldsymbol{\theta}^*) = \nabla_{\boldsymbol{\theta}}^2M(\boldsymbol{\theta}^*)\,,
\end{equation*}
therefore,
\begin{equation*}
    \nabla_{\boldsymbol{\theta}}^2M_n(\tilde{\boldsymbol{\theta}}) \xrightarrow{p}  \nabla_{\boldsymbol{\theta}}^2M(\boldsymbol{\theta}^*)\,,\mbox{ as }n\rightarrow \infty\,.
\end{equation*}

For the simplicity of notation, we set 
\begin{equation*}
    H := \nabla_{\boldsymbol{\theta}}^2M(\boldsymbol{\theta}^*)\,.
\end{equation*}
We will show that $H$ is positive definite, thus invertible.
In addition, if one shows that
\begin{equation}\label{eq:normal distribution}
    -\sqrt{n}\nabla_{\boldsymbol{\theta}} M_n(\boldsymbol{\theta}^*)\xrightarrow{d}N(\boldsymbol{0}, V)\,,\mbox{ as } n\rightarrow \infty\,,
\end{equation}
then, based on the Slutsky's theorem one obtains that 
\begin{equation*}
    \sqrt{n}(\hat{\boldsymbol{\theta}}_n -\boldsymbol{\theta}^* )\xrightarrow{d} N(\boldsymbol{0}, H^{-1}VH^{-1})\,, \mbox{ as } n\rightarrow \infty\,.
\end{equation*}
In the following, we will first prove $-\sqrt{n}\nabla_{\boldsymbol{\theta}} M_n(\boldsymbol{\theta}^*)$ converges in distribution to a normal distribution, that is,~\eqref{eq:normal distribution}, and calculate matrix $V$. Then we will calculate matrix $H$. 

\noindent{\bf Step \romannumeral 1}: We prove that $-\sqrt{n}\nabla_{\boldsymbol{\theta}}
M_n(\boldsymbol{\theta}^*)$ converges in distribution to a normal distribution, and we calculate matrix $V$.

By definition of the system parameter $\boldsymbol{\theta}$, one sees that
\begin{equation*}
    \boldsymbol{\theta} = (\boldsymbol{x}_0, A) = [\boldsymbol{x}_0^{\top}, a_{11},\ldots, a_{1d},\ldots, a_{dd}]^{\top}\in \mathbb{R}^{d+d^2}\,,
\end{equation*}
where $a_{jk}$ is the $jk$-th entry of parameter matrix $A$, for all $j,k=1,\ldots, d$.

Therefore, $\nabla_{\boldsymbol{\theta}} M_n(\boldsymbol{\theta})\in \mathbb{R}^{d+d^2}$, with
\begin{equation*}
    \nabla_{\boldsymbol{\theta}} M_n(\boldsymbol{\theta}) = \frac{\partial M_n(\boldsymbol{\theta})}{\partial \boldsymbol{\theta}} =\bigg[\bigg\{\frac{\partial M_n(\boldsymbol{\theta}) }{\partial \boldsymbol{x}_0}\bigg\}^{\top},
    \frac{\partial M_n(\boldsymbol{\theta}) }{\partial a_{11}},\ldots,  
    \frac{\partial M_n(\boldsymbol{\theta}) }{\partial a_{1d}},\ldots,
    \frac{\partial M_n(\boldsymbol{\theta}) }{\partial a_{dd}}
    \bigg]^{\top}\,.
\end{equation*}

Recall that 
\begin{equation}
    M_n(\boldsymbol{\theta}) = \cfrac{1}{n} \sum_{i=1}^n \parallel \boldsymbol{y}_i - e^{At_i}\boldsymbol{x}_0 \parallel_2^2 \, ,
\end{equation}
obviously, if one wants to calculate the partial derivative of $M_n(\boldsymbol{\theta})$ with respect to $a_{jk}$, that is, $\partial M_n(\boldsymbol{\theta})/\partial a_{jk}$, one needs to calculate $\partial e^{At}/\partial a_{jk}$ first.
Suppose that matrix $A$ has $d$ distinct eigenvalues $\lambda_1,\ldots, \lambda_d$, and $A$ has the Jordan decomposition 
$A = Q\Lambda Q^{-1}$,
where $\Lambda = diag(\lambda_1,\ldots, \lambda_d)$. Then according to \citep{kalbfleisch1985analysis,tsai2003note}, one obtains that 
\begin{equation*}
    Z_{jk}(t) := \cfrac{\partial e^{At}}{\partial a_{jk}} = Q \{(Q_{\cdot j}^{-1}Q_{k\cdot})\circ U(t)\} Q^{-1} \,,
\end{equation*}
here, for notational simplicity, we denote $\partial e^{At}/\partial a_{jk}$ as $Z_{jk}(t)$. The column vector $Q_{\cdot j}^{-1}$ stands for the $j$th column of matrix $Q^{-1}$ and the row vector $Q_{k\cdot}$ denotes the $k$th row of matrix $Q$. Let $B\circ C$ denote the Hadamard product, with each element 
\begin{equation*}
    (B\circ C)_{ij} = (B)_{ij}(C)_{ij}\,,
\end{equation*}
where matrices $B$ and $C$ are of the same dimension. $U(t)$ has the form:
\begin{equation*}
    U(t) = \begin{bmatrix}
    te^{\lambda_1t} &\cfrac{e^{\lambda_1 t}-e^{\lambda_2 t}}{\lambda_1-\lambda_2}  & \cdots & \cfrac{e^{\lambda_1 t}-e^{\lambda_d t}}{\lambda_1-\lambda_d}\\ \\
    \cfrac{e^{\lambda_2 t}-e^{\lambda_1 t}}{\lambda_2-\lambda_1} & te^{\lambda_2 t} & \cdots & \cfrac{e^{\lambda_2 t}-e^{\lambda_d t}}{\lambda_2-\lambda_d}\\
    \vdots & \vdots  & \ddots & \vdots \\
    \cfrac{e^{\lambda_d t}-e^{\lambda_1 t}}{\lambda_d-\lambda_1} & \cfrac{e^{\lambda_d t}-e^{\lambda_2 t}}{\lambda_d-\lambda_2} &  \cdots & te^{\lambda_d t}
    \end{bmatrix}\,.
\end{equation*}

In the following, we will calculate $\nabla_{\boldsymbol{\theta}}M_n(\boldsymbol{\theta})$.
Simple calculation shows that
\begin{equation*}
\begin{split}
    M_n(\boldsymbol{\theta}) &= \cfrac{1}{n}\sum_{i=1}^n\parallel \boldsymbol{y}_i - e^{At_i}\boldsymbol{x}_0 \parallel_2^2
    =\cfrac{1}{n}\sum_{i=1}^n (\boldsymbol{y}_i - e^{At_i}\boldsymbol{x}_0)^{\top} (\boldsymbol{y}_i - e^{At_i}\boldsymbol{x}_0)\\
    &= \cfrac{1}{n}\sum_{i=1}^n \{\boldsymbol{y}_i^{\top}\boldsymbol{y}_i - \boldsymbol{x}_0^{\top}(e^{At_i})^{\top}\boldsymbol{y}_i - \boldsymbol{y}_i^{\top}e^{At_i}\boldsymbol{x}_0 + \boldsymbol{x}_0^{\top}(e^{At_i})^{\top}e^{At_i}\boldsymbol{x}_0\}\,.
\end{split}
\end{equation*}
Then one obtains that
\begin{equation*}
    \cfrac{\partial M_n(\boldsymbol{\theta})}{\partial \boldsymbol{x}_0} 
    = \cfrac{1}{n}\sum_{i=1}^n \{-2(e^{At_i})^{\top}\boldsymbol{y}_i+2(e^{At_i})^{\top}e^{At_i}\boldsymbol{x}_0\}\,,
\end{equation*}
\begin{equation*}
\begin{split}
    \cfrac{\partial  M_n(\boldsymbol{\theta})}{\partial a_{jk}} 
    &= \cfrac{1}{n}\sum_{i=1}^n Tr\bigg \{\bigg( \cfrac{\partial  \parallel \boldsymbol{y}_i - e^{At_i}\boldsymbol{x}_0 \parallel_2^2}{\partial e^{At_i}}\bigg)^{\top} \cfrac{\partial e^{At_i}}{\partial a_{jk}}\bigg \}\\
    &= \cfrac{1}{n}\sum_{i=1}^n Tr \{(-\boldsymbol{y}_i\boldsymbol{x}_0^{\top}-\boldsymbol{y}_i\boldsymbol{x}_0^{\top}+2e^{At_i}\boldsymbol{x}_0\boldsymbol{x}_0^{\top})^{\top}Z_{jk}(t_i) \}\\
    &= \cfrac{1}{n}\sum_{i=1}^n Tr[\{-2\boldsymbol{x}_0 \boldsymbol{y}_i^{\top} + 2 \boldsymbol{x}_0 \boldsymbol{x}_0^{\top} (e^{At_i})^{\top}\}Z_{jk}(t_i)]\\
   &= \cfrac{1}{n}\sum_{i=1}^n \{-2\boldsymbol{y}_i^{\top}Z_{jk}(t_i)\boldsymbol{x}_0 + 2 \boldsymbol{x}_0^{\top} (e^{At_i})^{\top}Z_{jk}(t_i)\boldsymbol{x}_0\} \,.
\end{split}
\end{equation*}
Therefore, one obtains that
\begin{equation*}
\begin{split}
    \cfrac{\partial  M_n(\boldsymbol{\theta}^*)}{\partial \boldsymbol{x}_0} 
    :&= \cfrac{\partial M_n(\boldsymbol{\theta})}{\partial \boldsymbol{x}_0} \bigg |_{\boldsymbol{\theta} = \boldsymbol{\theta}^*}
    = \cfrac{1}{n}\sum_{i=1}^n \{-2(e^{A^*t_i})^{\top}\boldsymbol(e^{A^*t_i}\boldsymbol{x}_0^* + \boldsymbol{\epsilon}_i)+2(e^{A^*t_i})^{\top}e^{A^*t_i}\boldsymbol{x}_0^*\}\\
    &= -\cfrac{2}{n}\sum_{i=1}^n (e^{A^*t_i})^{\top}\boldsymbol{\epsilon}_i\,,
\end{split}
\end{equation*}
\begin{equation*}
\begin{split}
    \cfrac{\partial  M_n(\boldsymbol{\theta}^*)}{\partial a_{jk}} :&=\cfrac{\partial  M_n(\boldsymbol{\theta})}{\partial a_{jk}} \bigg |_{\boldsymbol{\theta} = \boldsymbol{\theta}^*}
    = \cfrac{1}{n}\sum_{i=1}^n \{-2(e^{A^*t_i}\boldsymbol{x}_0^* + \boldsymbol{\epsilon}_i)^{\top}Z_{jk}^*(t_i)\boldsymbol{x}_0^* + 2 (\boldsymbol{x}_0^*)^{\top} (e^{A^*t_i})^{\top}Z_{jk}^*(t_i)\boldsymbol{x}_0^*\} \\
    &= -\cfrac{2}{n}\sum_{i=1}^n \boldsymbol{\epsilon}_i^{\top}Z_{jk}^*(t_i)\boldsymbol{x}_0^*\,,
\end{split}
\end{equation*}
where
\begin{equation*}
    Z_{jk}^*(t) := \cfrac{\partial e^{At}}{\partial a_{jk}}\bigg |_{A = A^*} = Q^*[ \{(Q^*)_{.j}^{-1}Q_{k.}^*\}\circ U^*(t)] (Q^*)^{-1} \,,
\end{equation*}
with $Q^*, U^*(t)$ corresponding to the true parameter matrix $A^*$. That is,
the Jordan decomposition of $A^*$ is $A^*=Q^*\Lambda^* (Q^*)^{-1}$, where $\Lambda ^* = diag(\lambda_1^*,\ldots, \lambda_d^*)$, with
$\lambda_1^*, \lambda_2^*,\ldots, \lambda_d^*$ being the eigenvalues of $A^*$, and under assumption A2, these eigenvalues are distinct real values. Set
$ V := \lim_{n\rightarrow \infty}var(\sqrt{n}\nabla_{\boldsymbol{\theta}} M_n(\boldsymbol{\theta}^*))$, then we will calculate $V$ in the following.
By some calculation, one obtains that
\begin{equation*}
\begin{split}
    var\bigg \{\cfrac{\sqrt{n}\partial  M_n(\boldsymbol{\theta}^*)}{\partial \boldsymbol{x}_0}\bigg \} &= n \cdot var\bigg \{-\cfrac{2}{n}\sum_{i=1}^n(e^{A^*t_i})^{\top}\boldsymbol{\epsilon}_i\bigg \}\\
    &=\cfrac{4}{n}\sum_{i=1}^n  var\{(e^{A^*t_i})^{\top}\boldsymbol{\epsilon}_i\}
    =\cfrac{4}{n}\sum_{i=1}^n (e^{A^*t_i})^{\top} \Sigma e^{A^*t_i}\\
    &\rightarrow \cfrac{4}{T}\int_0^T  (e^{A^*t})^{\top} \Sigma e^{A^*t} \,dt \, ,\mbox{ as } n\rightarrow \infty\,.
\end{split}
\end{equation*}
Recall that $\Sigma$ is the covariance matrix of error terms $\boldsymbol{\epsilon}_i$ for all $i=1,\ldots,n$ under Assumption A5.
Similarly, one obtains that 
\begin{equation*}
\begin{split}
    cov\bigg \{\cfrac{\sqrt{n}\partial  M_n(\boldsymbol{\theta}^*)}{\partial \boldsymbol{x}_0}, \cfrac{\sqrt{n}\partial  M_n(\boldsymbol{\theta}^*)}{\partial a_{jk}} \bigg \}
    &= nE\bigg [\cfrac{\partial  M_n(\boldsymbol{\theta}^*)}{\partial \boldsymbol{x}_0} \cdot \cfrac{\partial  M_n(\boldsymbol{\theta}^*)}{\partial a_{jk}}\bigg] - nE\bigg[\cfrac{\partial  M_n(\boldsymbol{\theta}^*)}{\partial \boldsymbol{x}_0}\bigg]E\bigg[\cfrac{\partial  M_n(\boldsymbol{\theta}^*)}{\partial a_{jk}}\bigg]\\
    &= nE\bigg[\cfrac{2}{n}\sum_{i=1}^n (e^{A^*t_i})^{\top} \boldsymbol{\epsilon}_i \cdot \cfrac{2}{n}\sum_{i=1}^n  \boldsymbol{\epsilon}_i^{\top}Z_{jk}^*(t_i)\boldsymbol{x}_0^*\bigg]\\
    &= \cfrac{4}{n}\sum_{i=1}^n E[(e^{A^*t_i})^{\top} \boldsymbol{\epsilon}_i \boldsymbol{\epsilon}_i^{\top}Z_{jk}^*(t_i)\boldsymbol{x}_0^* ]\\
    &= \cfrac{4}{n}\sum_{i=1}^n (e^{A^*t_i})^{\top} E[\boldsymbol{\epsilon}_i \boldsymbol{\epsilon}_i^{\top}]Z_{jk}^*(t_i)\boldsymbol{x}_0^* \\
    &= \cfrac{4}{n}\sum_{i=1}^n (e^{A^*t_i})^{\top} \Sigma Z_{jk}^*(t_i)\boldsymbol{x}_0^*\\
    &\rightarrow \cfrac{4}{T}\int_0^T  (e^{A^*t})^{\top} \Sigma Z_{jk}^*(t)\boldsymbol{x}_0^* \,dt \, ,\mbox{ as } n\rightarrow \infty\,,
\end{split}
\end{equation*}
and
\begin{equation*}
\begin{split}
    cov\bigg \{\cfrac{\sqrt{n}\partial  M_n(\boldsymbol{\theta}^*)}{\partial a_{jk}},  \cfrac{\sqrt{n}\partial  M_n(\boldsymbol{\theta}^*)}{\partial a_{pq}}\bigg \}
    &= \cfrac{4}{n}\sum_{i=1}^n \{Z_{jk}^*(t_i)\boldsymbol{x}_0^*\}^{\top} \Sigma Z_{pq}^*(t_i)\boldsymbol{x}_0^*\\
    &\rightarrow \cfrac{4}{T}\int_0^T  \{Z_{jk}^*(t)\boldsymbol{x}_0^* \}^{\top} \Sigma Z_{pq}^*(t)\boldsymbol{x}_0^* \,dt \, ,\mbox{ as } n\rightarrow \infty\,. 
\end{split}
\end{equation*}
If one denotes 
\begin{equation}\label{eq:R(theta*,t)}
    R(\boldsymbol{\theta}^*, t) := \Sigma^{1/2} \big(e^{A^*t}, Z_{11}^*(t) \boldsymbol{x}_0^*,\ldots, Z_{1d}^*(t) \boldsymbol{x}_0^*,\ldots, Z_{dd}^*(t) \boldsymbol{x}_0^*\big)\,,
\end{equation}
one sees that
\begin{equation}\label{eq:V_RR^T}
    V = 4\int_0^T  R(\boldsymbol{\theta}^*, t)^{\top} R(\boldsymbol{\theta}^*, t)/T \,dt\in \mathbb{R}^{(d+d^2)\times (d+d^2)}\,.
\end{equation}

Now that we have calculated $V$, then we will prove that $-\sqrt{n}\nabla_{\boldsymbol{\theta}}
M_n(\boldsymbol{\theta}^*)$ converges in distribution to a normal distribution. We first present a Lemma we will use for our proof.
\begin{lemma}\label{lemma:Lindeberg-Feller CLT}
\textnormal{[Lindeberg-Feller Central Limit Theorem]} Suppose $\{\boldsymbol{w}_{ni}\}$ is a triangular array of $p\times 1$ random vectors such that $\boldsymbol{s}_n = \sum_{i=1}^n\boldsymbol{w}_{ni}/n $ and 
\begin{equation*}
    V_n = \frac{1}{n}\sum_{i=1}^n var(\boldsymbol{w}_{ni}) \rightarrow V\,,
\end{equation*}
where $V$ is positive definite. If for every $\varepsilon > 0$,
\begin{equation}\label{eq:CLT_condition}
     \frac{1}{n}\sum_{i=1}^n E[\parallel \boldsymbol{w}_{ni} \parallel_2^2 \mathbbm{1}(\parallel \boldsymbol{w}_{ni} \parallel_2 \geq \varepsilon \sqrt{n})] \rightarrow 0\,,
\end{equation}
then 
\begin{equation*}
    \sqrt{n}\boldsymbol{s}_n \xrightarrow{d} N(0, V)\,.
\end{equation*}
\end{lemma}

\noindent If one sets
\begin{equation*}
    \boldsymbol{s}_n := -\nabla_{\boldsymbol{\theta}}M_n(\boldsymbol{\theta}^*)\,,
\end{equation*}
\begin{equation*}
    \boldsymbol{w}_{ni} :=2[\{ (e^{A^*t_i})^{\top}\boldsymbol{\epsilon}_i\}^{\top}, \boldsymbol{\epsilon}_i^{\top}Z_{11}^*(t_i) \boldsymbol{x}_0^*,\ldots, \boldsymbol{\epsilon}_i^{\top}Z_{1d}^*(t_i) \boldsymbol{x}_0^*, ,
    \cdots,
    \boldsymbol{\epsilon}_i^{\top}Z_{dd}^*(t_i) \boldsymbol{x}_0^*]^{\top}
    \,,
\end{equation*}
where $\boldsymbol{w}_{ni} \in \mathbb{R}^{d+d^2}$, then $V$ in Equation~\eqref{eq:V_RR^T} corresponds the $V$ in Lemma~\ref{lemma:Lindeberg-Feller CLT}. Then the proof of the asymptotic normality of $-\sqrt{n}\nabla_{\boldsymbol{\theta}}
M_n(\boldsymbol{\theta}^*)$ can be broken down into two tasks. Proving condition~\eqref{eq:CLT_condition} is satisfied and proving variance matrix $V$ is positive definite.

To this end, we first prove the condition~\eqref{eq:CLT_condition} is satisfied. Simple calculation show that
\begin{equation*}
    \parallel \boldsymbol{w}_{ni} \parallel_2^2
    \,= \,4\parallel (e^{A^*t_i})^{\top}\boldsymbol{\epsilon}_i\parallel_2^2 + 4\sum_{j=1}^d\sum_{k=1}^d\{ \boldsymbol{\epsilon}_i^{\top}Z_{jk}^*(t_i) \boldsymbol{x}_0^*\}^2 \in (-\infty, \infty)\,,
\end{equation*}
therefore, one obtains that
\begin{equation*}
    \lim_{n\rightarrow\infty}\parallel \boldsymbol{w}_{ni} \parallel_2^2 \mathbbm{1}(\parallel \boldsymbol{w}_{ni} \parallel_2 \geq \varepsilon \sqrt{n}) = 0\,,
\end{equation*}
almost surely. 
By calculation, one obtains that
\begin{equation*}
\begin{split}
   E[ \parallel (e^{A^*t_i})^{\top}\boldsymbol{\epsilon}_i\parallel_2^2] 
   &\leq E[\parallel e^{A^*t_i}\parallel_F^2 \parallel \boldsymbol{\epsilon}_i \parallel_2^2]
   = \,\parallel e^{A^*t_i}\parallel_F^2 E[\parallel \boldsymbol{\epsilon}_i \parallel_2^2]\\
   &= \parallel e^{A^*t_i}\parallel_F^2\sum_{j=1}^d\sigma_j^2
   < \infty\,,
\end{split}
\end{equation*}
and
\begin{equation*}
\begin{split}
    E[\{ \boldsymbol{\epsilon}_i^{\top}Z_{jk}^*(t_i) \boldsymbol{x}_0^*\}^2] 
    &= E[\{ Z_{jk}^*(t_i) \boldsymbol{x}_0^*\}^{\top}\boldsymbol{\epsilon}_i \boldsymbol{\epsilon}_i^{\top} Z_{jk}^*(t_i) \boldsymbol{x}_0^*]\\   
    &= \{ Z_{jk}^*(t_i) \boldsymbol{x}_0^*\}^{\top} E[\boldsymbol{\epsilon}_i \boldsymbol{\epsilon}_i^{\top}] Z_{jk}^*(t_i) \boldsymbol{x}_0^*\\
    &= \{ Z_{jk}^*(t_i) \boldsymbol{x}_0^*\}^{\top} \Sigma Z_{jk}^*(t_i) \boldsymbol{x}_0^*
    < \infty\,.
\end{split}
\end{equation*}
Therefore, $E[\parallel \boldsymbol{w}_{ni} \parallel_2^2] < \infty$. Since 
\begin{equation*}
    \parallel \boldsymbol{w}_{ni} \parallel_2^2 \mathbbm{1}(\parallel \boldsymbol{w}_{ni} \parallel_2 \geq \varepsilon \sqrt{n}) \leq \, \parallel \boldsymbol{w}_{ni} \parallel_2^2\,,
\end{equation*}
then by Lebesgue's dominated convergence theorem, one obtains that
\begin{equation*}
    \lim_{n\rightarrow \infty}E[\parallel \boldsymbol{w}_{ni} \parallel_2^2 \mathbbm{1}(\parallel \boldsymbol{w}_{ni} \parallel_2 \geq \varepsilon \sqrt{n})] = E[\lim_{n\rightarrow \infty}\parallel \boldsymbol{w}_{ni} \parallel_2^2 \mathbbm{1}(\parallel \boldsymbol{w}_{ni} \parallel_2 \geq \varepsilon \sqrt{n}) ] = 0\,.
\end{equation*}
Now that we have proved that the condition~\eqref{eq:CLT_condition} is satisfied, we will then prove $V$ is positive definite. If one denotes
\begin{equation*}
    W(t) := \big(e^{A^*t}, Z_{11}^*(t) \boldsymbol{x}_0^*,\ldots, Z_{1d}^*(t) \boldsymbol{x}_0^*,\ldots, Z_{dd}^*(t) \boldsymbol{x}_0^*\big)\,,
\end{equation*}
then according to Equation~\eqref{eq:R(theta*,t)} and Equation~\eqref{eq:V_RR^T}, one sees that
\begin{equation}\label{eq:V_RR^T2}
\begin{split}
    V := \cfrac{4}{T}\int_0^T V(t) \,dt
    = \cfrac{4}{T}\int_0^T  W(t)^{\top} \Sigma W(t)\,dt \,.
\end{split}
\end{equation}
In the following, we will show that $\int_{0}^T V(t)\,dt \in \mathbb{R}^{(d+d^2) \times(d+ d^2)}$ is positive definite. Let $\boldsymbol{\xi}\in\mathbb{R}^{d+d^{2}}$ be such that
\begin{equation*}
    \boldsymbol{\xi}^{T}\cdot\big(\int_{0}^{T}V(t)dt\big)\cdot\boldsymbol{\xi}=0.
\end{equation*}
We want to show that $\boldsymbol{\xi}=\boldsymbol{0}$. Since $V(t)$ is non-negative definitely for every $t$, this implies that 
\begin{equation*}
    \boldsymbol{\xi}^{T}V(t)\boldsymbol{\xi}=\boldsymbol{\xi}^{T}W(t)^{\top} \Sigma W(t)\boldsymbol{\xi}=0\, \forall t
\end{equation*}
which further implies that $W(t)\boldsymbol{\xi}=\boldsymbol{0}$ for all $t$. By differentiation, one sees that 
\begin{equation*}
    W^{(m)}(0)\boldsymbol{\xi}=\boldsymbol{0}\ \forall m=0,1,2,\ldots,
\end{equation*}
where $W^{(m)}(t)$ denotes the $m$th deivative of $W(t)$. 
As we will see, the first $d+1$ equations are sufficient to yield that $\boldsymbol{\xi}=\boldsymbol{0}$.
Recall that 
\begin{equation*}
    Z^*_{jk}(t) = \cfrac{\partial e^{At}}{\partial a_{jk}} \bigg |_{A = A^*}\,,
\end{equation*}
by denoting the first $d$ components of $\boldsymbol{\xi}$ as $\boldsymbol{\xi}_d$ and the $(j,k)$-component of the last $d^2$ components of  $\boldsymbol{\xi}$ as $\boldsymbol{\xi}_{d+jk},$ one obtains that
\begin{equation}\label{eq:W^(m)(0)xi}
\begin{split}
    W^{(m)}(0)\boldsymbol{\xi}
    &= \bigg(\cfrac{\partial ^{(m)} e^{A^*t}}{\partial t^m}\bigg|_{t=0},
    \cfrac{\partial ^{(m+1)} e^{At}}{ \partial t^m\partial  a_{11}}\bigg|_{t=0,A=A^*} \cdot\boldsymbol{x}_0^*,\ldots, \cfrac{\partial ^{(m+1)} e^{At}}{\partial t^m \partial a_{dd}}\bigg |_{t=0, A=A^*}\cdot \boldsymbol{x}_0^*\bigg)\boldsymbol{\xi} \\
    &= \bigg((A^*)^m,
    \cfrac{\partial A^m}{\partial a_{11}} \bigg|_{A = A^*}\cdot\boldsymbol{x}_0^*,\ldots, 
    \cfrac{\partial A^m}{\partial a_{dd}} \bigg|_{A = A^*}\cdot\boldsymbol{x}_0^*\bigg)\boldsymbol{\xi}\\
    &= \bigg( (A^*)^m,
    \sum_{l=1}^m (A^*)^{m-l}E_{11}(A^*)^{l-1}\boldsymbol{x}_0^*,\ldots, 
    \sum_{l=1}^m (A^*)^{m-l}E_{dd}(A^*)^{l-1}\boldsymbol{x}_0^* \bigg)\boldsymbol{\xi}\\
    &= (A^*)^m \boldsymbol{\xi}_d + \sum_{l=1}^m\sum_{j=1}^d \sum_{k=1}^d\{(A^*)^{m-l}E_{jk}(A^*)^{l-1}\boldsymbol{x}_0^*\} \boldsymbol{\xi}_{d+jk}\,,
\end{split}
\end{equation}
where $\sum^0_{i=1}a_i=0$ for any sequence $\{a_i,i\in \mathbb{Z}\}$ denotes the empty sum, $E_{jk}$ is a $d\times d$ matrix with the $jk$-th entry being $1$ and all the other entries being $0$.

If one identifies the last $d^2$ elements of $\boldsymbol{\xi}$ with a $d\times d$ matrix $\Xi$ (the $jk$-th entry of $\Xi$ being $\boldsymbol{\xi}_{d+jk}$), then \eqref{eq:W^(m)(0)xi} becomes 
\begin{equation*}
   (A^*)^m \boldsymbol{\xi}_d + \sum_{l=1}^{m}(A^*)^{m-l}\cdot\Xi\cdot (A^*)^{l-1}\boldsymbol{x}_0^*=\boldsymbol{0}\ \ \ \forall m\geqslant 0\,. 
\end{equation*}

Taking $m=0,1,\ldots,d$ respectively, one obtains the following system
of equations:
\begin{equation*}
\begin{cases}
\boldsymbol{\xi}_d =\boldsymbol{0}\,,\\
A^*\boldsymbol{\xi}_d + \Xi\cdot\boldsymbol{x}_0^*=\boldsymbol{0}\,,\\
(A ^*)^2 \boldsymbol{\xi}_d + \Xi\cdot A^*\boldsymbol{x}_0^*+(A^*\Xi)\cdot\boldsymbol{x}_0^*=\boldsymbol{0}\,,\\
(A ^*)^3 \boldsymbol{\xi}_d + \Xi\cdot (A^*)^{2}\boldsymbol{x}_0^*+(A^*\Xi)\cdot A^*\boldsymbol{x}_0^*+((A^*)^{2}\Xi)\cdot\boldsymbol{x}_0^*=\boldsymbol{0}\,,\\
\cdots\\
(A ^*)^d \boldsymbol{\xi}_d + \Xi\cdot (A^*)^{d-1}\boldsymbol{x}_0^*+(A^*\Xi)\cdot (A^*)^{d-2}\boldsymbol{x}_0^*+\cdots+((A^*)^{d-1}\Xi)\cdot\boldsymbol{x}_0^*=\boldsymbol{0}\,.
\end{cases}  
\end{equation*}

From the first equation, we have $\boldsymbol{\xi}_d = \boldsymbol{0}$.
If one multiplies the second equation by $A^*$ on the left and subtract
it from the third equation, one obtains that 
\begin{equation*}
\Xi\cdot A^*\boldsymbol{x}_0^*=\boldsymbol{0}\,.
\end{equation*}
Similarly,
\begin{equation*}
``(l+2)\text{-th eqn }"-\ A^*\times``(l+1)\text{-th eqn}"\implies\Xi\cdot (A^*)^{l}\boldsymbol{x}_0^*=\boldsymbol{0}\,.
\end{equation*}
As a result, 
\begin{equation*}
\Xi\cdot\big(\boldsymbol{x}_0^*,A^*\boldsymbol{x}_0^*,\ldots,(A^*)^{d-1}\boldsymbol{x}_0^*\big)=\boldsymbol{0}\,.
\end{equation*}
Since the matrix $\big(\boldsymbol{x}_0^*,A^*\boldsymbol{x}_0^*,\ldots,(A^*)^{d-1}\boldsymbol{x}_0^*\big)$
is invertible by assumption A2, one concludes that $\Xi=\boldsymbol{0}$, that is, the last $d^2$ components of $\boldsymbol{\xi}$ are all zeros.

Therefore, we have proved that $\boldsymbol{\xi} = \boldsymbol{0}$, which means, $\int_0^T V(t)\,dt$ is positive definite.  Thus, V is positive definite. Since $V$ is symmetric, $V$ is also nonsingular. By Lemma~\ref{lemma:Lindeberg-Feller CLT}, one concludes that 
\begin{equation*}
    -\sqrt{n}\nabla_{\boldsymbol{\theta}} M_n(\boldsymbol{\theta}^*)\xrightarrow{d}N(\boldsymbol{0}, V)\,,\mbox{ as } n\rightarrow \infty\,,
\end{equation*}
where $V$ is defined in Equation~\eqref{eq:V_RR^T}.

\noindent{\bf Step \romannumeral 2}: We calculate matrix H.

\noindent Recall that
$H = \nabla^2_{\boldsymbol{\theta}}M(\boldsymbol{\theta}^*)$,
that is, the Hessian matrix of $M(\boldsymbol{\theta})$ at $\boldsymbol{\theta}^*$. And
\begin{equation*}
\begin{split}
    M(\boldsymbol{\theta}) 
    &= \cfrac{1}{T}\int_0^T \parallel e^{A^*t}\boldsymbol{x}_0^*-e^{At}\boldsymbol{x}_0\parallel_2^2 \,dt \\
    &= \cfrac{1}{T}\int_0^T \parallel e^{A^*t}\boldsymbol{x}_0^*\parallel_2^2 + \parallel e^{At}\boldsymbol{x}_0\parallel_2^2  -2(\boldsymbol{x}_0^*)^{\top}(e^{A^*t})^{\top}e^{At}\boldsymbol{x}_0   \,dt \,.
\end{split}
\end{equation*}
If one sets
\begin{equation*}
    h(\boldsymbol{x}_0, A) := \, \parallel e^{At}\boldsymbol{x}_0\parallel_2^2  -2(\boldsymbol{x}_0^*)^{\top}(e^{A^*t})^{\top}e^{At}\boldsymbol{x}_0\,,
\end{equation*}
by taking derivative of $h(\boldsymbol{x}_0, A)$ with respect to $\boldsymbol{x}_0$ one obtains
\begin{equation*}
    \cfrac{\partial h(\boldsymbol{x}_0, A)}{\partial \boldsymbol{x}_0}
    = 2(e^{At})^{\top}e^{At}\boldsymbol{x}_0 - 2(e^{At})^{\top}e^{A^*t}\boldsymbol{x}_0^*\,,
\end{equation*}
by further taking derivative with respect to $\boldsymbol{x}_0^{\top}$ one obtains that
\begin{equation*}
    \cfrac{\partial^2 h(\boldsymbol{x}_0, A)}{\partial \boldsymbol{x}_0 \partial \boldsymbol{x}_0^{\top}} = 2(e^{At})^{\top}e^{At}\,.
\end{equation*}
Therefore,
\begin{equation*}
\begin{split}
   \cfrac{\partial^2 M (\boldsymbol{\theta}^*)}{\partial \boldsymbol{x}_0 \partial \boldsymbol{x}_0^{\top}} &= \cfrac{\partial^2 M (\boldsymbol{\theta})}{\partial \boldsymbol{x}_0 \partial \boldsymbol{x}_0^{\top}}\bigg|_{\boldsymbol{\theta}=\boldsymbol{\theta}^*}
   = \cfrac{2}{T}\int_0^T (e^{A^*t})^{\top}e^{A^*t}\, dt \,.
\end{split}
\end{equation*}
Taking derivative of  $h(\boldsymbol{x}_0, A)$ with respect to $a_{jk}$ one obtains that
\begin{equation*}
\begin{split}
    \cfrac{\partial h(\boldsymbol{x}_0, A)}{\partial a_{jk}}
    &= Tr\bigg \{ \bigg( \cfrac{\partial h(\boldsymbol{x}_0, A)}{\partial e^{At}}\bigg) ^{\top} \cfrac{\partial e^{At}}{\partial a_{jk}}\bigg\}\\
    &= Tr\bigg\{(2e^{At}\boldsymbol{x}_0\boldsymbol{x}_0^{\top} -2e^{A^*t}\boldsymbol{x}_0^*\boldsymbol{x}_0^{\top})^{\top}\cfrac{\partial e^{At}}{\partial a_{jk}}\bigg\}\\
    &= Tr\bigg[\{2\boldsymbol{x}_0\boldsymbol{x}_0^{\top} (e^{At})^{\top} -2\boldsymbol{x}_0(\boldsymbol{x}_0^*)^{\top}(e^{A^*t})^{\top}\}\cfrac{\partial e^{At}}{\partial a_{jk}}\,\bigg]\\
    &= Tr\bigg\{ 2\boldsymbol{x}_0^{\top} (e^{At})^{\top}\cfrac{\partial e^{At}}{\partial a_{jk}}\boldsymbol{x}_0\bigg\} - Tr\bigg\{ 2(\boldsymbol{x}_0^*)^{\top}(e^{A^*t})^{\top}\cfrac{\partial e^{At}}{\partial a_{jk}}\boldsymbol{x}_0\bigg \}\\
    &=  2\boldsymbol{x}_0^{\top} (e^{At})^{\top}\cfrac{\partial e^{At}}{\partial a_{jk}}\boldsymbol{x}_0 -  2(\boldsymbol{x}_0^*)^{\top}(e^{A^*t})^{\top}\cfrac{\partial e^{At}}{\partial a_{jk}}\boldsymbol{x}_0\,,
\end{split}    
\end{equation*}
by further taking derivative with respect to $\boldsymbol{x}_0^{\top}$ one obtains that
\begin{equation*}
    \cfrac{\partial^2 h(\boldsymbol{x}_0, A)}{\partial a_{jk} \partial \boldsymbol{x}_0^{\top}} = 2\boldsymbol{x}_0^{\top}\bigg\{(e^{At})^{\top}\cfrac{\partial e^{At}}{\partial a_{jk}} +\bigg(\cfrac{\partial e^{At}}{\partial a_{jk}} \bigg)^{\top} e^{At}\bigg\} -2(\boldsymbol{x}_0^*)^{\top}(e^{A^*t})^{\top}\cfrac{\partial e^{At}}{\partial a_{jk}}\,.
\end{equation*}
Therefore,
\begin{equation*}
\begin{split}
    \cfrac{\partial^2 M(\boldsymbol{\theta}^*)}{\partial a_{jk}\partial \boldsymbol{x}_0^{\top}} 
   =\cfrac{\partial^2 M(\boldsymbol{\theta})}{\partial a_{jk}\partial \boldsymbol{x}_0^{\top}}\bigg|_{\boldsymbol{\theta}=\boldsymbol{\theta}^*}
   = \cfrac{2}{T}\int_0^T  \{Z_{jk}^*(t)\boldsymbol{x}_0^*\}^{\top}e^{A^*t}\, dt \, .   
\end{split}
\end{equation*}
Similarly,
\begin{equation*}
    \cfrac{\partial^2 h(\boldsymbol{x}_0, A)}{\partial a_{jk} \partial a_{pq}} = 2\boldsymbol{x}_0^{\top} \bigg(\cfrac{\partial e^{At}}{\partial a_{pq}}\bigg)^{\top}\cfrac{\partial e^{At}}{\partial a_{jk}}\boldsymbol{x}_0
    + 2\boldsymbol{x}_0^{\top} (e^{At})^{\top}\cfrac{\partial ^2 e^{At}}{\partial a_{jk}\partial a_{pq}}\boldsymbol{x}_0
    - 2(\boldsymbol{x}_0^*)^{\top}(e^{A^*t})^{\top}\cfrac{\partial ^2 e^{At}}{\partial a_{jk}\partial a_{pq}}\boldsymbol{x}_0\,,
\end{equation*}
then by some calculation, one obtains that
\begin{equation*}
\begin{split}
    \cfrac{\partial^2 M(\boldsymbol{\theta}^*)}{\partial a_{jk} \partial a_{pq}} 
    = \cfrac{\partial^2 M(\boldsymbol{\theta})}{\partial a_{jk} \partial a_{pq}} \bigg|_{\boldsymbol{\theta}=\boldsymbol{\theta}^*}
    = \cfrac{2}{T}\int_0^T(Z_{pq}^*(t)\boldsymbol{x}_0^*)^{\top} Z_{jk}^*(t)\boldsymbol{x}_0^*\,dt\,.
\end{split}
\end{equation*}
If one denotes 
\begin{equation}\label{eq:S(theta*,t)}
    S(\boldsymbol{\theta}^*, t) := \big(e^{A^*t}, Z_{11}^*(t) \boldsymbol{x}_0^*,\ldots, Z_{1d}^*(t) \boldsymbol{x}_0^*,\ldots, Z_{dd}^*(t) \boldsymbol{x}_0^*\big)\,,
\end{equation}
one sees that
\begin{equation}\label{eq:H}
    H = 2\int_0^T  S(\boldsymbol{\theta}^*, t)^{\top} S(\boldsymbol{\theta}^*, t)/T \,dt\,,
\end{equation}
and $H\in \mathbb{R}^{(d+d^2)\times (d+d^2)}$.

Rearranging $V$ in Equation~\eqref{eq:V_RR^T}, one obtains that
\begin{equation}\label{eq:H_SS^T}
    V = 4\int_0^T  S(\boldsymbol{\theta}^*, t)^{\top} \Sigma S(\boldsymbol{\theta}^*, t)/T \,dt\,,
\end{equation}
obviously, $V$ and $H$ have similar forms. Therefore, by using the same way of proving $V$ is positive definite, one can prove that $H$ is positive definite. Since $H$ is symmetric, $H$ is nonsingular. 

Therefore, we have completed the proof and provided the explicit forms of both $V$ and $H$. Moreover, we have shown that both $V$ and $H$ are positive definite and nonsingular.
\end{proof}

\subsection{Proof of Corollary~\ref{corollary:normality aggregate}}\label{proof:corollary3.2.1}

\begin{proof}
According to the proof of Corollary~\ref{corollary:consistency aggregate} in Appendix~\ref{proof:corollary3.1.1}, one sees that assumptions A1, A2 and A4-A6 hold with respect to the new ODE system, the new parameter $\tilde{\boldsymbol{\theta}}^*$ and the new error terms $\tilde{\boldsymbol{\epsilon}}_i$ corresponding to the aggregated observations $\tilde{\boldsymbol{Y}}$.

Therefore, by Theorem~\ref{theorem:Asymptotic normality}, one obtains that 
\begin{equation*}
    \sqrt{\tilde{n}}(\hat{\tilde{\boldsymbol{\theta}}} -\tilde{\boldsymbol{\theta}}^* )\xrightarrow{d} N(\boldsymbol{0}, \tilde{H}^{-1}\tilde{V} \tilde{H}^{-1}) \,, \mbox{ as } \tilde{n}\rightarrow \infty\,,
\end{equation*}
where 
\begin{equation*}
    \tilde{\boldsymbol{\theta}}^* = (\tilde{\boldsymbol{x}}_0^*, \tilde{A}^*) = \big((I+e^{A^*\Delta_t} + \cdots + e^{A^*(k-1)\Delta_t})\boldsymbol{x}_0^*/k, A^*\big)\,,
\end{equation*}
recall that we we have built the relationship between $\tilde{\boldsymbol{\theta}}^*$ and $\boldsymbol{\theta}^*$ in Corollary~\ref{corollary:aggregate}.

Using the same way we calculate $V$ and $H$ in the proof of Theorem~\ref{theorem:Asymptotic normality} in Appendix~\ref{proof:theorem3.2}, one obtains that 
\begin{equation*}
    \tilde{V} = 4\int_0^{\tilde{T}}  R(\tilde{\boldsymbol{\theta}}^*, t)^{\top} R(\tilde{\boldsymbol{\theta}}^*, t)/k\tilde{T} \,dt\,,
\end{equation*}
where $R(\tilde{\boldsymbol{\theta}}^*,t)$ is defined in Equation~\eqref{eq:R(theta*,t)}, and $\tilde{T} = (\lfloor 
n/k\rfloor-1)kT/(n-1)$. Note that, one devides $k$ in the formula of $\tilde{V}$, since the covariance matrix of the new error terms $\{ \tilde{\boldsymbol{\epsilon}}_i\}$ is $\Sigma/k$ based on the generation rules of the aggregated observations.

Similarly, one obtains that
\begin{equation*}
    \tilde{H} = 2\int_0^{\tilde{T}}  S(\tilde{\boldsymbol{\theta}}^*, t)^{\top} S(\tilde{\boldsymbol{\theta}}^*, t)/\tilde{T} \,dt\,,
\end{equation*}
where $ S(\tilde{\boldsymbol{\theta}}^*, t)$ is defined in Equation~\eqref{eq:S(theta*,t)}, and $\tilde{T} = (\lfloor 
n/k\rfloor-1)kT/(n-1)$.

Moreover, by using the same way we prove that both $V$ and $H$ are positive definite and nonsingular in the proof of Theorem~\ref{theorem:Asymptotic normality} in Appendix~\ref{proof:theorem3.2}, one obtains that both $\tilde{V}$
and $\tilde{H}$ are positive definite and nonsingular.

The definition of $\hat{\boldsymbol{\theta}}_{\tilde{n}}$ in Corollary~\ref{corollary:consistency aggregate}, that is,
\begin{equation*}
    \hat{\boldsymbol{\theta}}_{\tilde{n}}:=g(\hat{\tilde{\boldsymbol{\theta}}}) := \big(k(I+e^{\hat{\tilde{A}}\Delta_t} + \cdots + e^{\hat{\tilde{A}}(k-1)\Delta_t})^{-1}\hat{\tilde{\boldsymbol{x}}}_0, \hat{\tilde{A}}\big)\,,
\end{equation*}
shows that the function $g(\boldsymbol{\theta})$ has the following form
\begin{equation*}
    g(\boldsymbol{\theta}) = g(\boldsymbol{x}_0, A) = \big(k(I+e^{A\Delta_t} + \cdots + e^{A(k-1)\Delta_t})^{-1}\boldsymbol{x}_0, A\big)\,.
\end{equation*}
By plugging $\tilde{\boldsymbol{\theta}}^*$ in, one obtains that 
\begin{equation*}
\begin{split}
    g(\tilde{\boldsymbol{\theta}}^*) &= \big(k(I+e^{\tilde{A}^*\Delta_t} + \cdots + e^{\tilde{A}^*(k-1)\Delta_t})^{-1}\tilde{\boldsymbol{x}}_0^*, \tilde{A}^*\big)
    = (\boldsymbol{x}_0^*, A^*)
    = \boldsymbol{\theta}^*\,.
\end{split}
\end{equation*}
Taking derivative of $g(\boldsymbol{\theta})$ with respect to $\boldsymbol{x}_0$ one obtains that
\begin{equation*}
\begin{split}
    \cfrac{\partial g(\tilde{\boldsymbol{\theta}}^*)}{\partial \boldsymbol{x}_0^{\top}}  
    &= \cfrac{\partial g(\boldsymbol{\theta})}{\partial \boldsymbol{x}_0^{\top}}\bigg|_{\boldsymbol{\theta}=\tilde{\boldsymbol{\theta}}^*}\\
    &= \big(k\{(I+e^{A^*\Delta_t} + \cdots + e^{A^*(k-1)\Delta_t})^{-1}\}^{\top}, \underbrace{\boldsymbol{0}_d,\ldots, \boldsymbol{0}_d}_\text{$d^2$ entries}\big)^{\top}
    \in \mathbb{R}^{(d+d^2)\times d}\,.
\end{split}
\end{equation*}

Since
\begin{equation*}
\begin{split}
    &\cfrac{\partial k(I+e^{A\Delta_t} + \cdots + e^{A(k-1)\Delta_t})^{-1}\boldsymbol{x}_0}{\partial a_{pq}}\\
    = &-k(I+e^{A\Delta_t} + \cdots + e^{A(k-1)\Delta_t})^{-2}\cfrac{\partial (I+e^{A\Delta_t} + \cdots + e^{A(k-1)\Delta_t}) }{\partial a_{pq}}\boldsymbol{x}_0\\
    = & -k(I+e^{A\Delta_t} + \cdots + e^{A(k-1)\Delta_t})^{-2} [Z_{pq}(\Delta_t) + \cdots + Z_{pq}\{(k-1)\Delta_t\}]\boldsymbol{x}_0\,,
\end{split}
\end{equation*}
one obtains that
\begin{equation*}
\begin{split}
    &\cfrac{\partial g(\tilde{\boldsymbol{\theta}}^*)}{\partial a_{pq}}  
    = \cfrac{\partial g(\boldsymbol{\theta})}{\partial a_{pq}}\bigg|_{\boldsymbol{\theta}=\tilde{\boldsymbol{\theta}}^*}\\
    = &\big[ -k\big\{(I+e^{A^*\Delta_t} + \cdots + e^{A^*(k-1)\Delta_t})^{-2} [Z_{pq}^*(\Delta_t) + \cdots + Z_{pq}^*\{(k-1)\Delta_t\}]\tilde{\boldsymbol{x}}_0^*\big\}^{\top}, \\ 
    & \underbrace{0,\ldots,0, 1, 0,\ldots, 0}_\text{$d^2$ entries}\big] ^{\top}\\
    \in & \mathbb{R}^{(d+d^2)\times 1}\,,
\end{split}
\end{equation*}
where $1$ in the last $d^2$ entries corresponds to the $pq$-th entry, where $p,q = 1,\ldots, d$.

Therefore, the gradient of $g(\boldsymbol{\theta})$ with respect to $\boldsymbol{\theta}$ at $\tilde{\boldsymbol{\theta}}^*$ is
\begin{equation*}
    G:=\nabla g(\tilde{\boldsymbol{\theta}}^*) = \bigg( \cfrac{\partial g(\tilde{\boldsymbol{\theta}}^*)}{\partial \boldsymbol{x}_0^{\top}},
    \cfrac{\partial g(\tilde{\boldsymbol{\theta}}^*)}{\partial a_{11}},
    \cdots,
    \cfrac{\partial g(\tilde{\boldsymbol{\theta}}^*)}{\partial a_{1d}},
    \cdots,
    \cfrac{\partial g(\tilde{\boldsymbol{\theta}}^*)}{\partial a_{dd}}\bigg)
    \in \mathbb{R}^{(d+d^2)\times (d+d^2)}\,.
\end{equation*}

By the multivariate Delta method, one obtains that
\begin{equation*}
    \sqrt{\tilde{n}}\{g(\hat{\tilde{\boldsymbol{\theta}}}) -g(\tilde{\boldsymbol{\theta}}^*) \}\xrightarrow{d} N(\boldsymbol{0}, G\tilde{H}^{-1}\tilde{V} \tilde{H}^{-1}G^{\top}) \,, \mbox{ as } \tilde{n}\rightarrow \infty\,,
\end{equation*}
which is 
\begin{equation*}
    \sqrt{\tilde{n}}(\hat{\boldsymbol{\theta}}_{\tilde{n}} -\boldsymbol{\theta}^* )\xrightarrow{d} N(\boldsymbol{0}, G\tilde{H}^{-1}\tilde{V} \tilde{H}^{-1}G^{\top})\,, \mbox{ as } \tilde{n}\rightarrow \infty\,.
\end{equation*}
\end{proof}


\subsection{Proof of Corollary~\ref{corollary:normality timescaled}}\label{proof:corollary3.2.2}

\begin{proof}
According to the proof of Corollary~\ref{corollary:consistency timescaled} in Appendix~\ref{proof:corollary3.1.2}, one sees that assumptions A1, A2 and A4-A6 hold with respect to the new ODE system, the new parameter $\tilde{\boldsymbol{\theta}}^*$ and the new error terms $\tilde{\boldsymbol{\epsilon}}_i$ corresponding to the time-scaled observations $\tilde{\boldsymbol{Y}}$.

Therefore, by Theorem~\ref{theorem:Asymptotic normality}, one obtains that 
\begin{equation*}
    \sqrt{\tilde{n}}(\hat{\tilde{\boldsymbol{\theta}}} -\tilde{\boldsymbol{\theta}}^* )\xrightarrow{d} N(\boldsymbol{0}, \tilde{H}^{-1}\tilde{V} \tilde{H}^{-1}) \,, \mbox{ as } \tilde{n}\rightarrow \infty\,,
\end{equation*}
where 
\begin{equation*}
    \tilde{\boldsymbol{\theta}}^* = (\tilde{\boldsymbol{x}}_0^*, \tilde{A}^*) = (\boldsymbol{x}_0^*, A^*/k)\,,
\end{equation*}
recall that we we have built the relationship between $\tilde{\boldsymbol{\theta}}^*$ and $\boldsymbol{\theta}^*$ in Corollary~\ref{corollary:timescaled}.

Using the same way we calculate $V$ and $H$ in the proof of Theorem~\ref{theorem:Asymptotic normality} in Appendix~\ref{proof:theorem3.2}, one obtains that 
\begin{equation*}
    \tilde{V} = 4\int_0^{kT}  R(\tilde{\boldsymbol{\theta}}^*, t)^{\top} R(\tilde{\boldsymbol{\theta}}^*, t)/(kT) \,dt\,,
\end{equation*}
where $R(\tilde{\boldsymbol{\theta}}^*,t)$ is defined in Equation~\eqref{eq:R(theta*,t)}.

Similarly, one obtains that
\begin{equation*}
    \tilde{H} = 2\int_0^{kT}  S(\tilde{\boldsymbol{\theta}}^*, t)^{\top} S(\tilde{\boldsymbol{\theta}}^*, t)/(kT) \,dt\,,
\end{equation*}
where $ S(\tilde{\boldsymbol{\theta}}^*, t)$ is defined in Equation~\eqref{eq:S(theta*,t)}.

Moreover, by using the same way we prove that both $V$ and $H$ are positive definite and nonsingular in the proof of Theorem~\ref{theorem:Asymptotic normality} in Appendix~\ref{proof:theorem3.2}, one obtains that both $\tilde{V}$
and $\tilde{H}$ are positive definite and nonsingular.

The definition of $\hat{\boldsymbol{\theta}}_{\tilde{n}}$ in Corollary~\ref{corollary:consistency timescaled}, that is,
\begin{equation*}
    \hat{\boldsymbol{\theta}}_{\tilde{n}}:=g(\hat{\tilde{\boldsymbol{\theta}}}) := (\hat{\tilde{\boldsymbol{x}}}_0, k\hat{\tilde{A}})\,,
\end{equation*}
shows that the function $g(\boldsymbol{\theta})$ has the following form
\begin{equation*}
    g(\boldsymbol{\theta}) = g(\boldsymbol{x}_0, A )
    = (\boldsymbol{x}_0, kA)\,.
\end{equation*}
By plugging $\tilde{\boldsymbol{\theta}}^*$ in, one obtains that 
\begin{equation*}
    g(\tilde{\boldsymbol{\theta}}^*) 
    = (\tilde{\boldsymbol{x}}_0^*, k\tilde{A}^*\big)
    = (\boldsymbol{x}_0^*, A^*)
    = \boldsymbol{\theta}^*\,.
\end{equation*}

Following the same way we calculate the gradient of $g(\boldsymbol{\theta})$ with respect to $\boldsymbol{\theta}$ at $\tilde{\boldsymbol{\theta}}^*$ in the proof of Corollary~\ref{corollary:normality aggregate} in Appendix~\ref{proof:corollary3.2.1}, one obtains that
\begin{equation*}
     G:=\nabla g(\tilde{\boldsymbol{\theta}}^*) = \cfrac{\partial g(\tilde{\boldsymbol{\theta}}^*)}{\partial \boldsymbol{\theta}} =
    \cfrac{\partial g(\boldsymbol{\theta})}{\partial \boldsymbol{\theta}}\bigg |_{\boldsymbol{\theta} = \tilde{\boldsymbol{\theta}}^*}
    =\begin{bmatrix}
    I_d & \boldsymbol{0}_d & \boldsymbol{0}_d & \cdots & \boldsymbol{0}_d\\
    \boldsymbol{0}_d^{\top} & k & 0 & \cdots & 0\\
    \boldsymbol{0}_d^{\top} & 0 & k & \cdots & 0\\
    \vdots & \vdots & \vdots & \ddots & 0\\
    \boldsymbol{0}_d^{\top} & 0 & 0 & \cdots & k
    \end{bmatrix}
    \in \mathbb{R}^{(d+d^2) \times (d+d^2)}\,.
\end{equation*}

By the multivariate Delta method, one obtains that
\begin{equation*}
    \sqrt{\tilde{n}}\{g(\hat{\tilde{\boldsymbol{\theta}}}) -g(\tilde{\boldsymbol{\theta}}^*) \}\xrightarrow{d} N(\boldsymbol{0}, G\tilde{H}^{-1}\tilde{V} \tilde{H}^{-1}G^{\top}) \,, \mbox{ as } \tilde{n}\rightarrow \infty\,,
\end{equation*}
which is 
\begin{equation*}
    \sqrt{\tilde{n}}(\hat{\boldsymbol{\theta}}_{\tilde{n}} -\boldsymbol{\theta}^* )\xrightarrow{d} N(\boldsymbol{0}, G\tilde{H}^{-1}\tilde{V} \tilde{H}^{-1}G^{\top})\,, \mbox{ as } \tilde{n}\rightarrow \infty\,.
\end{equation*}
\end{proof}

\vskip 0.2in
\bibliography{references}

\end{document}